%% file: main_arXiv.tex
\documentclass[twoside,11pt]{article}

\usepackage{blindtext}

\usepackage[preprint]{jmlr2e} 

\input{header}
\usepackage{macros}

\usepackage{enumitem}

\usepackage{lastpage}
\jmlrheading{23}{2026}{1-\pageref{LastPage}}{1/21; Revised 5/22}{9/22}{21-0000}{Daichi Kuroda, Maximilien Dreveton, Matthias Grossglauser, and Patrick Thiran}

\ShortHeadings{Axiomatic Hierarchical Clustering}{Kuroda, Dreveton, Grossglauser, and Thiran}
\firstpageno{1}

\begin{document}

\title{
Hierarchical Clustering Can Jointly Satisfy Richness, Consistency, and Scale Invariance
}

\author{\name Daichi Kuroda \email daichi.kuroda@epfl.ch \\
\addr School of Computer and Communication Sciences\\
École Polytechnique Fédérale de Lausanne (EPFL)\\
Lausanne, 1015, Switzerland
\AND
\name Maximilien Dreveton \email maximilien.dreveton@univ-eiffel.fr \\
\addr LAMA, UMR-CNRS 8050, \\
Université Gustave Eiﬀel \\
5 Bd Descartes, 77454 Marne-la-Vallée, France
\AND
\name Matthias Grossglauser \email matthias.grossglauser@epfl.ch \\
\addr School of Computer and Communication Sciences\\
École Polytechnique Fédérale de Lausanne (EPFL)\\
Lausanne, 1015, Switzerland 
\AND
\name Patrick Thiran \email patrick.thiran@epfl.ch \\
\addr School of Computer and Communication Sciences\\
École Polytechnique Fédérale de Lausanne (EPFL)\\
Lausanne, 1015, Switzerland
}
\editor{My editor}

\maketitle

\begin{abstract}
\input{abstract}
\end{abstract}



\input{text}

\appendix
\input{appendix}
\vskip 0.2in
\bibliography{biblio}

\end{document}

%% file: abstract.tex
Despite its ubiquity, clustering lacks a universally accepted definition of what is a cluster. 
Kleinberg's Impossibility Theorem formalizes this difficulty by showing that no flat clustering method can simultaneously satisfy three natural axioms: scale invariance, richness, and consistency. 
In this paper, we ask whether 
this impossibility persists when the output is a hierarchy rather than a single partition.
We show that, in contrast to the flat clustering setting, the hierarchical analog of these axioms are jointly satisfiable. 
In fact, there exist uncountably many hierarchical clustering methods satisfying these
axioms, which we call admissible. We explicitly construct several admissible methods, including methods based on well-separated clusters and a non-binary version of single linkage. 
For certain pairs of admissible methods, the hierarchy produced by one always refines that produced by the other. This refinement relation defines a partial order on the class of admissible methods. 
This partially ordered set has no greatest element and contains uncountably many pairwise incompatible maximal elements, revealing substantial diversity among admissible methods. Nevertheless, this diversity is constrained: every admissible method contains a hierarchy of sufficiently well-separated clusters, and every finite collection of admissible methods shares such a nontrivial common backbone. 


%% file: text.tex
\section{Introduction}
\label{sec:intro}
Clustering is one of the most fundamental tasks in unsupervised learning. 
Given pairwise dissimilarities between data points, clustering seeks to uncover meaningful group structure without access to labels or ground truth. Yet this task is intrinsically underdetermined: there is no universally accepted definition of a cluster. Axiomatic frameworks therefore provide a principled way to state and compare desirable properties of clustering methods. 

A central result in this direction is Kleinberg’s Impossibility Theorem \citep{kleinberg2002impossibility}, that shows that no flat clustering method mapping dissimilarities to partitions can simultaneously satisfy three natural axioms: scale invariance, richness, and consistency. 
This theorem has had a profound influence on the theory of clustering: It implies that some trade-off among the axioms, the problem formulation, and/or the clustering output, is unavoidable. 
A large body of subsequent work explored ways to circumvent this impossibility by weakening or modifying the axioms, 
or by modifying the input or output  space~\citep{ben2008measures,zadeh2012uniqueness,strazzeri2022possibility,willson2024axioms}. 
In contrast, in this paper, we ask a different question:
\begin{quote}
\emph{Does Kleinberg's impossibility persist when the output is a hierarchy rather than a flat partition?}
\end{quote}

Out of the three Kleinberg axioms, \emph{consistency} is arguably the most contentious for the flat clustering setting. It requires that if all within-cluster dissimilarities (with respect to the output partition) are decreased and all cross-cluster dissimilarities are increased, then the output clustering must remain unchanged. This formalizes the intuition that strengthening the evidence for a clustering should not overturn it. However, such transformations can also create an arbitrarily strong \emph{substructure} within a cluster, thus revealing finer distinctions that a flat partition is unable to express. This tension lies at the heart of Kleinberg’s impossibility result. Figure~\ref{fig:axiom-motivation} illustrates this phenomenon: Starting from a single cluster $C_1 \cup C_2$, a permissible strengthening can make each subcluster, $C_1$ and $C_2$, much tighter than their union, thereby revealing two new clusters that a flat clustering method that respects the consistency axiom would be forced to ignore. 

\begin{figure}[!ht]
  \centering
    \centering
    \includegraphics[width=0.25\textwidth]{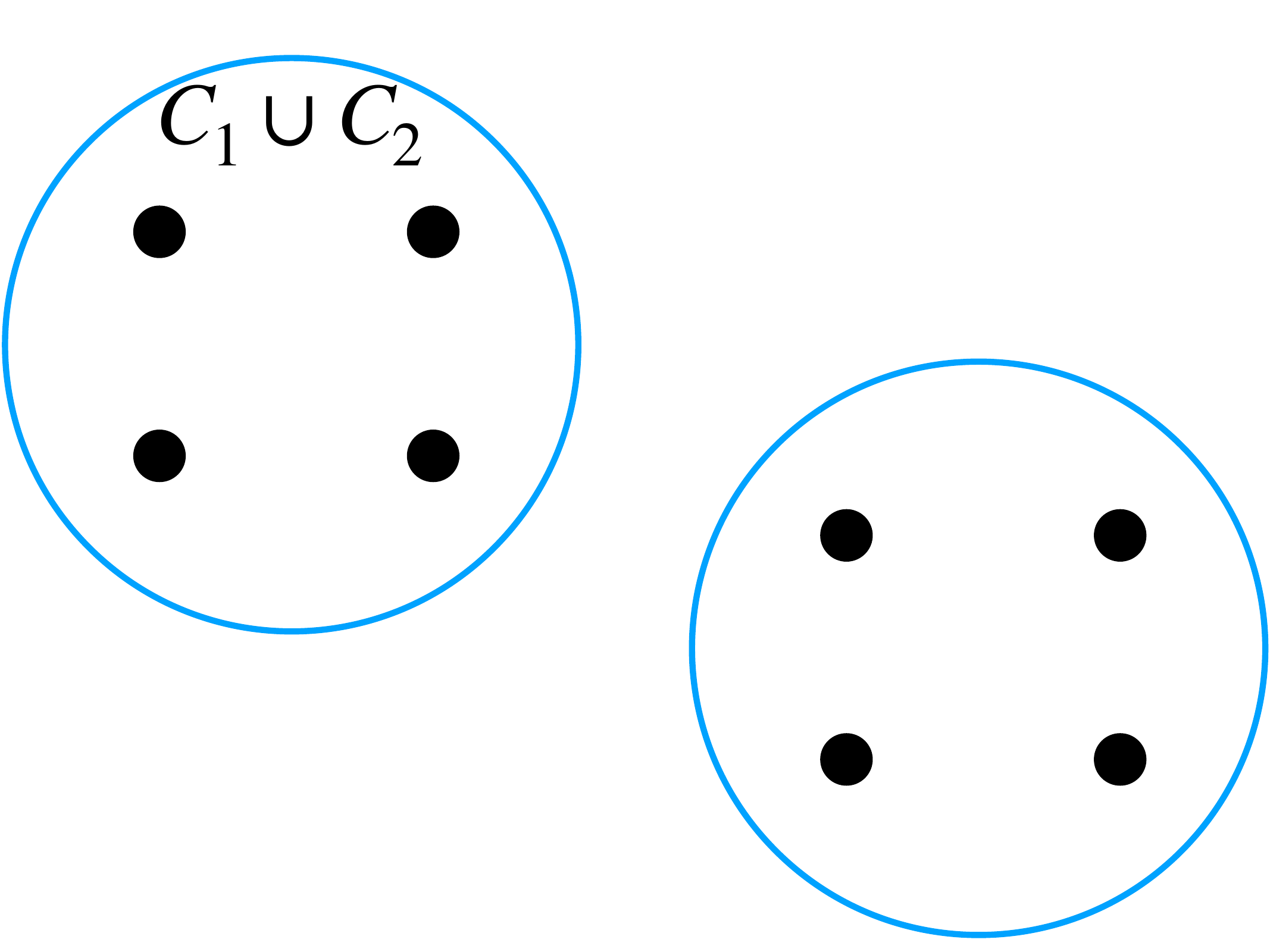}
    \hspace{90pt}
    \centering
    \includegraphics[width=0.25\textwidth]{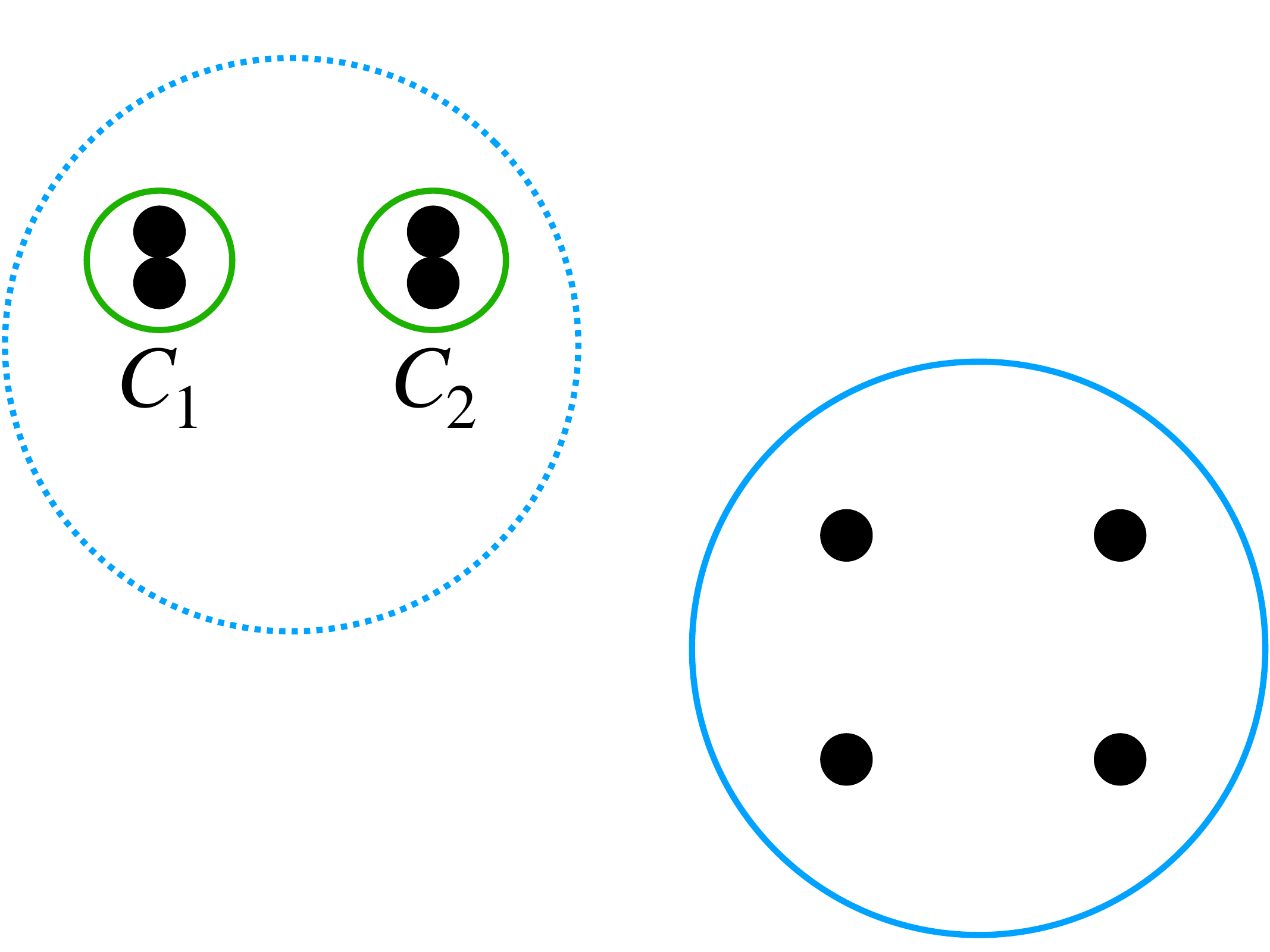}
  \caption{Strengthening the cluster $C_1 \cup C_2$ can create two well-separated subclusters, $C_1$ and $C_2$, which a Kleinberg-consistent flat clustering method is nevertheless forced to ignore. Whereas hierarchical clustering methods can have both $C_1 \cup C_2$, and $C_1$ and $C_2$ as they are properly nested.
  }
  \label{fig:axiom-motivation}
\end{figure}

This observation suggests moving beyond flat clustering by considering hierarchical clustering methods. Indeed, a hierarchical clustering can preserve the desirable cluster $C_1 \cup C_2$ while simultaneously incorporating the newly emerging subclusters $C_1$ and $C_2$: strengthening the evidence for an existing cluster need not prevent the method from expressing finer structure within that cluster. 
Formally, a hierarchical clustering method is a map from dissimilarities to hierarchies, represented as laminar families of clusters. We formulate hierarchical analogs of scale invariance, richness, and consistency, and additionally impose permutation invariance as a basic symmetry requirement.

Our first main result is positive: Contrary to the flat setting, these axioms are \emph{jointly satisfiable} in the hierarchical setting. In fact, there exist uncountably many admissible methods. We explicitly construct several admissible methods, including hierarchies based on well-separated clusters and a non-binary version of single linkage. In contrast, non-binary variants of other classical linkage methods (such as complete, average, Ward, centroid, and median linkage) fail to satisfy the axioms. 

Beyond existence, we study the global structure of the family of admissible methods under the refinement order. This family forms a remarkably diverse partially ordered set: it has uncountable height, width, and cellularity, and contains uncountably many pairwise incompatible maximal elements. In particular, there is no greatest admissible method. Nevertheless, the axioms impose a nontrivial common structure. We prove a \emph{backbone property}: every admissible method refines a hierarchy of sufficiently well-separated clusters. Moreover, every finite collection of admissible methods shares such a common backbone. Thus, the axioms allow substantial diversity while still enforcing agreement on sufficiently well-separated cluster structure.

We then consider a stronger requirement that is specific to hierarchical clustering. When the input dissimilarity is an ultrametric,\footnote{An ultrametric is a dissimilarity satisfying the strong triangle inequality \(u(x,y)\leq \max\{u(x,z),u(y,z)\}\) for all \(x,y,z\); such a dissimilarity canonically encodes a hierarchy through its nested distance balls.} it already encodes a canonical hierarchical structure. It is therefore natural to require a hierarchical clustering method to recover this hierarchy exactly. We call this property \emph{exactness on ultrametrics}. The resulting class of strongly admissible methods remains uncountable and retains the backbone and maximality phenomena described above. 
However, the additional requirement sharpens the order-theoretic structure: unlike the canonical admissible class, the strongly admissible class has a least element under refinement.

 Finally, motivated by practical clustering pipelines, we study preprocessing transformations of the input dissimilarity. We derive general conditions under which composition with such a transformation preserves each axiom, and illustrate these principles with several common preprocessing operations.

 Our work connects several strands of the clustering literature. 
 It contributes to the efforts to understand and bypass Kleinberg's impossibility theorem by modifying the axioms and/or the clustering problem formulation~\citep{ben2008measures,cohen2018clustering,willson2024axioms}. 
 It is also related to axiomatic and structural characterizations of hierarchical clustering methods~\citep{carlsson2010characterization,ackerman2010characterization,ackerman2016characterization}, as well as to population-level axiomatizations of hierarchical clustering~\citep{thomann2015axiomatic,arias2025axiomatic}. 
 
 In particular, hierarchical outputs have previously been shown to support positive axiomatic results:  \citet{carlsson2010characterization} characterize single linkage using an axiom system that differ substantially from Kleinberg's framework, whereas \citet{ackerman2016characterization} characterize linkage-based hierarchical methods using locality and a weaker consistency that only considers moving farther apart clusters that are already well-separated. 
 Both works retain the numerical scales at which clusters merge, so their outputs are height-labeled hierarchies (dendrograms), equivalently represented by ultrametrics. They therefore study maps from input dissimilarities to output ultrametrics. Instead, we consider methods returning unweighted hierarchies, which retain only the nested cluster structure. 
 Within this less structured output framework, our axioms more closely mirror Kleinberg's original requirements while imposing fewer structural constraints. Accordingly, rather than characterizing a unique method or a prescribed algorithmic family, we study the much more diverse class of hierarchical methods satisfying these axioms. Section~\ref{sec:related_works} provides a detailed comparison.

\subsection{Definitions and Notation}
Throughout the paper, $\cX$ is a finite set of items with cardinality $n$. 
Moreover, as the labeling of the elements of $\cX$ is irrelevant to an unsupervised task such as clustering, we implicitly assume $\cX = [n]$, where $n = |\cX|$ is finite and $[n] = \{1,\dots,n\}$. 
To avoid trivial cases, we always assume $n \ge 4$.\footnote{For \(n\le 2\), a hierarchy can contain only the root and singleton leaves and is therefore necessarily the star hierarchy. For \(n=3\), there is only one non-star hierarchy up to relabeling.} 
A \emph{dissimilarity function} on $\cX$ is a map $d \colon \cX\times\cX\to\R_{\ge 0}$ such that $d(x,x)=0$ and $d(x,y)=d(y,x)>0$ for all distinct $x, y\in\cX$. We do not assume that $d$ obeys the triangle inequality, hence $d$ is not necessarily a distance. We denote by $\cD(\cX)$ the set of all dissimilarity functions on~$\cX$.
We use the convention $\min \emptyset = \infty$. 

A \emph{cluster} $C$ is a nonempty subset of $\cX$, and a \emph{partition} of~$\cX$ is a set $\cC=\{C_1,\dots,C_k\}$ of pairwise disjoint clusters whose union is $\cX$. Let $\bC:= \cX \setminus C$ denote the complement of the cluster $C$ with respect to $\cX$. We write $\cP(\cX)$ for the set of all partitions of $\cX$.

\subsection{Structure of the Paper} 
The remainder of the paper is organized as follows.
In Section~\ref{sec:achievability}, we state the main axioms and establish the corresponding achievability results.
In Section~\ref{sec:admissible_HCs}, we present several admissible hierarchical clustering methods.
In Section~\ref{sec:structure}, we study the structural properties of the set of admissible methods.
In Section~\ref{sec:extension}, we add exactness on ultrametrics to the axiom system and characterize preprocessing transformations that preserve the axioms.
In Section~\ref{sec:related_works}, we discuss related work and, in Section~\ref{sec:conclusion}, conclude the paper.
Omitted proofs and technical lemmas are provided in the Appendix. 

\subsection{Use of Large Language Models (LLM)}
Whereas the conceptualization of the project underlying this paper was carried out entirely by the authors, we also used GPT-5.6 Sol to improve clarity, wording, presentation, and to assist in identifying and correcting minor errors and typos in earlier drafts. 
However, the vast majority of the mathematical content and proofs were generated by us. The only mathematical results for which LLMs were used to develop proof techniques are Proposition~\ref{prop:stable_cluster_is_admissible}, Lemmas~\ref{lem:incompatible_family} and \ref{lem:pairwise_incompatible}, and Proposition~\ref{prop:general}(iv).
We also used the model to assist with the numerical simulations in Appendix~\ref{app:size_Tglob}.
All text proposed by LLM was rigorously checked and edited by the authors, and we take full responsibility for this article.

\section{From Impossibility to Achievability}
\label{sec:achievability}
 In this section, we first recall Kleinberg's impossibility theorem for flat clustering and then show that, in contrast, the hierarchical analogs of his axioms are jointly satisfiable.
 
\subsection{Kleinberg's Impossibility Theorem for Flat Clustering}
\label{subsec:kleinberg_flat}
We begin by recalling the axiomatic framework for flat clustering, introduced by \cite{kleinberg2002impossibility}. A \emph{flat clustering method} is a map
$$f \colon \cD(\cX) \to \cP(\cX)$$ 
such that $f(d)$ is a partition of $\cX$ for every dissimilarity $d \in \cD(\cX)$. 

We first introduce a transformation of dissimilarities, called strengthening\footnote{This transformation is called a $\Gamma$-transformation in~\citet{kleinberg2002impossibility}, where $\Gamma$ is a partition.} that reinforces a given clustering. Intuitively, a strengthening moves points within the same cluster closer to each other, and points in different clusters further apart.

\begin{definition}[$\cC$-strengthening]
\label{def:strengthening}
Let $\cC \in \cP(\cX)$ and $d, d' \in \cD(\cX)$ be two dissimilarity functions. We say that $d' \in \cD(\cX)$ is a $\cC$-strengthening of $d$ if for all clusters $C \in \cC$,
    \begin{itemize}
    \itemsep0em
        \item (Intra-cluster contraction) $d'(x,y) \le d(x,y)$ for all $x,y \in C$; 
        \item (Inter-cluster expansion) $d'(x,y) \ge d(x,y)$ for all $x\in C$, $y \notin C$.
    \end{itemize}
\end{definition}

We now formalize the properties that a clustering method should satisfy. These axioms capture natural requirements such as invariance to scaling, expressiveness, and stability under strengthening.

\begin{definition}
\label{def:admissible_Kleinberg}
A flat clustering method $f$ is:
\begin{itemize}[nosep]
\item \textbf{scale invariant} if $f(\beta \,d) = f(d)$ for every $d\in \cD(\cX)$ and every $\beta>0$;
\item \textbf{rich} if $f$ is surjective, that is, for every $\cC \in \cP(\cX)$, there exists $d \in \cD(\cX)$ such that $f(d) = \cC$; 
\item \textbf{consistent} if $f(d) = f(d')$ for every $d\in \cD(\cX)$ and every $f(d)$-strengthening $d'$ of $d$.
\end{itemize}
\end{definition} 

These axioms appear natural when considered individually, 
although the consistency axiom is sometimes viewed as contentious because it permits the creation of new clusters
(see Section~\ref{sec:intro}). 
Kleinberg established that these three requirements are not jointly satisfiable.

\begin{theorem}[\cite{kleinberg2002impossibility}]
\label{thm:kleinberg}
There is no flat clustering method $f$ that is simultaneously scale invariant, rich, and consistent.  
\end{theorem}

\subsection{Achievability of Hierarchical Clustering}
\label{subsec:achievability}

We now turn to hierarchical clustering. Unlike flat clustering, which outputs a single partition of the dataset, hierarchical clustering outputs \emph{nested clusters}. Formally, a hierarchy on $\cX$ is a rooted tree whose leaves are the singleton sets $\{x\}$ for $x \in \cX$ and whose root is the full set $\cX$. We represent such a tree by its associated nested set of clusters. 

\begin{definition}[Hierarchy]
\label{def:hierarchy}
Let $2^{\cX}$ denote the power set of $\cX$. 
A set $\Psi \subseteq 2^{\cX} \setminus \{\emptyset\}$ is a \emph{hierarchy} on $\cX$ if it satisfies:
\begin{enumerate}[nosep]
\item \emph{Laminarity}: For all $C_1,C_2\in\Psi$, we have either $C_1\cap C_2=\emptyset$, or $C_1\subseteq C_2$, or $C_2\subseteq C_1$. 
    \item \emph{Root and leaves:} $\cX \in \Psi$ and $\{x\} \in \Psi$ for every $x \in \cX$.
\end{enumerate}
We denote by $\cT(\cX)$ the set of all hierarchies on $\cX$.
\end{definition}
\begin{example}
Let $\cX = \{1,2,3,4\}$. The set
$\Psi = \big\{\{1,2,3,4\}, \{1,2\}, \{3,4\}, \{1\},\{2\},\{3\},\{4\}\big\}$
is a hierarchy representing a binary tree with root $\{1,2,3,4\}$ and two children $\{1,2\}$ and $\{3,4\}$. 
\end{example}

A \emph{hierarchical clustering method} is a map 
\[
 T \colon \cD(\cX) \to \cT(\cX)
\]
such that for every $d \in \cD(\cX)$, $T(d)$ is a hierarchy on $\cX$. 
Kleinberg’s axioms of scale invariance, consistency, and richness admit natural analog in the hierarchical setting. 

\begin{definition}
\label{def:admissible}
A hierarchical clustering method $T$ on $\cX$ is 
\begin{enumerate}[nosep]
\item \textbf{scale invariant} if \(T(\beta \, d) = T(d) \) for every $\beta>0$ and every $d\in\cD(\cX)$;
\item \textbf{partition rich} if for every partition $\cC\in\cP(\cX)$, there exists $d \in \cD(\cX)$ such that $\cC \subseteq T(d)$; 
\item \label{def:admissible_consistency}
 \textbf{partition consistent} 
 if for every $d \in \cD(\cX)$ and every partition $\cC \in \cP(\cX)$ such that $\cC \subseteq T(d)$, we have $\cC \subseteq T(d')$ for every $\cC$-strengthening $d'$ of $d$\new{\footnote{For a hierarchy represented as a rooted tree, partitions \(\cC\subseteq T(d)\) are in bijection with vertex cuts separating the root from the leaves: the vertices in the cut are precisely the nodes corresponding to the clusters in \(\cC\).}}; 
\item \textbf{permutation invariant} if $T( d_{\phi}) \weq \phi \cdot T(d)$ for every permutation $\phi$ of $\cX$, where $d_{\phi}(x,y) := d(\phi^{-1}(x),\phi^{-1}(y))$ and $\phi \cdot T(d) = \{ \{\phi(x) \colon x\in C\} \colon C\in T(d) \}$. 
\end{enumerate}
A hierarchical clustering method $T$ is \emph{admissible} if it satisfies all four axioms.
\end{definition}
Definition~\ref{def:admissible} is a direct extension of Kleinberg’s original axioms to the hierarchical setting. The flat clustering method $f$ is replaced by a hierarchical clustering method $T$, and the three axioms of Definition~\ref{def:admissible_Kleinberg} are reformulated accordingly. We also explicitly include permutation invariance, requiring that the output depends only on the underlying dissimilarity structure and not on the labeling of the points. In Kleinberg’s original framework, this requirement is implicit, as clusterings are defined directly on sets rather than on a labeled representation. 

We refer to the set of four axioms in Definition~\ref{def:admissible} as the canonical axiom system $\axiadm$, and we call \emph{admissible} the hierarchical clustering methods that satisfy $\axiadm$. 
There are, however, other reasonable axiomatizations of hierarchical clustering. In particular, there is a well-known correspondence between ultrametrics and hierarchies, which is not addressed by the axioms of $\axiadm$ because it has no direct counterpart in flat clustering. In Section~\ref{sec:ultrametric}, we study the consequences of such an additional requirement. Moreover, further alternative axioms are discussed in Appendix~\ref{app:alternative_axioms}, where we show that most of the results established for this canonical axiom system in fact remain valid under these alternative formulations. All the axioms and their alternatives are cataloged in Table~\ref{tab:axioms-summary}.

The first main result of this paper is that the hierarchical clustering problem does admit (many) methods that satisfy $\axiadm$, in contrast with flat clustering.

\begin{theorem}
\label{thm:achievability}
There exist uncountably many admissible methods.
\end{theorem}
The proof of Theorem~\ref{thm:achievability} is constructive. 
In the next section, we give explicit examples of admissible hierarchical clustering methods.

\section{Explicit Construction of Admissible Methods}
\label{sec:admissible_HCs}

 In this section, we construct admissible methods from three complementary perspectives. 
 We first examine the admissibility of some standard {\em linkage methods}, as they are the most widely used class of hierarchical clustering methods. 
 We then introduce two families of methods defined through explicit separation conditions and finally study Bryant-Berry stable clusters. 

These two parameterized families play a central role in the analysis of the class of admissible hierarchical clustering methods made in Section~\ref{sec:structure}: the separation methods will be useful to order the admissible methods based on the cluster structures they produce, whereas the Bryant-Berry cluster construction will be useful for showing the existence of many admissible methods that are fundamentally different (more precisely, incompatible).

\subsection{Linkage Methods}
\label{subsec:linkage}

Linkage-based methods form a classical and widely used class of hierarchical clustering algorithms. They are typically defined as greedy agglomerative procedures: Starting from the singleton partition, clusters are iteratively merged according to a prescribed inter-cluster dissimilarity rule, such as single, complete, average, or Ward linkage. We refer to Appendix~\ref{app:linkage} for formal definitions of these classic linkage rules. Standard formulations of linkage algorithms merge exactly two clusters at each step, and hence always produce binary hierarchies.\footnote{A hierarchy  $\Psi \subseteq 2^{\cX}$ on $\cX$ is a \emph{binary hierarchy} if, for every non-singleton cluster $C\in \Psi$, there exist exactly two proper subsets $C_1,C_2\in \Psi$ such that: (i) $C_1,C_2 \subsetneq C$; (ii) $C_1\cap C_2 = \emptyset$; (iii) $C_1 \cup C_2 = C$. 
}

We first observe that this restriction to pairwise merges and hence to binary hierarchies is incompatible with permutation invariance, regardless of the specific linkage rule employed.

\begin{lemma}
\label{lem:binary_not_perm}
Any hierarchical clustering method that is constrained to always produce a binary hierarchy is not permutation invariant.
\end{lemma}

\begin{proof}
Given a dissimilarity $d$, we say that two elements $x \ne y$ have pairwise identical dissimilarity profiles if $x$ and $y$ are equidistant from every other element of $\cX$ (i.e., \(\forall z \in \cX \setminus \{x,y\} \colon \quad d(x,z) = d(y,z)\)). 
Let $d$ be a dissimilarity on $\cX$, and suppose that $\cX$ has three distinct elements $x_1$, $x_2$ and $x_3$, such that any two have pairwise identical profiles.  As the elements of \(S:=\{x_1,x_2,x_3\}\) have pairwise identical dissimilarity profiles, every permutation \(\sigma\) of \(S\) that fixes \(\cX\setminus S\) preserves \(d\). Permutation invariance therefore requires that $C\in T(d) \Longrightarrow \sigma(C)\in T(d).$

Let \(C\in T(d)\) be non-singleton and suppose that \(C\cap S\neq\varnothing\). If \(C\cap S=\{x_i\}\), then \(C=U\cup\{x_i\}\) for some nonempty \(U\), and exchanging \(x_i\) with another element of \(S\) produces a cluster
\(U\cup\{x_j\}\in T(d)\) incompatible with \(C\). 
If \(|C\cap S|=2\), exchanging one of these two elements with the third likewise produces a cluster incompatible with \(C\). 
Both cases contradict laminarity. Hence every non-singleton cluster having a non-empty intersection with \(S\) contains all elements of~\(S\), and thus $T$ must be non-binary. 
\end{proof}

Motivated by Lemma~\ref{lem:binary_not_perm}, we consider variants of linkage methods whose output is not constrained to binary hierarchies.
In these variants, whenever multiple pairs of clusters attain the same minimal inter-cluster dissimilarity, all clusters involved in this tie are merged together simultaneously. This produces hierarchies that are not necessarily binary, but that respect the symmetries of the input dissimilarity function. 
Among the classic linkage rules mentioned earlier, single linkage is the only one that produces admissible hierarchies, as established by the following proposition. 

\begin{proposition}[Admissibility of Linkage Methods]
\label{prop:linkage_admissibility}
The non-binary variant of the single linkage method $\Tsl$ is admissible. 
In contrast, non-binary variants of the complete, average, Ward, centroid, and median linkage violate at least one of the axioms in Definition~\ref{def:admissible}. 
\end{proposition}
\begin{proof}
We prove the admissibility of $\Tsl$ in Appendix~\ref{app:proofs_single_linkage}, and the non-admissibility of other linkage methods in Appendix~\ref{app:linkage}.    
\end{proof}

\subsection{Separation Methods}
\label{subsec:separation_methods}
We now introduce two classes of admissible hierarchical clustering methods, based on explicit separation conditions between within-cluster and cross-cluster dissimilarities. 
These methods declare a nonempty subset $C\subseteq\cX$ to be a cluster if the internal dissimilarities within $C$ are sufficiently small compared to the dissimilarities between elements of $C$ and of its complement~$\bC$. 
The resulting hierarchies are defined directly from $d$, without resorting to an agglomerative procedure. 
We quantify the separability of a cluster~$C$ by the ratios
\[
\frac{d(x_1,y)}{d(x_2,z)},
\qquad x_1,x_2,y\in C,\ z\notin C,
\]
and we consider below two families of methods resulting from thresholding these ratios. The two families differ in the choice of the reference points $x_1, x_2 \in C$: the first family enforces a worst-case comparison globally for any pair $x_1, x_2 \in C$, whereas the second family enforces this comparison locally by setting $x_1=x_2$.

\begin{definition}
\label{def:separation_based}
For any finite set $\cX$, define the global and local separabilities with respect to $d\in \cD(\cX)$ and a proper subset $C \subsetneq \cX$ with $|C|\ge 2$, respectively by
\begin{align*}
 \varrho(d,C)\, := \, \max_{ \substack{ x_1,x_2,y \in C, z\in \bC} } \frac{d(x_1,y)}{ d(x_2,z)}
\quad \text{ and } \quad  
\tau(d,C) \, := \,\max_{ \substack{ x,y\in C, z \in \bC} }
\frac{ d(x,y)}
{d(x,z)}.
\end{align*}
We call \emph{separation margin sequence} any sequence $\bdeta=(\eta_m)_{1 \le m \le n-2}$ such that $0<\eta_s\le 1$ for every $s$. For such sequence $\bdeta$, we define the set of $\bdeta$-globally and $\bdeta$-locally separated clusters on $\cX$, respectively as
\begin{align*}
\Tglob^{\bdeta}(d) &:= \left\{ C\subsetneq\cX : |C|\ge 2, \, \varrho(d,C) < \eta_{|C|-1} \right\} \cup \{\{x\}: x \in \cX\} \cup \{\cX\}; \\ 
\Tloc^{\bdeta}(d) &:= \left\{ C\subsetneq\cX : |C|\ge 2, \, \tau(d,C) < \eta_{|C|-1} \right\} \cup \{\{x\}: x \in \cX\} \cup \{\cX\}. 
\end{align*}
\end{definition}
The global (resp., local) separability $\varrho(d,C)$ (resp., $\tau(d,C)$) measures the sensitivity of the global (resp., local) separability of the cluster $C$ to the dissimilarity function $d$. 
The following proposition establishes that the maps $\Tglob^{\bdeta} \colon d \mapsto \Tglob^{\bdeta}(d)$ and $\Tloc^{\bdeta}\colon d \mapsto \Tloc^{\bdeta}(d)$ are admissible hierarchical clustering methods. 

\begin{proposition}
\label{prop:admissible_separation-based}
For any separation margin sequence $\bdeta$, $\Tglob^{\bdeta}$ and $\Tloc^{\bdeta}$ are admissible. Furthermore, if $\bdeta \neq \bdeta'$, 
then $\Tglob^{\bdeta} \neq \Tglob^{\bdeta'}$ and $\Tloc^{\bdeta} \neq \Tloc^{\bdeta'}$.
\end{proposition}

When $\bdeta = \bone = (1,\cdots,1)$, we simply write $\Tglob, \Tloc$ instead of $\Tglob^{\bone}, \Tloc^{\bone}$. The clusters belonging to $\Tglob(d)$ and $\Tloc(d)$ are the sets of points that are strictly closer to any point in the set than to any point outside the set. 
Among these methods, the most studied one is $\Tloc(d)$, which is called the \emph{Apresjan hierarchy}~\citep{apresjan1966algorithm}.
The clusters belonging to $\Tloc(d)$ have been studied independently by different authors, and are referred to as Apresjan clusters, $K$-clumps, strong clusters, nice clusters, or valid clusters in the literature \citep{ackerman2014incremental,diatta1994apresjan,balcan2008discriminative,bryant2001structured}. 

The separation margins $\bdeta$ control the sensitivity to separability $\varrho(d,C)$ and $\tau(d,C)$. 
The higher the values of $\bdeta$ are, the more sensitive these methods become to close-to-one separabilities, but at the price of being too prone to noise. 
For example, consider $\cX = \{1,2,3,4\}$ and $d\in \cD(\cX)$ such that every pairwise dissimilarity is exactly equal to 1. For this~$d$, because no nontrivial subset satisfies the global separation condition, $\Tglob$ identifies only the root and leaves as clusters. 
However, if $d(1,2)$ is slightly decreased to be equal to $1 - \varepsilon$, even for an extremely small $\varepsilon>0$ such as $\varepsilon= 10^{-10}$, then $\{1,2\}$ belongs to the hierarchy returned by $\Tglob$. But, setting $\bdeta = (1-\varepsilon') \cdot \mathbf{1}$ with $\varepsilon'>\varepsilon$, the cluster $\{1,2\}$ does not belong to the hierarchy $\Tglob^{\bdeta}$, as $\Tglob^{\bdeta}$ outputs only the root and leaves. 

\begin{example}
\label{ex:adm_methods}
Let $\cX = \{1,2,3,4,5\}$ and  
\begin{align}
d = 
 \begin{pmatrix}
 0 & 1 & 2 & 4 & 3 \\
 1 & 0 & 4 & 5 & 6 \\
 2 & 4 & 0 & 5 & 6 \\
 4 & 5 & 5 & 0 & 7 \\
 3 & 6 & 6 & 7 & 0 
 \end{pmatrix}.
 \label{eq:ex_d}
 \end{align}
Then, as shown in Figure~\ref{fig:ex_hierarchies}, the outputs of $\Tsl,\Tloc,\Tglob$, and $\Tglob^{\bdeta}$ with $\bdeta = \frac{1}{2} \boldsymbol{1}$ are \(
\Tsl(d) = \{ \{1,2,3,5\}, \{1,2,3 \}, \{1,2\}\} \cup \{\cX\} \cup \{\{x\}: x\in \cX\};
\) \(
\Tloc(d) = \{\{1,2,3 \}, \{1,2\}\} \cup \{\cX\} \cup \{\{x\}: x\in \cX\};
\) \(
\Tglob(d) = \{\{1,2\}\} \cup \{\cX\} \cup \{\{x\}: x\in \cX\};
\) \(
\Tglob^{\bdeta}(d) = \{\{\cX\} \cup \{\{x\}: x\in \cX\}.
\)
\begin{figure}[!ht]
\centering
\input{drawings/ex_hcs}
\caption{Output of $\Tsl,\Tloc,\Tglob,$ and $\Tglob^{\bdeta}$ with $\bdeta = \frac{1}{2} \boldsymbol{1}$ on the dissimilarity $d$ given in Equation~\eqref{eq:ex_d}. 
}
\label{fig:ex_hierarchies}
\end{figure}
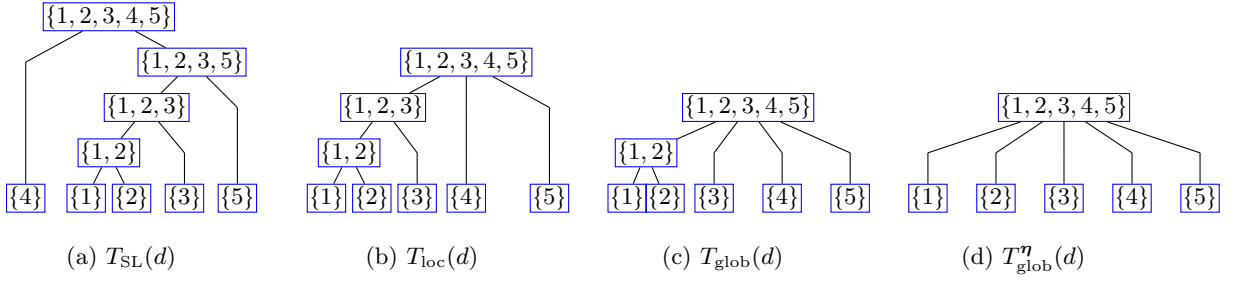
\end{example}

\subsection{Bryant-Berry Stable Clusters}
\citet{bryant2001structured} introduced the notion of \emph{stable clusters}, a combinatorial criterion that is strictly weaker than the separation conditions in Definition~\ref{def:separation_based}, yet still yields sets of clusters that form hierarchies. In the following, we recall their construction, by translating the similarity framework in \citet{bryant2001structured} to the dissimilarity-based framework used in this paper. We then demonstrate that the induced hierarchical clustering method is admissible.

For \(x,y,z\in\cX\), the \emph{Bryant--Berry isolation weight} of the pair $(x,y)$ with respect to $z$ is
\[
\rho_{d}(xy\mid z)
:=
\min\{d(x,z),d(y,z)\}-d(x,y).
\]
Intuitively, $\rho_d(xy \mid z) > 0$ means that the point $z$ is farther from both $x$ and $y$ than they are from each other. 
For disjoint nonempty subsets $U,V,Z \subseteq \cX$, we define the average Bryant--Berry isolation weight by
\[
\overline{\rho}_{d}(UV\mid Z) :=
\frac{1}{|U||V||Z|} \sum_{u\in U, v \in V, z\in Z} \rho_{d}(uv\mid z),
\]
and the \emph{stable-cluster index} of a proper subset $C \subseteq \cX$ with $|C|\ge 2$ by
\[
\iota^d(C)
:=
\min_{\substack{U,V\neq\emptyset,\ U\cap V=\emptyset\\ U\cup V=C\\ \emptyset \neq Z\subseteq \cX\setminus C}}
\overline{\rho}_{d}(UV\mid Z).
\]
A subset $C$ is \emph{stable} if $\iota^d(C) > 0$, i.e., if for every bipartition $(U,V)$ of $C$ and every nonempty set $Z$ external to $C$, the average isolation weight is strictly positive. The Bryant-Berry method is then\footnote{\citet{bryant2001structured} work with similarities $s$; by letting $d := M - s$ with $M \ge \max s$, the isolation weight $\rho$ is invariant, so the stable-cluster family transfers verbatim to the dissimilarity setting.} 
\[
\Tstable(d)
:=
\{\cX\}
\cup
\{\{x\}:x\in\cX\}
\cup
\{C\subsetneq\cX:|C|\ge 2, \, \iota^d(C)>0\}.
\]
\citet{bryant2001structured} show that the set of stable clusters is laminar, so $\Tstable(d)$ is indeed a hierarchy. 
Computing $\iota^d(C)$ requires minimizing over all bipartitions of $C$ and all external witness sets, and \citet{bryant2001structured} show that deciding stability is NP-hard in general. They further establish that the stable-cluster family is contained in the average-linkage hierarchy, thus providing a tractable outer approximation of their method. 

We also consider variants of this method resulting from preprocessing the dissimilarity function with power transformations. This yields the following parameterized family of admissible methods, which turns out to be useful for the analysis in the next section. 
As shown in the next proposition, the Bryant--Berry method and its power-transformation variants satisfy the admissibility axioms. 
\begin{proposition}
\label{prop:stable_cluster_is_admissible}
$\Tstable$ is admissible. Furthermore, for every $\power > 0$, the composed method $\Tstable^{(\power)} := \Tstable \circ \trfpower^{\power}$ is admissible, where $\trfpower^{\power}(d)(x,y) = (d(x,y))^\power$ for every $x,y$.
\end{proposition}

\section{The Structure of the Set of Admissible Methods}
\label{sec:structure}

This section studies the admissible class as a whole, by focusing on structural properties shared across all of its members rather than on any individual method. We establish three main insights, each developed in its own subsection. 
First, in Section~\ref{subsec:partial_order}, we highlight the \emph{diversity} of the admissible class: there exist uncountably many pairwise incompatible admissible methods, and both the width and the height of the class (suitably defined) are uncountable. Second, we show in Section~\ref{subsec:core_hierarchy} that despite their diversity, all admissible methods agree on a nontrivial \emph{common core} of conservative cluster structures. Finally, in Section~\ref{subsec:order_theory}, we study \emph{maximal} admissible methods and show that there are uncountably many of them.

\subsection{Comparing Admissible Methods: Refinement and Incompatibility}
\label{subsec:partial_order}

Section~\ref{sec:admissible_HCs} introduced several admissible methods, including a number of parameterized variants. This naturally raises the question of how restrictive the admissibility axioms are, and to what extent they constrain the class of admissible methods. 

\subsubsection{Refinement: A Partial Order on Hierarchical Clustering Methods}
Addressing the questions above requires comparing hierarchical clustering methods. We begin by formalizing the intuition that some methods always return more fine-grained cluster structures than others. 

\begin{definition}
\label{def:refinement}
Let $T_1$ and $T_2$ be two hierarchical clustering methods. We say $T_2$ \emph{refines}~$T_1$ 
(or equivalently, that $T_2$ is \emph{finer} than $T_1$), 
 and denote it by $T_1 \sqsubseteq T_2$, if 
\[
 T_1(d) \subseteq T_2(d) \quad \text{for all } d \in \cD(\cX).
\]
\end{definition}

The binary relation $\sqsubseteq$ is a partial order between hierarchical clustering methods. 
We denote by $\cH(\cX)$ and $\cH_{\axiadm}(\cX)$ the sets of hierarchical clustering methods on $\cX$ and hierarchical clustering methods on $\cX$ that satisfy $\axiadm$, respectively (thus $\cH_{\axiadm}(\cX) \subsetneq \cH(\cX)$). 
Then, $(\cH(\cX),\, \sqsubseteq)$ is a partially ordered set, and so is $(\cH_{\axiadm}(\cX) ,\, \sqsubseteq)$.

\begin{example}[Known and immediate refinement relations]
\label{ex:refinement_adms_in_3}
Among the admissible methods introduced in Section~\ref{sec:admissible_HCs}, we have the following refinement relationships: 
\begin{enumerate}[label=\textup{(\roman*)},nosep]
 \item $\Tloc \sqsubseteq \Tsl$ and $\Tloc \sqsubseteq \Tstable$
 \item $\Tglob^{\bdeta} \sqsubseteq \Tloc^{\bdeta}$ for any $\bdeta$; 
 \item $\Tglob^{\bdeta'} \sqsubseteq \Tglob^{\bdeta}$ and $\Tloc^{\bdeta'} \sqsubseteq \Tloc^{\bdeta}$ for any $\bdeta' \leq \bdeta$ (where inequalities are element-wise, i.e., $\eta'_s \leq \eta_s$ for all $s$).
\end{enumerate}
Point~(i) is established by~\citet{bryant2001structured}. ($\Tloc \sqsubseteq \Tsl$ has also been re-proven by~\citet{balcan2008discriminative} and \citet{dreveton2025hierarchical}.)  
Point~(ii) holds because $\tau(d,C)\le\varrho(d,C)$ for every $C\subseteq\cX$ and $d\in\cD(\cX)$.
For point~(iii), if $\bdeta'\le\bdeta$, then $\varrho(d,C)<\eta'_{|C|-1}$ implies $\varrho(d,C)<\eta_{|C|-1}$, and likewise for $\tau(d,C)$. Hence $\Tglob^{\bdeta'}\sqsubseteq\Tglob^{\bdeta}$ and $\Tloc^{\bdeta'}\sqsubseteq\Tloc^{\bdeta}$.
\end{example}

\subsubsection{Order-theoretic Terminology}

Before analyzing $\cH_{\axiadm}(\cX)$ as a subset of the partially ordered set (poset) $(\cH(\cX) , \, \sqsubseteq)$, we first briefly recall the standard order-theoretic terminology. 

In a poset $(P,\sqsubseteq)$, and a subset $S \subseteq P$, an element $g \in S$ is a \emph{greatest element} of $S$ if $s \sqsubseteq g$ for all $s \in S$. 
An element $m \in S$ is a \emph{maximal element} of $S$ if there is no $s \in S$ such that $m \sqsubsetneq s$. 
A poset may contain multiple maximal elements but at most one greatest element; if a greatest element exists, it is necessarily the unique maximal element.
An element $a \in P$ is a \emph{lower bound} of $S$ if for every element $b \in S$, $a \sqsubseteq b$. Similarly, $a \in P$ is an \emph{upper bound} of $S$ if for every element $b \in S$, $b \sqsubseteq a$.
The \emph{infimum} (i.e., the greatest lower bound) of the subset $S$ is the lower bound $a\in P$ of $S$ such that $b \sqsubseteq a$ for any lower bound $b$ of $S$ in $P$. The \emph{supremum} (i.e., the least upper bound or join) of the subset $S$ is the upper bound $a \in P$ of $S$ satisfying $a \sqsubseteq b$ for any upper bound $b$ of $S$ in $P$. 

Furthermore, two elements $x, y \in P$ are \emph{comparable} if either $x \sqsubseteq y$ or $y \sqsubseteq x$; otherwise, they are incomparable. 
A \emph{chain} is a subset $C \subseteq P$ in which every pair of distinct elements is comparable. 
The \emph{height} of the poset $P$ is the supremum of the cardinalities of all chains. 

A nonempty subset \(S\subseteq P\) is \emph{upward directed} if, for every \(x,y\in S\), there exists \(z\in S\) such that \(x\sqsubseteq z\) and \(y\sqsubseteq z\). Equivalently, every finite nonempty subset of \(S\) has an upper bound that belongs to \(S\). Every chain is upward directed, but an upward-directed family need not be totally ordered.

We say that two elements $x, y \in P$ are \emph{upward compatible} if they share a common upper bound in $P$; otherwise, they are said to be (upward) incompatible. When analyzing the structure of subsets within $P$, these concepts of comparability and compatibility give rise to different types of antichains. A standard (or weak) \emph{antichain} is a subset $A \subseteq P$ in which no two distinct elements are comparable. The quantity used to capture the maximal size of these mutually incomparable sets is the \emph{width} of the poset, defined as the supremum of the cardinalities of all standard antichains in $P$.

Besides incomparability, a stricter structural condition requires subsets to be mutually incompatible. A \emph{strong upward antichain} is a subset $A \subseteq P$ where no two distinct elements are upward compatible (for any $x \neq y \in A$, there is no $z \in P$ such that $x \sqsubseteq z$ and $y \sqsubseteq z$). Analogous to how width measures the maximum size of standard antichains, the \emph{cellularity} of a poset is defined as the supremum of the cardinalities of all its strong upward antichains.

\subsubsection{Diversity of Admissible Methods}
\label{subsec:nonuniqueness}

\begin{lemma}
\label{lem:incompatible_family}
The family of admissible methods
\(
\left\{\Tstable^{(\power)}:\power>0 \right\} \subsetneq \cH_{\axiadm}(\cX)
\)
is pairwise incompatible and thus forms a strong upward antichain.
\end{lemma}

Lemma~\ref{lem:incompatible_family} reveals that $\cH_{\axiadm}(\cX)$ is genuinely diverse, as it shows that no single admissible method refines all the others: not only do admissible methods differ, but uncountably many pairs admit no common refinement (or upper bound in the order-theoretic terminology) even in the larger class $\cH(\cX)$. In particular, $(\cH_{\axiadm}(\cX),\sqsubseteq)$ has no greatest element: there is no universal admissible method that simultaneously refines all others. We obtain the following theorem. 

\begin{theorem}
The height, width, and cellularity of $\cH_{\axiadm}(\cX)$ are uncountable.
\end{theorem}
\begin{proof}
Uncountable height follows immediately from the family of admissible methods \(\left\{\Tloc^{\bdeta}: \eta \in (0,1], \bdeta = \eta \cdot \boldsymbol{1} \right\}\), which is itself uncountable. 
Uncountable width and cellularity hold because the set \(\left\{\Tstable^{(p)} : p > 0 \right\}\) is an uncountable set of admissible methods whose elements are pairwise incompatible by Lemma~\ref{lem:incompatible_family}, and therefore mutually incomparable. 
Finally, we show that $\Tsl$ and $\Tstable^{(\power)}$ with any $\power >0$ are incompatible in Appendix~\ref{app:incompatibility}.
\end{proof}

\subsection{Uniformity: A Well-separated Backbone}
\label{subsec:core_hierarchy}
Section~\ref{subsec:nonuniqueness} established the diversity of the set of admissible methods $\cH_{\axiadm}(\cX)$.
This raises another natural question: do admissible methods nevertheless share some common cluster structure?
As a motivating example, consider $\Tsl$ and $\Tstable$.
Both refine $\Tloc$: as noted in Section~\ref{subsec:partial_order}, $\Tloc \sqsubseteq \Tsl$ and $\Tloc \sqsubseteq \Tstable$.
Hence, despite being incompatible, $\Tsl$ and $\Tstable$ share $\Tloc$ as common cluster structure. 
The following theorem shows that this is not a coincidence specific to $\Tsl$ and $\Tstable$: every method in $\cH_{\axiadm}(\cX)$ refines a globally separated hierarchy $\Tglob^{\bdeta}$ for some margin sequence $\bdeta$.

\begin{theorem}[Backbone hierarchy]
\label{thm:core_hierarchy}
 For any $T\in \cH_{\axiadm}(\cX)$, there exists a separation margin sequence $\bdeta$ (that may depend on $T$ and on $|\cX|$) such that $\Tglob^{\bdeta} \sqsubseteq T$. 
\end{theorem}

Although admissible methods may disagree on individual clusters, Theorem~\ref{thm:core_hierarchy} reveals a structural common ground: every admissible method refines a globally $\bdeta$-separated hierarchy $\Tglob^{\bdeta}$. 

Whereas the separation margin sequence $\bdeta$ is method-dependent, the following corollary shows that any finite collection of admissible methods admits a common backbone hierarchy, obtained by taking the pointwise minimum of their individual margin sequences.

\begin{corollary}
\label{cor:common_backbone}
Let $T_1,\cdots, T_m$ be methods in $\cH_{\axiadm}(\cX)$, with $m$ finite. There exists a separation margin sequence $\bdeta$ such that $\Tglob^{\bdeta} \sqsubseteq T_{\ell}$ for all $\ell \in [m]$. 
\end{corollary}
\begin{proof}
For each $\ell \in [m]$, Theorem~\ref{thm:core_hierarchy} provides a separation margin sequence $\bdeta^{(\ell)} = (\eta_s^{(\ell)})_{1\leq s \leq n-2}$ such that $\Tglob^{\bdeta^{(\ell)}} \sqsubseteq T_{\ell}$. Define $\bdeta = (\eta_s)_{1\leq s \leq n-2}$ by
\( \eta_s := \min_{\ell \in [m]} \eta_s^{(\ell)},
\) 
which satisfies $\eta_s \in (0,1]$ for every $s$ since the minimum is taken over a finite set of positive values. By construction, $\eta_s \leq \eta_s^{(\ell)}$ for all $\ell \in [m]$ and all $s \geq 1$, so $\Tglob^{\bdeta} \sqsubseteq \Tglob^{\bdeta^{(\ell)}} \sqsubseteq T_{\ell}$.
\end{proof}

The preceding corollary shows that every finite family of admissible methods has an admissible common lower bound. 
This conclusion does not, however, extend to the entire admissible class. To make this precise, define the \emph{trivial method}
\(T_{\mathrm{triv}}\in\cH(\cX)\) by
\begin{align}
\label{eq:def_Ttriv}
T_{\mathrm{triv}}(d) := \{\cX\}\cup\bigl\{\{x\}:x\in\cX\bigr\}
\qquad \forall d\in\cD(\cX).
\end{align}
In other words, \(T_{\mathrm{triv}}\) is the method that always returns only the root and the singleton leaves. 
The following proposition shows that \(T_{\mathrm{triv}}\) is the infimum of $\cH_{\axiadm}(\cX)$. Moreover, $T_{\mathrm{triv}}$ does not satisfy partition richness and hence is not admissible. Hence Proposition~\ref{proposition:least_element} also implies that \((\cH_{\axiadm}(\cX),\sqsubseteq)\) has no least element.

\begin{proposition}
\label{proposition:least_element}
The infimum of \(\cH_{\axiadm}(\cX)\), computed in the poset \((\cH(\cX),\sqsubseteq)\), is \(T_{\mathrm{triv}}\). 
\end{proposition}

\begin{proof}
As every hierarchy on \(\cX\) contains the root and all singleton clusters, $T_{\mathrm{triv}}\sqsubseteq T$ for every $T\in\cH(\cX)$. 
In particular, \(T_{\mathrm{triv}}\) is a lower bound of \(\cH_{\axiadm}(\cX)\).

We now show that no nontrivial proper cluster belongs to every admissible method. Fix \(d\in\cD(\cX)\) and a nontrivial proper subset \(C\subsetneq\cX\). Since \(\varrho(d,C)>0\), we may choose a separation margin sequence \(\bdeta\) such that $0<\eta_{|C|-1}\leq \varrho(d,C).$ 
Then, by definition, $C\notin \Tglob^{\bdeta}(d)$, whereas \(\Tglob^{\bdeta}\) is admissible. Hence \(C\) does not belong to the intersection of all admissible methods. 
Because this holds for every \(d\in\cD(\cX)\) and every nontrivial proper subset \(C\subsetneq\cX\), and because the root and singleton clusters belong to every hierarchy, for every $d \in \cD(\cX)$
\(
\bigcap_{T\in\cH_{\axiadm}(\cX)}T(d) = T_{\mathrm{triv}}(d).
\)
Therefore \(T_{\mathrm{triv}}\) is the infimum of \(\cH_{\axiadm}(\cX)\) in \(\cH(\cX)\). 
\end{proof}

Theorem~\ref{thm:core_hierarchy} carries an additional, perhaps unexpected, implication: the partition richness axiom can be upgraded to a seemingly stronger one, \emph{hierarchical richness}. Recall that partition richness only requires that for every partition $\cC$, there exists a dissimilarity $d$ such that $\cC \subseteq T(d)$. Combined with the other axioms of $\axiadm$, it in fact implies that $T$ satisfies hierarchical richness, that is, for any $\Psi \in \cT(\cX)$, there exists $d\in \cD(\cX)$ such that $\Psi \subseteq T(d)$.\footnote{This is a direct consequence of $T$ refining $\Tglob^{\bdeta}$ for some $\bdeta$, and of $\Tglob^{\bdeta} \colon \cD(\cX) \to \cT(\cX)$ being surjective.} 

\subsection{Existence of Uncountably Many Maximal Admissible Methods}
\label{subsec:order_theory}

\subsubsection{Intersection and Union of Methods}

Corollary~\ref{cor:common_backbone} shows that every finite collection of admissible methods has a lower bound in $\cH_{\axiadm}(\cX)$, whereas by Proposition~\ref{proposition:least_element} an infinite collection of admissible methods may not admit a lower bound in $\cH_{\axiadm}(\cX)$. 
There are, however, particular cases for which the infimum and the supremum, of an arbitrary, possibly uncountable, collection of admissible methods remain in $\cH_{\axiadm}(\cX)$. 
We introduce intersections and unions of hierarchical clustering methods and determine when these operations remain hierarchical clustering methods and preserve admissibility.

\begin{definition}
\label{def:joint_union_methods}
For two hierarchical clustering methods $T_1, T_2 \in \cH(\cX)$, define their intersection $T_1 \sqcap T_2$ and union $T_1 \sqcup T_2$ by 
\[
  (T_1\sqcap T_2)(d) := T_1(d)\cap T_2(d)
  \quad \text{ and } \quad 
    (T_1\sqcup T_2)(d) := T_1(d)\cup T_2(d)
    \qquad \forall d \in \cD(\cX).
\]
More generally, for any nonempty subset $L \subseteq \cH(\cX)$ of hierarchical clustering methods, define $T_{\sqcap L}$ and $T_{\sqcup L}$ by 
\[
T_{\sqcap L}(d) \;:=\; \bigcap_{T_i \in L} T_i(d) \quad \text{and} \quad T_{\sqcup L}(d) \;:=\; \bigcup_{T_i \in L} T_i(d)
    \qquad \forall d \in \cD(\cX).
\]
\end{definition}
For any dissimilarity function $d$ and any nonempty $L \subseteq \cH(\cX)$, $T_{\sqcap L}(d)$ is always a hierarchy, as the intersection of laminar families remains laminar and still contains $\cX$ and all singletons. In contrast, $T_{\sqcup L}(d)$ need not be laminar, and therefore need not define a hierarchy in general. However, closure under union holds under an additional compatibility assumption, as captured by the following lemma.

\begin{lemma}
\label{lem:hierarchy_intersection_union}
Let $L \subseteq \cH(\cX)$ be a nonempty set of hierarchical clustering methods.
\begin{enumerate}[label=\textup{(\roman*)},nosep]
\item $T_{\sqcap L} \in \cH(\cX)$.
\item $T_{\sqcup L} \in \cH(\cX)$ if and only if the methods in $L$ are pairwise compatible, i.e., for all $T_i, T_j \in L$, $T_i \sqcup T_j \in \cH(\cX)$.
\end{enumerate}
\end{lemma}

\begin{proof}
(i) For any $d \in \cD(\cX)$, each $T(d)$ is a laminar family containing $\cX$ and all singletons. The intersection $T_{\sqcap L}(d) = \bigcap_{T \in L} T(d)$ still contains $\cX$ and all singletons, and remains laminar since laminarity is a pairwise condition preserved under intersection.

(ii) If the methods in $L$ are pairwise compatible, then for any $d \in \cD(\cX)$ and any $A, B \in T_{\sqcup L}(d)$, there exist $T_i, T_j \in L$ with $A \in T_i(d)$ and $B \in T_j(d)$. Pairwise compatibility implies that $A, B$ are laminar. Together with $\cX$ and all singletons being in $T_{\sqcup L}(d)$, this proves $T_{\sqcup L} \in \cH$. Conversely, if $T_{\sqcup L} \in \cH$, then for any $T_i, T_j \in L$, $T_i \sqcup T_j \subseteq T_{\sqcup L}$ inherits laminarity, so $T_i \sqcup T_j \in \cH(\cX)$.
\end{proof}

Lemma~\ref{lem:hierarchy_intersection_union} settles the structural question of when $T_{\sqcap L}$ and $T_{\sqcup L}$ define hierarchies. We now turn to the axiomatic question: under what conditions are these operations closed within the admissible class $\cH_{\axiadm}(\cX)$?

\begin{lemma}
\label{lem:admissible_intersection_union}
Let $L \subseteq \cH_{\axiadm}(\cX)$ be nonempty.
\begin{enumerate}[label=\textup{(\roman*)},nosep]
\item If $L$ is finite, then $T_{\sqcap L} \in \cH_{\axiadm}(\cX)$. 
\item If \(L\) is upward directed under \(\sqsubseteq\), then \(T_{\sqcup L}\in\cH_{\axiadm}(\cX)\). 
\end{enumerate}
\end{lemma}
Observe that finiteness of $L$ is required to ensure $T_{\sqcap L} \in \cH_{\axiadm}(\cX)$. As a simple counterexample, consider $L = \{ \Tglob^{\alpha \boldsymbol{1}} \colon \alpha \in (0,1] \}$. Then $L$ is an infinite subset of $\cH_{\axiadm}(\cX)$, and $T_{\sqcap L} = T_{\mathrm{triv}}$ is the trivial method defined in Equation~\eqref{eq:def_Ttriv}, which is not admissible. 

A partially ordered set admitting an infimum (resp.\ supremum) for every finite nonempty subset is called a \emph{meet-semilattice} (resp.\ \emph{join-semilattice}). Lemma~\ref{lem:admissible_intersection_union} thus implies that $(\cH_{\axiadm}(\cX), \sqsubseteq)$ is a meet-semilattice but not a join-semilattice. 

\subsubsection{Existence of Maximal Admissible Methods}
A consequence of Lemma~\ref{lem:incompatible_family} is that $(\cH_{\axiadm}(\cX), \sqsubseteq)$ has no greatest element and does not even admit a supremum in $\cH(\cX)$. Nevertheless, closure under unions of chains yields, via Zorn's lemma, the existence of \emph{maximal} elements.

\begin{theorem}
\label{thm:no_unique_finest}
The set $\cH_{\axiadm}(\cX)$ has maximal elements, and every $T \in \cH_{\axiadm}(\cX)$ is refined by at least one maximal element $T^M \in \cH_{\axiadm}(\cX)$.
Moreover, $\cH_{\axiadm}(\cX)$ contains an uncountable family of pairwise-incompatible maximal elements.
\end{theorem}

\begin{proof}
We apply Zorn's lemma (see, e.g., \citet[Theorem 30]{shen2002basic}). Let $L \subseteq \cH_{\axiadm}(\cX)$ be any chain (i.e., a subset totally ordered by $\sqsubseteq$). We show that $T_{\sqcup L} \in \cH_{\axiadm}(\cX)$, providing an upper bound for $L$ in $\cH_{\axiadm}(\cX)$.

Because every chain is upward directed, Lemma~\ref{lem:admissible_intersection_union}(ii) yields that \(T_{\sqcup L}\in\cH_{\axiadm}(\cX)\). By Zorn's lemma, $\cH_{\axiadm}(\cX)$ has maximal elements, and every $T \in \cH_{\axiadm}(\cX)$ is refined by some maximal element.

For cardinality and structural incompatibility, consider the uncountable family $\{\Tstable^{(p)}:p>0\}$ from Lemma~\ref{lem:incompatible_family}. Extend each $\Tstable^{(p)}$ to a maximal method $M_p$. If $p\ne q$, then there is no hierarchy-valued method that can refine both $M_p$ and $M_q$, because such a method would also refine both $\Tstable^{(p)}$ and $\Tstable^{(q)}$. Thus $\{M_p:p>0\}$ is an uncountable pairwise incompatible family of maximal elements.
\end{proof}

Incompatibility arises when two methods select clusters that cannot coexist in a single hierarchy for some input. Conversely, any pairwise-compatible family of methods can be combined through their union, which remains hierarchy-valued.

Although $\cH_{\axiadm}(\cX)$ has no greatest element, Theorem~\ref{thm:no_unique_finest} guarantees that every admissible method is refined by a maximal admissible method that cannot be further refined within $\cH_{\axiadm}(\cX)$. 

Figure~\ref{fig:hasse_diagram_sec4} summarizes the results of this section on the structure of $\cH_{\axiadm}(\cX)$.

\begin{figure}[!ht]
\centering
\resizebox{0.5\linewidth}{!}{%
\input{drawings/hasse_combined}
}
\vspace{-20pt}
\caption{
Hasse diagrams of a subset of $\cH_{\axiadm}(\cX)$ under the refinement order $\sqsubseteq$, where transitive edges are omitted and an arrow from $T'$ to $T$ indicates that $T' \sqsubseteq T$. 
The blue solid and red dashed boundaries delimit $\cH_{\axiadm}(\cX)$ and $\cH(\cX)$, respectively. The superscript $M$ denotes the maximal admissible element of each chain.
In particular, the diagram includes the strong upward antichain $\{\Tstable^{(\power)}\}_{\power \in [n-2]}$,
and chains $\{\Tglob^{\bdeta^{(i)}}\}_{i \in [n-2]}$, and $\{\Tloc^{\bdeta^{(i)}}\}_{i \in [n-2]}$ with $\bdeta^{(i)} = \eta^i \mathbf{1}$, for some $\eta \in (0,1)$. 
}
\label{fig:hasse_diagram_sec4}
\end{figure}
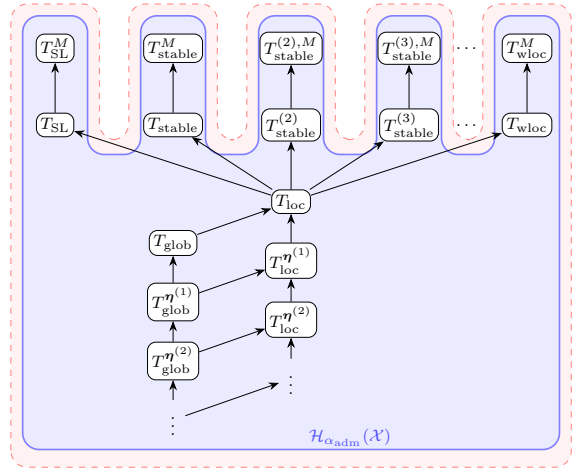

\section{Extensions of the Canonical Framework}
\label{sec:extension}

The previous sections study hierarchical clustering under the canonical axiom system. 
We now examine the robustness of this framework in two directions that are specific to hierarchical clustering and its practical use. First, we add exactness on ultrametric inputs, a fidelity requirement with no direct analog in flat clustering. Second, we allow the input dissimilarity to be preprocessed before the hierarchy is constructed and ask which axioms are preserved under composition.

\subsection{Exactness on Ultrametric Inputs}
\label{sec:ultrametric}

The admissibility axioms closely follow Kleinberg's original setting. 
However, in the particular case where the input distance is an ultrametric, hierarchical clustering induces a known correspondence with dendrograms, which motivates an additional requirement tailored to hierarchical outputs.
Recall that a distance $d$ is an \emph{ultrametric} if $d$ satisfies the strong triangle inequality
 \[
 d(x,y) \wle \max \{ d(x,z), d(y,z) \} \quad \forall x,y,z \in \cX,
 \]
 and that a \emph{dendrogram} $(\Psi, \theta)$ is a hierarchy $\Psi \in \cT$ in which each node is assigned height through the dendrogram's height function $\theta \colon \Psi \to \R_+$, with leaves (singletons) having height $0$. 
 The function $\theta$ satisfies $\theta(\{x\}) = 0$ for every $x \in \cX$ and is strictly increasing from the leaves towards the root, that is, $\theta(C)<\theta(C')$ for every $C\subsetneq C'$. 

Ultrametric distances and dendrograms are in one-to-one correspondence (see \textit{e.g.,} \cite{carlsson2010characterization}): for every ultrametric~$u$, there exists a unique dendrogram $(\Psi_u,\theta_u)$ such that  $u(x,y)=\theta_u(\lca_{\Psi_u}(x,y))$ for all $x,y\in\cX$, where $\lca_{\Psi_u}(x,y)$ denotes the smallest cluster in~$\Psi_u$ containing both $x$ and $y$.
Moreover, the hierarchy $\Psi_u$ can be constructed explicitly from the ultrametric $u$ by 
\begin{align}
\label{eq:ultrametric_hierarchy}
\Psi_u=\{B_u(x,r): x\in\cX,\ r\ge 0\}
\quad \text{ where }
B_u(x,r):=\{y\in\cX: u(x,y)\le r\},
\end{align}
and the dendrogram's height function $\theta_u$ is given by
\(
\theta_u(C) = \max_{x,y  \in C} u(x,y).
\) 

This correspondence motivates an additional fidelity requirement: when the input is already an ultrametric, the method should recover its canonical hierarchy (albeit we do not requires to output the associated height).

\begin{definition}
\label{def:exact_on_ultrametrics}
A hierarchical clustering method $T \in \cH(\cX)$ is \emph{exact on ultrametrics} if, for every ultrametric $u$, we have $T(u) = \Psi_{u}$.
\end{definition}
In particular, when restricted to ultrametric inputs, all methods that are exact on ultrametrics coincide. 

We say that a hierarchical clustering method is \emph{strongly admissible} if it is admissible and exact on ultrametrics. We denote this strengthened axiom system by $\axiadmp$, and write $\cH_{\axiadmp}(\cX)$ for the class of strongly admissible hierarchical clustering methods on $\cX$. 

We prove in the appendix that the non-binary variant of single linkage is exact on ultrametrics. Moreover, $\Tglob^{\bdeta}$ and $\Tloc^{\bdeta}$ are exact on ultrametrics if and only if $\bdeta = \boldsymbol{1}$. 
Finally, we also prove that for every $\power > 0$, the composed method $\Tstable^{(\power)} := \Tstable \circ \trfpower^{\power}$ is exact on ultrametrics. Together, these results show that there exist uncountably many strongly admissible methods. Moreover, the results of Section~\ref{sec:structure} regarding the existence of uncountably many maximal elements and of a backbone hierarchy hold when we replace $\axiadm$ by $\axiadmp$. 
There are, however, some important differences. Indeed, for every strongly admissible method $T$, we have $\Tglob \sqsubseteq T$. 
Thus, whereas $(\cH_{\axiadm}(\cX), \, \sqsubseteq)$ has no least element (Proposition~\ref{proposition:least_element}), $(\cH_{\axiadmp}(\cX), \, \sqsubseteq)$ has one least element, $\Tglob$, which is thus the coarsest strongly admissible hierarchical method. 
In practice, $\Tglob$ returns a hierarchy composed of many non-singleton clusters. To illustrate it, we provide in Appendix~\ref{app:size_Tglob} numerical experiments to evaluate the size of the hierarchy returned by $\Tglob$ on real data sets.

\begin{theorem}
The height, width, and cellularity of $\cH_{\axiadmp}(\cX)$ are uncountable. 
In particular, \(\Tsl\), \(\Tglob\), \(\Tloc\), and \(\Tstable^{(p)}\) for every \(p>0\) are strongly admissible. Moreover,
\begin{enumerate}[nosep]
 \item The backbone and maximality results of Section~4 remain valid;
 \item $\cH_{\axiadmp}(\cX)$ admits uncountably many pairwise incompatible maximal elements; 
 \item \(\Tglob\) is the least strongly admissible method under refinement.
\end{enumerate}
\end{theorem}

\subsection{Axiom-Preserving Preprocessing}
\label{subsec:composed_methods}

Modern clustering pipelines rarely apply a hierarchical clustering method directly to the raw observations or their original pairwise dissimilarities. 
Instead, the data is typically preprocessed before the hierarchy is constructed. Common examples include dimensionality reduction, feature re-weighting, and the choice or transformation of a metric (for example using a kernel). Density-aware procedures also fit this framework: they first construct a nearest-neighbor-based dissimilarity and then apply a single-linkage-type method to the transformed data. These examples explain treating preprocessing as a part of the clustering pipeline and asking which axioms are preserved by it. 

\subsubsection{General Preservation Principles}

We model preprocessing as a map $\trf\colon \cD(\cX)\to\cD(\cX)$, called a \emph{transformation}. 
Given a hierarchical clustering method \(T\), the composed method $T\circ\trf \colon d \mapsto T(\trf(d))$ first transforms the input dissimilarity and then applies $T$. Because \(\trf(d)\) is a dissimilarity on the same domain as \(d\), the composition \(T\circ\trf\) is again a hierarchical clustering method.

Our goal is to study when the axiomatic properties of \(T\) are inherited by the composed method \(T\circ\trf\). Rather than studying this question separately for each method, we lift each axiom to transformations: a transformation preserves an axiom if composition with it preserves that axiom for every method satisfying it.

\begin{definition}
\label{def:property_preserving_transformation}
A transformation $\trf \colon \cD(\cX) \to \cD(\cX)$ \emph{preserves} an axiom $\axi$ if, for every method $T$ that satisfies $\axi$, the composed method $T \circ \trf$ also satisfies $\axi$.
\end{definition}

\begin{lemma}
\label{lem:trf_preserve_axioms}
Let $\trf \colon \cD(\cX) \to \cD(\cX)$. Each of the following conditions is sufficient for \(\trf\) to preserve the corresponding axiom:
\begin{itemize}[nosep]
 \item \emph{Scale invariance}: for every \(d\in\cD(\cX)\) and every \(\beta>0\), there exists \(\beta'>0\) such that $\trf(\beta d)=\beta'\trf(d)$;
 \item \emph{Partition richness:} $\trf$ is surjective; 
 \item \emph{Partition consistency}: for every partition $\cC\in\cP(\cX)$ and every $d,d'\in\cD(\cX)$, if $d'$ is a $\cC$-strengthening of $d$, then $\trf(d')$ is a $\cC$-strengthening of \(\trf(d)\);
 \item \emph{Permutation invariance:} for every \(d\in\cD(\cX)\) and every permutation \(\phi\), $\trf(d_\phi)=\trf(d)_\phi$; 
 \item \emph{Exactness on ultrametrics:} for every ultrametric $u$, $\trf(u)$ is an ultrametric and $\Psi_{\trf(u)}=\Psi_u$. 
\end{itemize}
\end{lemma}

 Except for partition richness, each condition in Lemma~\ref{lem:trf_preserve_axioms} ensures that the input structure relevant to the corresponding axiom is preserved by the transformation: scaling remains scaling, strengthenings remain strengthenings, and permutations commute with the transformation. Similarly, ultrametrics are mapped to ultrametrics that induce the same hierarchy. 

 These observations are instances of a more general principle: the axioms above can be expressed by imposing a condition on the outputs associated with one or two dissimilarities satisfying a prescribed relation. In Appendix~\ref{app:relational_prop}, we formalize this class of relational properties and recover the corresponding parts of Lemma~\ref{lem:trf_preserve_axioms} as corollaries.

\begin{corollary}
\label{cor:monotone_power_trf}
Let \(g\colon\R_{\geq 0}\to\R_{\geq 0}\) satisfy $g(0)=0$ and $g(t)>0$ for every $t>0$. 
Define the pointwise transformation \(\trf_g\colon\cD(\cX)\to\cD(\cX)\) by
\[
\trf_g(d)(x,y):=g\bigl(d(x,y)\bigr)
\qquad
\text{for all }d\in\cD(\cX)\text{ and }x,y\in\cX.
\]
Then \(\trf_g\) preserves permutation invariance. Moreover, 
\begin{itemize}[nosep]
 \item if $g$ is strictly increasing, then $\trf_g$ preserves partition consistency and exactness on ultrametrics; 
 \item if $g$ is surjective onto $\R_{\ge 0}$, then $\trf_g$ preserves partition richness;
 \item if, for every $\beta>0$, there exists \(c_\beta>0\) such that $g(\beta t)=c_\beta g(t)$ for every $t\geq 0$, then $\trf_g$ preserves scale invariance.
\end{itemize}
In particular, for every \(p>0\), the power transformation $\trf_{\mathrm{power}}^{p}(d)(x,y):=d(x,y)^p$  preserves scale invariance, partition richness, partition consistency, permutation invariance, and exactness on ultrametrics. Thus, \(\trf_{\mathrm{power}}^{p}\) preserves both admissibility and strong admissibility. 
\end{corollary}

\subsubsection{Examples and Applications}

\emph{Minimax-path preprocessing and single linkage.} 
Our first example gives an exact preprocessing representation of single linkage. Given a dissimilarity $d\in \cD(\cX)$, consider the weighted complete graph whose vertices are the points in $\cX$ with edge-weight $d$. The bottleneck of a path $\gamma$ is the largest weight along the path, that is, $\max_{e \in \gamma} d(e)$. The \emph{minimum bottleneck} between two points $x, y \in \cX$ is 
\begin{align}
\label{eq:def_min_bottleneck}
 B^*(d)(x,y) \ := \ 
 \begin{cases}
\displaystyle
\min_{\gamma\in\Gamma_{x,y}} \max_{\{u,v\}\in\gamma} d(u,v), & x\neq y,\\[1ex]
0,  & x=y,
\end{cases}
\end{align}
where \(\Gamma_{x,y}\) denotes the set of paths from \(x\) to \(y\). 

Although the definition of $B^*$ given in~\eqref{eq:def_min_bottleneck} involves a minimum over all paths, this quantity can efficiently be computed from any minimum spanning tree \(M\). Indeed, if \(\gamma^M_{x,y}\) denotes the unique path from \(x\) to \(y\) in \(M\), then
\[
 B^*(d)(x,y) \weq \max_{\{u,v\}\in\gamma^M_{x,y}}d(u,v).
\]
The equivalence between this minimum-bottleneck construction and single linkage is standard; see, for example,~\cite{carlsson2010characterization}. 

The dissimilarity \(B^*(d)\) is an ultrametric, commonly called the \emph{subdominant ultrametric} associated with \(d\); its values are also known as minimax-path distances. Moreover, for every \(r\geq 0\), two points $x$ and $y$ satisfy $B^*(d)(x,y)\leq r$ if and only if they belong to the same connected component of the graph whose edges are the pairs \(\{u,v\}\) satisfying \(d(u,v)\leq r\). These connected components, as \(r\) varies, are exactly the clusters produced by single linkage. Hence, by the ultrametric-hierarchy correspondence detailed in Equation~\eqref{eq:ultrametric_hierarchy}, we have $\Tsl(d) \weq \Psi_{B^*(d)}.$ 
Moreover, because $\Tglob$ is exact on ultrametric inputs, we obtain $\Tsl(d) = \Tglob\bigl(B^*(d)\bigr)$, and hence 
\begin{align}
\label{eq:relationship_Tsl_Tglob}
 \Tsl\weq\Tglob\circ B^*.
\end{align}
Thus, single linkage can be viewed as first replacing the original dissimilarity by its minimax-path ultrametric and then extracting the globally separated clusters.
Observe also that the same reasoning applies to any method that is exact on ultrametric inputs; in particular, we may replace \(\Tglob\) in~\eqref{eq:relationship_Tsl_Tglob} with \(\Tloc\) or \(\Tstable\). 

Finally, the transformation $B^* \colon d \in \cD(\cX) \mapsto B^*(d)$ satisfies
\[
B^*(\beta d)=\beta B^*(d),
\qquad
B^*(d_\phi)=B^*(d)_\phi,
\quad \text{ and } \quad 
B^*(u)=u, 
\]
for every \(\beta>0\), every permutation \(\phi\), and every ultrametric \(u\). Hence, \(B^*\) preserves scale invariance, permutation invariance, and exactness on ultrametrics. Together with the corresponding properties of \(\Tglob\), the factorization \(\Tsl=\Tglob\circ B^*\) yields these properties for single linkage.

\medskip
\emph{PCA preprocessing.}
In this paragraph only, we restrict to Euclidean distances $d$, meaning that there exist points \((z_x)_{x\in\cX}\subset\R^p\) such that $d(x,y) \weq \lVert z_x-z_y\rVert_2$ for all $x,y\in\cX$. 
Center these points, that is define $\tz_x:=z_x - \bar{z}$ where $\bar{z} = \frac{1}{|\cX|}\sum_{x\in\cX}z_x$, and let \(P_r\) be the orthogonal projection onto the subspace spanned by the first \(r\) principal components of the centered configuration $(\tz_x)_{x\in\cX}$. Assuming that this principal subspace is uniquely defined, we define 
\[
\trf_{\mathrm{PCA},r}(d)(x,y)
:=
\bigl\lVert P_r\widetilde z_x-P_r\widetilde z_y \bigr\rVert_2.
\]
This definition does not depend on the chosen Euclidean realization, since centered realizations of the same dissimilarity differ only by an orthogonal transformation. Under the assumption that the projected points \(P_r\widetilde z_x\), \(x\in\cX\) are pairwise distinct,\footnote{Otherwise, two distinct points $x \ne y$ may have the same projection, in which case $\trf_{\mathrm{PCA},r}(d)(x,y) = 0$ and the construction yields a pseudodissimilarity rather than an element of \(\cD(\cX)\). Such collisions may occur in particular for highly symmetric configurations, including some Euclidean realizations of ultrametrics. In that case, the chosen projection dimension \(r\) is unsuitable if the preprocessing is required to remain dissimilarity-valued.}
$\trf_{\mathrm{PCA},r}(d)$ is a dissimilarity. 

Multiplying \(d\) by \(\beta>0\) amounts to multiplying the realizing configuration by \(\beta\), which leaves its principal subspace unchanged and multiplies all projected distances by
\(\beta\). Hence $\trf_{\mathrm{PCA},r}(\beta d) = \beta\,\trf_{\mathrm{PCA},r}(d)$,
and thus PCA preprocessing preserves scale invariance.

PCA preprocessing is also equivariant under relabelling. Indeed, permuting the points leaves the covariance operator, and hence the principal \(r\)-dimensional subspace, unchanged, while merely relabelling the projected points. Therefore, $\trf_{\mathrm{PCA},r}(d_\phi) = \trf_{\mathrm{PCA},r}(d)_\phi$, and thus PCA preprocessing also preserves permutation invariance.

 In general, however, PCA preprocessing need not preserve the remaining axioms. A partition strengthening may alter the principal subspace, so its image need not remain a strengthening of the projected dissimilarity.
 Moreover, for fixed \(r\), every output of \(\trf_{\mathrm{PCA},r}\) is a Euclidean dissimilarity realizable in dimension at most \(r\). Because a general dissimilarity need not have this form, \(\trf_{\mathrm{PCA},r}\) is not surjective onto \(\cD(\cX)\). Finally, projecting an ultrametric realization need not produce an ultrametric inducing the same hierarchy.

\medskip
\emph{HDBSCAN and mutual-reachability preprocessing.}
Fix \(k\geq 1\), and let \(d_{\mathrm{core}}(x)\) be the dissimilarity from
\(x\) to its \(k\)-th nearest neighbor. The mutual-reachability transformation
used in HDBSCAN is defined, for \(x\neq y\), by
\[
    \trfhdb(d)(x,y)
    :=
    \max\bigl\{
        d_{\mathrm{core}}(x),
        d_{\mathrm{core}}(y),
        d(x,y)
    \bigr\},
\]
with zero diagonal. Since all nearest-neighbor dissimilarities scale together
with \(d\),
\[
    \trfhdb(\beta d)=\beta\,\trfhdb(d)
    \qquad\text{for every }\beta>0.
\]
Moreover, the construction is equivariant under relabelling:
$\trfhdb(d_\phi)=\trfhdb(d)_\phi.$ 
Hence mutual-reachability preprocessing preserves both scale invariance and permutation invariance of the downstream hierarchical method.

\medskip
\emph{Similarity-valued preprocessing.} 
The same preservation principle extends to pipelines expressed in terms of similarities. For example, the Gaussian kernel
$k_\sigma(x,y) := \exp\left\{-\frac{\left(d(x,y) \right)^2}{2\sigma^2}\right\}$
is a strictly decreasing function of \(d(x,y)\). It therefore converts a dissimilarity strengthening into the corresponding similarity strengthening: within-cluster similarities increase, whereas cross-cluster similarities
decrease. The transformation also commutes with permutations. 
(Note that this example lies outside the dissimilarity-valued framework adopted above, so it should be understood using the similarity analog of the relevant axioms.)

\section{Related Work}
\label{sec:related_works}

Prior to Kleinberg, \citet{puzicha2000theory} developed an axiomatic framework for clustering objective functions, based on properties such as monotonicity, invariance, and robustness, and identified objectives satisfying these requirements. Kleinberg's impossibility theorem~\citep{kleinberg2002impossibility} subsequently motivated a broad line of work that modifies the axioms, the input, or the problem formulation in order to circumvent the impossibility result~\citep{ben2008measures,zadeh2012uniqueness,strazzeri2022possibility,
willson2024axioms}. 
Our work contributes to this line by asking whether the impossibility can be resolved by changing the form of the output, from a flat partition to a hierarchy, while retaining a dissimilarity-based input and direct hierarchical analogs of Kleinberg's axioms.

More precisely, these works relax Kleinberg's framework in different ways. 
\citet{ben2008measures} consider an analogous axiomatic setting for clustering-quality functions and establish the existence of functions satisfying their axioms. \citet{zadeh2012uniqueness} take the desired number of clusters as an additional input and correspondingly restrict richness to partitions having that number of clusters. \citet{cohen2018clustering} introduce a cost function to estimate the number of clusters and weaken
consistency by requiring it only when this estimate remains unchanged.
\citet{strazzeri2022possibility} restrict the transformations allowed by the consistency axiom and extend the setting to graph clustering.
In the graph setting, \citet{van2014axioms} adapt axioms for clustering-quality functions and introduce adaptive scale modularity, while \citet{willson2024axioms} translate Kleinberg's axioms to unweighted graphs and identify clustering methods satisfying their axioms.
A different structural perspective is developed by \citet{carlsson2013classifying}, who study clustering schemes through functoriality, requiring compatibility of clustering outputs with suitable maps between input metric spaces.

A distinct line of work develops axiomatic characterizations specific to hierarchical clustering.
\citet{carlsson2010characterization} view hierarchical clustering as a map from finite metric spaces to proximity dendrograms, equivalently ultrametrics, and characterize single linkage using normalization, separation, and functoriality. Functoriality requires every distance-nonincreasing map between input metric spaces to remain distance-nonincreasing between the corresponding output ultrametric spaces, including maps between datasets of different sizes. As distance-preserving relabelings are examples of such maps, functoriality is substantially stronger than permutation invariance. In fact, it has no direct counterpart in our framework, because it compares outputs across different ground sets and constrains their merge scales, whereas our outputs are unweighted hierarchies on a fixed set.

In a complementary direction, \citet{ackerman2010characterization} and \citet{ackerman2016characterization} study characterizations of linkage-based methods. In particular, the latter show that locality together with outer consistency characterizes hierarchical linkage methods. Outer consistency increases dissimilarities between the clusters of the partition induced by cutting a dendrogram at a selected height, while keeping within-block dissimilarities fixed. In contrast, partition consistency is thus more faithful to Kleinberg’s consistency as it allows within-block contractions and applies to every partition represented in the hierarchy. As a result, several linkage methods other than single linkage satisfy outer consistency whereas failing partition consistency.


Related axiomatic frameworks have also been developed for asymmetric dissimilarities. In particular, \citet{carlsson2014quasi} introduce hierarchical quasi-clustering for asymmetric networks and obtain a uniqueness result under their axioms, while \citet{carlsson2017admissible} characterize broader families of admissible hierarchical methods for asymmetric networks. 

 These earlier frameworks use additional structural requirements to characterize particular methods or algorithmic families. Our axioms remain closer to Kleinberg's original requirements and admit a much more diverse class of methods, whose structure under refinement is a key focus of our work. 

A complementary perspective axiomatizes objective functions for evaluating hierarchical clusterings rather than the behavior of the clustering method itself. 
\citet{dasgupta2016cost} introduced a cost function for similarity-based hierarchical clustering, thereby formulating hierarchical clustering as an optimization problem.
Building on this perspective, \citet{cohen2019hierarchical} give an axiomatic characterization of a broad class of admissible objective functions for both similarity- and dissimilarity-based hierarchical clustering, including Dasgupta's objective. This differs from our approach: their axioms determine which objective functions appropriately score a hierarchy, whereas ours directly constrain the behavior of the map from dissimilarities to hierarchies. 

Population-level approaches provide another, substantially different, axiomatic perspective on hierarchical clustering. 
\citet{thomann2015axiomatic} axiomatize hierarchical clustering of probability measures without assuming an underlying metric or dissimilarity. Their starting point is a user-specified clustering on a class of elementary measures, and they study how this clustering can be extended to more general distributions. Under suitable conditions, additivity and continuity requirements determine a unique such extension.
More recently, \citet{arias2025axiomatic} propose an axiomatic definition of hierarchical clustering based on the topology of density level sets. Rather than axiomatizing a clustering method and its behavior under transformations of the input, they seek to characterize which subsets of the support of a density should constitute population clusters. They first consider piecewise-constant densities and require clusters to have connected interior, not to split connected regions of constant density, and to be surrounded by regions of lower density. Among the cluster trees satisfying these requirements, they select the finest one, and subsequently extend the construction to more general densities, recovering Hartigan's cluster tree under suitable conditions.

Our framework is different in both its input and its object of study. In the spirit of Kleinberg, we axiomatize maps from pairwise dissimilarities to hierarchies: The axioms constrain how the output hierarchy behaves when the input dissimilarity is rescaled, strengthened, or relabeled, as well as which hierarchical structures the method can realize. Population-level approaches such as \citet{thomann2015axiomatic} and \citet{arias2025axiomatic} specify what a population hierarchy should be based on distributional or topological information, whereas ours impose structural requirements on hierarchical clustering methods that act directly on dissimilarities. 

\section{Conclusion}
\label{sec:conclusion}

Kleinberg's impossibility theorem shows that scale invariance, richness, and consistency cannot be jointly satisfied by a flat clustering method. In this work, we have shown that this incompatibility disappears when the output is allowed to be hierarchical. Natural hierarchical analogs of Kleinberg's axioms are jointly satisfiable; in fact, there exist uncountably many hierarchical clustering methods satisfying them.

Moreover, our work reveals that the set of admissible hierarchical methods is extremely complex. Under the refinement order, the class of admissible methods is a poset that has uncountable height, width, and cellularity, has neither a greatest nor a least element, and contains uncountably many pairwise incompatible maximal elements. Nevertheless, admissible methods cannot differ arbitrarily. Every admissible method contains a hierarchy of sufficiently well-separated clusters, and every finite collection of admissible methods therefore shares a nontrivial common backbone. Thus, the axioms simultaneously allow for substantial freedom in the additional clusters reported by a method while enforcing a common conservative hierarchical structure.

We also considered extensions that are specific to the hierarchical setting. Adding exactness on ultrametric inputs preserves the main structural picture while imposing a stronger connection between dissimilarities and their canonical hierarchies. In addition, we studied preprocessing transformations and identified conditions under which they preserve the axioms, allowing the framework to apply to clustering pipelines in which the input dissimilarity is transformed before the hierarchy is constructed.

This study raises several open questions. A first one is to determine which other hierarchical clustering methods are admissible, and, in particular, to identify explicit maximal admissible methods. 
Another one is to understand which additional axioms meaningfully reduce the large admissible class. Our results on ultrametric exactness provide one step in this direction, yielding, in particular, a least element among strongly admissible methods.

Finally, hierarchical clustering has an expressive advantage over flat clustering, as it does not require the number of clusters to be specified in advance and returns a richer, nested cluster structure that cannot be represented by a single partition. 
Nevertheless, some downstream tasks ultimately require a flat partition, thereby requiring a cut of the hierarchy, that is, a selection of a set of clusters from the hierarchy that forms a partition.
Kleinberg's theorem implies that no such cut-selection rule can induce a flat clustering method that simultaneously satisfies scale invariance, richness, and consistency. 
Characterizing the trade-offs inherent in cut selection, as well as the benefits of hierarchy-aware downstream procedures that avoid this projection altogether, remains an interesting direction for future work.

%% file: drawings/ex_hcs.tex
\begin{minipage}[t]{0.2\textwidth}
\centering
\begin{tikzpicture}[level distance=6mm,
  every node/.style={draw=blue, inner sep=0.8pt},
  level 1/.style={sibling distance=22mm,nodes={}},
  level 2/.style={sibling distance=12mm,nodes={}},
  level 3/.style={sibling distance=10mm,nodes={}},
  level 4/.style={sibling distance=6mm,nodes={}},
  ]
\node {\footnotesize $\{1,2,3,4,5\}$}
  child {
          child { 
            child {
                child {node {\footnotesize $\{4\}$}
                }
            }
        }
  }
  child {node {\footnotesize $\{1,2,3,5\}$}
    child {node {\footnotesize $\{1,2,3\}$}
        child {node {\footnotesize $\{1, 2\}$}
            child {node {\footnotesize $\{1\}$}}
            child {node {\footnotesize $\{2\}$}}
        }
        child {
            child {node {\footnotesize $\{3\}$}}
        }
    }
    child {
        child { 
            child {node {\footnotesize $\{5\}$}}
        }
    }
  }
  ;
\end{tikzpicture}

\vspace{2mm}
{\footnotesize (a) $\Tsl(d)$}
\end{minipage}
\hfill
\begin{minipage}[t]{0.2\textwidth}
\centering
\begin{tikzpicture}[level distance=6mm,
  every node/.style={draw=blue, inner sep=0.8pt},
  level 1/.style={sibling distance=11mm,nodes={}},
  level 2/.style={sibling distance=9mm,nodes={}},
  level 3/.style={sibling distance=6mm,nodes={}}]
\node {\footnotesize $\{1,2,3,4,5\}$}
    child {node {\footnotesize $\{1,2,3\}$}
        child {node {\footnotesize $\{1, 2\}$}
            child {node {\footnotesize $\{1\}$}}
            child {node {\footnotesize $\{2\}$}}
        }
        child {
            child {node {\footnotesize $\{3\}$}}
        }
    }
    child {
        child { 
            child {node {\footnotesize $\{4\}$}}
        }
    }
  child {
    child {
        child {node {\footnotesize $\{5\}$}}
    }
  }
  ;
\end{tikzpicture}

\vspace{2mm}
{\footnotesize (b) $\Tloc(d)$}
\end{minipage}
\hfill
\begin{minipage}[t]{0.2\textwidth}
\centering
\begin{tikzpicture}[level distance=6mm,
  every node/.style={draw=blue, inner sep=0.8pt},
  level 1/.style={sibling distance=9mm,nodes={}},
  level 2/.style={sibling distance=5mm,nodes={}}]
\node {\footnotesize $\{1,2,3,4,5\}$}
    child {node {\footnotesize $\{1, 2\}$}
        child {node {\footnotesize $\{1\}$}}
        child {node {\footnotesize $\{2\}$}}
    }
    child {
        child { 
            node {\footnotesize $\{3\}$}
        }
    }
    child {
        child { 
            node {\footnotesize $\{4\}$}
        }
    }
    child {
        child {
            node {\footnotesize $\{5\}$}
        }
    }
  ;
\end{tikzpicture}

\vspace{2mm}
{\footnotesize (c) $\Tglob(d)$}
\end{minipage}
\hfill
\begin{minipage}[t]{0.2\textwidth}
\centering
\begin{tikzpicture}[level distance=6mm,
  every node/.style={draw=blue, inner sep=0.8pt},
  level 1/.style={sibling distance=9mm,nodes={}},
  level 2/.style={sibling distance=5mm,nodes={}}]
\node {\footnotesize $\{1,2,3,4,5\}$}
    child {
        child {node {\footnotesize $\{1\}$}}}
    child {
        child {node {\footnotesize $\{2\}$}}
    }
    child {
        child { 
            node {\footnotesize $\{3\}$}
        }
    }
    child {
        child { 
            node {\footnotesize $\{4\}$}
        }
    }
    child {
        child {
            node {\footnotesize $\{5\}$}
        }
    }
  ;
\end{tikzpicture}

\vspace{2mm}
{\footnotesize (d) $\Tglob^{\bdeta}(d)$}
\end{minipage}

%% file: drawings/hasse_combined.tex
\begin{tikzpicture}[
    >=Stealth,
    every node/.style={draw, rounded corners, minimum width=4mm, minimum height=4mm, font=\scriptsize, inner sep=2pt, fill=white}
]

\begin{scope}[shift={(0,0)}]

    
    \node (TslM) at (-4.0, -1.2) {$\Tsl^M$};
    \node (Tst1M) at (-2.0, -1.2) {$\Tstable^{M}$};
    \node (Tst2M) at (0, -1.2) {$\Tstable^{(2), M}$};
    \node (Tst3M) at (2.0, -1.2) {$\Tstable^{(3), M}$};
    \node[draw=none, fill=none] (TdotsM) at (3.0, -1.2) {$\dots$};
    \node (TwaM) at (4.0, -1.2) {$\Twa^M$};
    
    \node (Tsl) at (-4.0, -2.5) {$\Tsl$};
    \node (Tst1) at (-2.0, -2.5) {$\Tstable$};
    \node (Tst2) at (0, -2.5) {$\Tstable^{(2)}$};
    \node (Tst3) at (2.0, -2.5) {$\Tstable^{(3)}$};
    \node[draw=none, fill=none] (Tdots) at (3.0, -2.5) {$\dots$};
    \node (Twa) at (4.0, -2.5) {$\Twa$};
    
    \node (Tloc) at (0, -3.8) {$\Tloc$};
    \node (Tloc1) at (0, -4.8) {$\Tloc^{\bdeta^{(1)}}$};
    \node (Tloc2) at (0, -5.8) {$\Tloc^{\bdeta^{(2)}}$};
    \node[draw=none, fill=none] (Tlocdots) at (0, -6.8) {$\vdots$};
    
    \node (Tglob) at (-2.0, -4.5) {$\Tglob$};
    \node (Tglob1) at (-2.0, -5.5) {$\Tglob^{\bdeta^{(1)}}$};
    \node (Tglob2) at (-2.0, -6.5) {$\Tglob^{\bdeta^{(2)}}$};
    \node[draw=none, fill=none] (Tglobdots) at (-2.0, -7.5) {$\vdots$};
    
    \draw[->] (Tsl) -- (TslM);
    \draw[->] (Tst1) -- (Tst1M);
    \draw[->] (Tst2) -- (Tst2M);
    \draw[->] (Tst3) -- (Tst3M);
    \draw[->] (Twa) -- (TwaM);
    
    
    \draw[->] (Tloc) -- (Tsl);
    \draw[->] (Tloc) -- (Tst1);
    \draw[->] (Tloc) -- (Tst2);
    \draw[->] (Tloc) -- (Tst3);
    \draw[->] (Tloc) -- (Twa);
    
    \draw[->] (Tloc1) -- (Tloc);
    \draw[->] (Tloc2) -- (Tloc1);
    \draw[->] (Tlocdots) -- (Tloc2);
    \draw[->] (Tglob1) -- (Tglob);
    \draw[->] (Tglob2) -- (Tglob1);
    \draw[->] (Tglobdots) -- (Tglob2);
    
    \draw[->] (Tglob) -- (Tloc);
    \draw[->] (Tglob1) -- (Tloc1);
    \draw[->] (Tglob2) -- (Tloc2);
    \draw[->] (Tglobdots) -- (Tlocdots); 
    
    
\end{scope}

\begin{scope}[on background layer]
  
  \begin{scope}[shift={(0,0)}]
    \filldraw[red!5, draw=red!50, dashed, rounded corners=8pt]
      (-4.75, -8.3) -- (-4.75, -0.45) -- (-3.25, -0.45) -- 
      (-3.25, -2.75) -- (-2.75, -2.75) -- (-2.75, -0.45) -- 
      (-1.25, -0.45) -- (-1.25, -2.75) -- (-0.75, -2.75) -- 
      (-0.75, -0.45) -- (0.75, -0.45) -- (0.75, -2.75) -- 
      (1.25, -2.75) -- (1.25, -0.45) -- (2.75, -0.45) -- 
      (2.75, -2.75) -- (3.25, -2.75) -- (3.25, -0.45) -- 
      (4.75, -0.45) -- (4.75, -8.3) -- cycle;
    \node[draw=none, fill=none, text=red!70, font=\scriptsize\bfseries] at (0, -8.6) {$\cH(\cX)$};

    \filldraw[blue!8, draw=blue!50, thick, rounded corners=8pt]
      (-4.55, -8.0) -- (-4.55, -0.65) -- (-3.45, -0.65) -- 
      (-3.45, -3.0) -- (-2.55, -3.0) -- (-2.55, -0.65) -- 
      (-1.45, -0.65) -- (-1.45, -3.0) -- (-0.55, -3.0) -- 
      (-0.55, -0.65) -- (0.55, -0.65) -- (0.55, -3.0) -- 
      (1.45, -3.0) -- (1.45, -0.65) -- (2.55, -0.65) -- 
      (2.55, -3.0) -- (3.45, -3.0) -- (3.45, -0.65) -- 
      (4.55, -0.65) -- (4.55, -8.0) -- cycle;
    \node[draw=none, fill=none, text=blue!70, font=\scriptsize\bfseries] at (1, -7.8) {$\cH_{\axiadm}(\cX)$};
    
  \end{scope}

\end{scope}

\end{tikzpicture}

%% file: appendix.tex
\section{Alternative Axioms and Robustness}
\label{app:alternative_axioms}
\label{app:generalized_structure}

Whereas the main text focuses on the canonical axiom system \(\axiadm\) and its strengthening by exactness on ultrametrics, several closely related formulations are also natural in the hierarchical setting. This appendix collects the variants that are useful for comparison, describes their logical relations, and shows that the principal structural conclusions of Section~\ref{sec:structure} are robust to many of these choices.

We denote by $\alpha$ a requirement, or a property, that may or may not be satisfied by a given hierarchical clustering method. In particular, all axioms considered in this paper are requirements. For a requirement \(\axi\), let
\[
\cH_{\axi}(\cX) := \{T\in\cH(\cX): T\text{ satisfies }\axi\}.
\]
We say that a requirement $\alpha_2$ is \emph{stronger} than another requirement $\alpha_1$, denoted $\axi_2\triangleright\axi_1$, if $\cH_{\axi_2}(\cX)\subseteq \cH_{\axi_1}(\cX).$ 
If both $\axi_2\triangleright\axi_1$ and $\axi_1\triangleright\axi_2$ hold, we say the two requirements are equivalent, and write \(\axi_1\equiv\axi_2\). 
Finally, the conjunction of $\alpha_1$ and $\alpha_2$ is denoted by \(\alpha_1 \wedge \alpha_2\). 

We retain the notation \(\axisca,\axipr,\axipc,\axiperm\) for scale invariance, partition richness, partition consistency, and permutation invariance, respectively, and \(\axisuhr\) for exactness on
ultrametrics. Thus
\[
\axiadm := \axisca\wedge\axipr\wedge\axipc\wedge\axiperm,
\qquad
\axiadmp := \axiadm\wedge\axisuhr
\ \equiv \
\axisca\wedge\axisuhr\wedge\axipc\wedge\axiperm,
\]
where the last equivalence holds because exactness on ultrametrics implies partition richness.
More generally, an \emph{axiom system} is a conjunction of requirements. 

Table~\ref{tab:axioms-summary} summarizes all the requirements/properties considered in this paper, grouped into four categories: admissibility, invariance, richness, and consistency. 

\begin{table}[!ht]
\centering
\setlength{\tabcolsep}{4pt}
\renewcommand{\arraystretch}{1.0}
\begin{tabular}{llcccc}
\toprule
\emph{Category} & \emph{Symbol / name} & \(\saxijoint\) & \( \saxiup\) & \( \saxiunion\) & \(\saxitb\) \\
\midrule
\multirow{2}{*}{Admissibility} 
&\cellcolor{lightgray}\(\axiadm\) (admissibility)        & \cellcolor{lightgray}\checkmark & \cellcolor{lightgray}\checkmark & \cellcolor{lightgray}$\times$ & \cellcolor{lightgray}$\times$ \\
& \cellcolor{magenta}\(\axiadmp\) (strong admissibility)        & \cellcolor{magenta}\checkmark & \cellcolor{magenta}\checkmark & \cellcolor{magenta}$\times$ & \cellcolor{magenta}\checkmark \\
\midrule
\multirow{3}{*}{Invariance}
& \cellcolor{lightgray}\(\axisca\) (scale invariance)        & \cellcolor{lightgray}\checkmark & \cellcolor{lightgray}\checkmark & \cellcolor{lightgray}\checkmark & \cellcolor{lightgray}\checkmark \\
& \(\axiord\) (order invariance)       & \checkmark & \checkmark & \checkmark & \checkmark \\
&\cellcolor{lightgray}\(\axiperm\) (permutation invariance)  & \cellcolor{lightgray}\checkmark & \cellcolor{lightgray}\checkmark & \cellcolor{lightgray}\checkmark & \cellcolor{lightgray}\checkmark \\
\midrule
\multirow{8}{*}{Richness variants}
& \cellcolor{lightgray}\(\axipr\) (partition richness)        & \cellcolor{lightgray}$\times$ & \cellcolor{lightgray}$\checkmark$ & \cellcolor{lightgray}\checkmark & \cellcolor{lightgray}$\times$ \\
& \(\axicwr\) (cluster-wise richness)          & $\times$ & \checkmark & \checkmark & $\times$ \\
& \(\axihr\) (hierarchical richness)          & $\times$ & \checkmark & \checkmark & $\times$ \\
& \(\axishr\) (strict hierarchical richness)  & $\times$ & $\times$ & $\times$ & $\times$ \\
& \(\axiuhr\) (refinement on ultrametrics)     & \checkmark & \checkmark & \checkmark & \checkmark \\
& \cellcolor{magenta}\(\axisuhr\) (exact on ultrametrics)         & \cellcolor{magenta}\checkmark & \cellcolor{magenta}\checkmark & \cellcolor{magenta}\checkmark & \cellcolor{magenta}\checkmark \\
& \(\axicore\) (backbone-hierarchy)    & \checkmark & \checkmark & \checkmark & $\times$\\
& \(\axicore^1\) (unit-backbone-hierarchy)    & \checkmark & \checkmark & \checkmark & $\times$\\
\midrule
\multirow{4}{*}{Consistency variants}
& \cellcolor{lightgray}\(\axipc\) (partition consistency)     & \cellcolor{lightgray}\checkmark & \cellcolor{lightgray}\checkmark & \cellcolor{lightgray}$\times$ & \cellcolor{lightgray}\checkmark \\
& \(\axicwc\) (cluster-wise consistency)  & \checkmark & \checkmark & \checkmark & \checkmark \\
& \(\axihc\) (hierarchical consistency)   & \checkmark & \checkmark & \checkmark & \checkmark\\
& \(\axiabc\) (absence consistency)       & \checkmark & \checkmark & $\times$ & \checkmark \\
\bottomrule
\end{tabular}
\caption{
Summary of the (alternative) axioms and their structural properties. 
Gray shading indicates the canonical axioms; magenta shading indicates exactness on ultrametrics and strong admissibility. Unshaded rows correspond to requirements introduced in this appendix.
A checkmark in the $\saxijoint$, $\saxiup$, or $\saxiinfuni$ column means that the corresponding method class is closed, respectively, under finite nonempty intersections, unions of nonempty upward-directed families, or unions of pairwise compatible nonempty families; a cross indicates that no such closure property is asserted. 
The three closure properties $\saxijoint$, $\saxiup$, or $\saxiinfuni$ are preserved under finite conjunctions of axioms; see Proposition~\ref{prop:individual_property_closure}. 
Similarly, a checkmark in the column $\saxitb$ means that the axiom is representable as a relational property or as a conjunction of relational properties (we refer to Section~\ref{app:relational_prop} for a precise definition and to Section~\ref{subsec:relational_axioms} for a proof of the checkmarks in the last column). 
}
\label{tab:axioms-summary}
\end{table}

\subsection{Alternative Axiom Formulations}
\label{subsec:axioms:catalog}

\subsubsection{Order invariance}

Scale invariance can be strengthened by requiring the output to depend only
on the ordering of the pairwise dissimilarities.

\begin{definition}[Order invariance \(\axiord\)]
\label{def:ax-ord}
A method \(T\in\cH(\cX)\) is \emph{order invariant} if, for every
\(d\in\cD(\cX)\) and every strictly increasing \(g:\R_{\ge0}\to\R_{\ge0}\) satisfying \(g(0)=0\), we have 
\(
T(g\circ d)=T(d),
\)
where $(g\circ d)(x,y):=g(d(x,y))$ for all elements $x,y \in \cX$. 
\end{definition}

\subsubsection{Richness variants}
\label{subsubsec:axioms:catalog:richness}

Partition richness requires that for any partition, there is a dissimilarity function $d$ that makes the partition part of the output hierarchy. In the hierarchical setting, one may instead require the realization of individual clusters or of entire hierarchies.

\begin{definition}[Richness variants]
A method \(T\in\cH(\cX)\) satisfies
\begin{itemize}[nosep]
\item \emph{cluster-wise richness} \(\axicwr\) if, for every nonempty
\(C\subseteq\cX\), there exists \(d\in\cD(\cX)\) such that \(C\in T(d)\);
\item \emph{hierarchical richness} \(\axihr\) if, for every
\(\Psi\in\cT(\cX)\), there exists \(d\in\cD(\cX)\) such that
\(\Psi\subseteq T(d)\);
\item \emph{strict hierarchical richness} \(\axishr\) if, for every
\(\Psi\in\cT(\cX)\), there exists \(d\in\cD(\cX)\) such that
\(T(d)=\Psi\);
\item \emph{refinement on ultrametrics} $\axiuhr$ if,  $\Psi_u\subseteq T(u)$ for every ultrametric~$u$.
\end{itemize}
\end{definition} 
Refinement on ultrametrics is an alternative to exactness on ultrametrics (Definition~\ref{def:exact_on_ultrametrics}; denoted $\axisuhr$), and parallel to $\axisuhr$, this can be regarded as a richness variant.

\subsubsection{Consistency variants}
\label{subsubsec:axioms:catalog:consistency}

Partition consistency preserves all blocks of a partition under a
strengthening of that partition. A natural alternative is to require the
same property cluster by cluster. To compare the two formulations, we first
extend the notion of strengthening.

\begin{definition}[\(\psi\)-strengthening]
\label{def:psi_strengthening}
Let \(\psi\subseteq 2^\cX\setminus\{\emptyset\}\) be nonempty.
A dissimilarity \(d'\in\cD(\cX)\) is a \(\psi\)-strengthening of
\(d\in\cD(\cX)\) if, for every \(C\in\psi\),
\begin{itemize}[nosep]
 \item (Intra-cluster contraction) \(d'(x,y)\le d(x,y)\) for all \(x,y\in C\);
 \item (Inter-cluster expansion) \(d'(x,y)\ge d(x,y)\) for all \(x\in C\) and \(y\notin C\);
 \item (External pairs unconstrained) Pairs with both endpoints outside \(\bigcup_{C\in\psi}C\) are unconstrained.
\end{itemize}
\end{definition}
Observe that in the case of a $\cC$-strengthening as defined in Definition~\ref{def:strengthening}, $\cC$ is a partition, so there exists no pair $x,y \notin \bigcup_{C\in \cC} C$ as $\cC$ covers~$\cX$. However, because the set $\psi$ in Definition~\ref{def:psi_strengthening} does not necessarily cover $\cX$, we make the treatment of such external pairs explicit.  

\begin{definition}[Consistency variants]
\label{def:ax-cwc}
A method \(T\in\cH(\cX)\) satisfies
\begin{itemize}[nosep]
\item \emph{cluster-wise consistency} \(\axicwc\) if, whenever
\(C\in T(d)\), every \(\{C\}\)-strengthening \(d'\) of \(d\) satisfies
\(C\in T(d')\);
\item \emph{hierarchical consistency} \(\axihc\) if, whenever
\(\psi\subseteq T(d)\), every \(\psi\)-strengthening \(d'\) of \(d\)
satisfies \(\psi\subseteq T(d')\).
\end{itemize}
\end{definition}

For completeness, one can also state partition consistency in the reverse
direction. Say that \(d'\) is a \(\cC\)-\emph{weakening} of \(d\) when
\(d\) is a \(\cC\)-strengthening of \(d'\), and call a method
\emph{absence consistent} (\(\axiabc\)) if
\[
\cC\not\subseteq T(d)
\quad\Longrightarrow\quad
\cC\not\subseteq T(d')
\]
for every \(\cC\)-weakening \(d'\) of \(d\). As shown below, this is simply
the contrapositive formulation of partition consistency.

\subsection{Relations Among the Axioms}
\label{subsec:axioms:relations}

The elementary implications among the richness and invariance requirements are $\axiord \, \triangleright \, \axisca$, 
$\axisuhr \, \triangleright \, \axishr
\, \triangleright \, \axihr
\, \triangleright \, \axipr
\, \triangleright \, \axicwr,
$
as well as $\axisuhr \, \triangleright \, \axiuhr \, \triangleright \, \axihr$. 
The following proposition shows that the consistency variants collapse more strongly than their definitions suggest.

\begin{proposition}
\label{thm:relation_consistencies}
We have $\axicwc\equiv\axihc\triangleright\axipc\equiv\axiabc$.
\end{proposition}

\begin{proof}
Hierarchical consistency implies cluster-wise consistency by taking
\(\psi=\{C\}\). Conversely, if \(d'\) is a \(\psi\)-strengthening of \(d\),
then it is a \(\{C\}\)-strengthening for every \(C\in\psi\). Cluster-wise
consistency therefore preserves every \(C\in\psi\), proving
\(\axicwc\equiv\axihc\). Taking \(\psi=\cC\) for a partition
\(\cC\subseteq T(d)\) gives \(\axihc\triangleright\axipc\).

Finally, if \(d'\) is a \(\cC\)-weakening of \(d\), then \(d\) is a
\(\cC\)-strengthening of \(d'\). Hence
\[
\cC\subseteq T(d')\Longrightarrow\cC\subseteq T(d)
\]
is precisely partition consistency applied to the pair \((d',d)\), and its
contrapositive is absence consistency.
\end{proof}

We next record the relation between richness and the backbone theorem (Theorem~\ref{thm:core_hierarchy}). Introduce
the auxiliary \emph{backbone property} \(\axicore\):
\[
T\text{ satisfies }\axicore
\quad\Longleftrightarrow\quad
\Tglob^{\bdeta}\sqsubseteq T
\text{ for some separation margin sequence }\bdeta.
\]
Since \(\Tglob^{\bdeta}\) is hierarchically rich for all separation margin $\bdeta$,
\(
\axicore \, \triangleright \, \axihr.
\)

\begin{proposition}
\label{thm:richness-relation}
Let \(\axi_1\triangleright(\axisca\wedge\axipc)\) and \(\axi_2\triangleright(\axisca\wedge\axicwc)\). Then, 
\[
(\axi_1\wedge\axipr)
\equiv
(\axi_1\wedge\axihr)
\equiv
(\axi_1\wedge\axicore), \quad \text{and} \quad (\axi_2\wedge\axicwr)
\equiv
(\axi_2\wedge\axipr)
\equiv
(\axi_2\wedge\axihr)
\equiv
(\axi_2\wedge\axicore).
\]
\end{proposition}

\begin{proof}
We prove the equivalences by the sandwiching argument i.e, by showing 
\[(\axi_1\wedge\axipr)
\triangleright
(\axi_1\wedge\axihr)
\triangleright
(\axi_1\wedge\axicore)
\quad \text{and} \quad 
(\axi_1\wedge\axipr)
\triangleleft
(\axi_1\wedge\axihr)
\triangleleft
(\axi_1\wedge\axicore). \]
At the beginning of the subsection, we already observed $\axicore \, \triangleright \, \axihr \, \triangleright \, \axipr \, \triangleright \, \axicwr$ while Lemma~\ref{lemma:core_hierarchy} gives the other direction:
\[
 \axisca\wedge\axipr\wedge\axipc \, \triangleright \, \axicore.
\]
The same argument with cluster-wise richness and cluster-wise consistency
gives
\[
\axisca\wedge\axicwr\wedge\axicwc\triangleright\axicore. 
\]
The two chains of equivalences stated in the proposition follow by sandwiching.
\end{proof}

The following corollary follows by applying Proposition~\ref{thm:richness-relation} to $\alpha_1 = \axisca\wedge\axipc\wedge\axiperm$. 
\begin{corollary}
\label{cor:replace_pr_by_core}
The canonical axiom system admits the equivalent formulation
\[
\axiadm
\equiv
\axisca\wedge\axicore\wedge\axipc\wedge\axiperm.
\]
\end{corollary}
Corollary~\ref{cor:replace_pr_by_core} is useful because partition richness itself is not naturally preserved by intersections, whereas a common positive-margin backbone is. Hence this corollary, combined with Proposition~\ref{prop:individual_property_closure}, proves Lemma~\ref{lem:admissible_intersection_union}(i). 

\subsection{Robustness of the Structural Results}
\label{subsec:general_structural_results}

We now ask to what extent the structural results of Section~\ref{sec:structure} persist under the alternative requirements introduced above. 
To state these results compactly, we introduce notation for several order-theoretic properties of the class of methods induced by an axiom system.

\begin{definition}
\label{def:structural_properties}
For an axiom system \(\axi\), we write
\begin{align*}
\axi\in\saxijoint
&\iff
T_{\sqcap L}\in\cH_\axi(\cX)
\quad\text{for every finite nonempty }
L\subseteq\cH_\axi(\cX),
\\
\axi\in\saxiup
&\iff
T_{\sqcup L}\in\cH_\axi(\cX)
\quad\text{for every nonempty upward-directed }
L\subseteq\cH_\axi(\cX),
\\
\axi\in\saxiinfuni
&\iff
T_{\sqcup L}\in\cH_\axi(\cX)
\quad\text{for every nonempty pairwise-compatible }
L\subseteq\cH_\axi(\cX),
\\
\axi\in\saxiem
&\iff
\text{every }T\in\cH_\axi(\cX)
\text{ is refined by some maximal element of }\cH_\axi(\cX).
\end{align*}
\end{definition}
Thus, \(\saxijoint\), \(\saxiup\), \(\saxiinfuni\) are the sets of requirements/properties that satisfy closure under finite intersections, upward-directed unions, and compatible unions, respectively, while \(\saxiem\) is the set of requirements/properties that admit the maximal elements of the induced poset.

Since every upward-directed family is pairwise compatible, we have 
$\saxiinfuni\subseteq\saxiup$. 
Moreover, closure under upward-directed unions implies that every refinement chain has an upper bound in the same class; hence Zorn's lemma gives 
$\saxiup\subseteq\saxiem.$ 
Therefore, 
\begin{align}
\label{eq:chain_union_results}
\saxiinfuni\subseteq\saxiup\subseteq\saxiem. 
\end{align}

In addition to \emph{refinement on ultrametrics} \(\axiuhr\) and to the \emph{backbone-hierarchy property} \(\axicore\) introduced earlier, we will use the following stronger backbone property: 
\(T\) satisfies the \emph{unit-backbone property} \(\axicore^1\) if $\Tglob\sqsubseteq T$.

\begin{proposition}[Closure principles]
\label{prop:individual_property_closure}
The closure properties indicated by checkmarks in the \(\saxijoint\), \(\saxiup\), and \(\saxiunion\) columns of Table~\ref{tab:axioms-summary} hold. 
Moreover, these closure properties are preserved under finite conjunctions: if \(\axi=\bigwedge_{i=1}^k\axi_i\), then any of the three closure properties (\(\saxijoint\), \( \saxiup\), and \( \saxiunion\)) shared by all \(\axi_i\) is also satisfied by \(\axi\). 
\end{proposition}

\begin{proof} 
We first establish the closure properties of the individual axioms, following the categories in Table~\ref{tab:axioms-summary}.

Throughout this proof, \(L\) denotes a nonempty family of hierarchical clustering methods. Whenever a union \(T_{\sqcup L}\) is considered, \(L\) is assumed to be either upward directed (when studying \(\saxiup\)) or
pairwise compatible (when studying \(\saxiinfuni\)). 
In either case, \(T_{\sqcup L}\) is hierarchy-valued by Lemma~\ref{lem:hierarchy_intersection_union}, because every upward-directed family is pairwise compatible.

\medskip \noindent\emph{Invariance properties.} Scale invariance, order invariance, and permutation invariance are preserved under both intersections and unions. For example, if every \(T\in L\) is scale invariant, then \[ T_{\sqcap L}(\beta d) = \bigcap_{T\in L}T(\beta d) = \bigcap_{T\in L}T(d) = T_{\sqcap L}(d), \] and the same calculation with unions gives \(T_{\sqcup L}(\beta d)=T_{\sqcup L}(d)\). The arguments for order and permutation invariance are identical. Therefore, 
\( \axisca,\axiord,\axiperm \in \saxijoint\cap\saxiup\cap\saxiunion. \) 

\medskip \noindent\emph{Ultrametric properties.} Refinement on ultrametrics is also preserved under intersections and compatible unions. Indeed, if every \(T\in L\) satisfies \(\Psi_u\subseteq T(u)\), then \[ \Psi_u \subseteq \bigcap_{T\in L}T(u) \subseteq \bigcup_{T\in L}T(u). \] Similarly, if every \(T\in L\) is exact on ultrametrics, then $T_{\sqcap L}(u) = T_{\sqcup L}(u) = \Psi_u.$ Hence 
\( \axiuhr,\axisuhr \in \saxijoint\cap\saxiup\cap\saxiunion. \)

\medskip \noindent\emph{Richness properties.} Cluster-wise, partition, and hierarchical richness are preserved by every nonempty compatible union. To see this, fix any \(T_0\in L\). Any dissimilarity realizing the relevant richness requirement for \(T_0\) also realizes it for \(T_{\sqcup L}\), because $T_0(d)\subseteq T_{\sqcup L}(d)$. Therefore \( \axicwr,\axipr,\axihr \in \saxiunion\subseteq\saxiup. \) The same argument does not apply to strict hierarchical richness, because additional clusters contributed by other methods in \(L\) may prevent the union from realizing a prescribed hierarchy exactly. 

\medskip \noindent\emph{Consistency properties.} 
We begin with cluster-wise consistency, so we let \(L\) be a finite nonempty family of cluster-wise consistent methods.

Suppose that \(C \in T_{\sqcap L}(d)\). Hence \(C\in T(d)\) for every \(T\in L\), and thus, for every \(\{C\}\)-strengthening \(d'\) of \(d\), we have $C\in T(d')$ for every $T\in L$. This means that \(C\in T_{\sqcap L}(d')\), which proves \(\axicwc\in\saxijoint\). 

Now suppose that \(L\) is pairwise compatible and let \(C\in T_{\sqcup L}(d)\). Then \(C\in T_0(d)\) for some \(T_0\in L\). Cluster-wise consistency of \(T_0\) ensures \(C\in T_0(d')\) for any \(\{C\}\)-strengthening \(d'\) of~$d$, and hence $C\in T_0(d')\subseteq T_{\sqcup L}(d').$ 
Thus $\axicwc\in\saxijoint\cap\saxiup\cap\saxiunion. $ 
As \(\axihc\equiv\axicwc\) by Proposition~\ref{thm:relation_consistencies}, the same conclusions hold for hierarchical consistency. 

Partition consistency is likewise preserved under intersections. Indeed, let $\cC$ be a partition such that \(\cC\subseteq T_{\sqcap L}(d)\). Then \(\cC\subseteq T(d)\) for every \(T\in L\), so every \(\cC\)-strengthening~\(d'\) of $d$ satisfies \(\cC\subseteq T(d')\) for every \(T\in L\), and therefore \(\cC\subseteq T_{\sqcap L}(d')\). 
This proves $\axipc\in\saxijoint$. 

\medskip \noindent
To prove $\axipc\in\saxiup$, an additional argument is needed because the different clusters of a partition may initially come from different methods. Let \(L\) be upward directed and suppose $\cC=\{C_1,\ldots,C_m\}\subseteq T_{\sqcup L}(d).$ 
For each \(j \in [m]\), choose \(T_j\in L\) such that \(C_j\in T_j(d)\). Since \(\cC\) is finite and \(L\) is upward directed, there exists \(T^\star\in L\) refining all \(T_1,\ldots,T_m\). Hence $\cC\subseteq T^\star(d).$ 
If \(d'\) is a \(\cC\)-strengthening of \(d\), partition consistency of \(T^\star\) yields $\cC\subseteq T^\star(d') \subseteq T_{\sqcup L}(d').$ 
Thus $\axipc\in\saxiup.$ 
Finally, \(\axiabc\equiv\axipc\), so absence consistency has the same closure properties.

\medskip \noindent\emph{Backbone property.} 
Let \(L=\{T_1,\ldots,T_m\}\) be finite and suppose that each \(T_i\) satisfies the backbone property with margin sequence \(\bdeta^{(i)}\), \textit{i.e.,} $\Tglob^{\bdeta^{(i)}}\sqsubseteq T_i.$ 
Define the coordinate-wise minimum $\eta_s:=\min_{i\in[m]}\eta_s^{(i)}.$ 
Because \(L\) is finite, \(\bdeta\) is a separation margin sequence, and $\Tglob^{\bdeta} \sqsubseteq T_i$ for every $i\in[m]$. 
Therefore $\Tglob^{\bdeta}\sqsubseteq T_{\sqcap L},$ which proves \(\axicore\in\saxijoint\). 
For any nonempty union, the backbone of an arbitrary fixed member \(T_0\in L\) is contained in \(T_{\sqcup L}\); hence $\axicore\in\saxijoint\cap\saxiup\cap\saxiunion.$

\medskip

This proves all positive closure claims for the individual axioms in the first three structural columns of Table~\ref{tab:axioms-summary}.

Finally, we prove closure under conjunction. Let $\axi=\bigwedge_{i=1}^k\axi_i.$ 
Then $\cH_\axi(\cX) = \bigcap_{i=1}^k\cH_{\axi_i}(\cX).$ 
Suppose that \(\axi_i\in\saxijoint\) for every \(i\in[k]\), and let \(L\subseteq\cH_\axi(\cX)\) be finite and nonempty. 
Then \(L\subseteq\cH_{\axi_i}(\cX)\) for every \(i\). Moreover, because each $\axi_i$ satisfies $\saxijoint$, we have 
$T_{\sqcap L}\in\cH_{\axi_i}(\cX)$ for every $i\in[k]$. 
Therefore, $T_{\sqcap L} \in \bigcap_{i=1}^k\cH_{\axi_i}(\cX) = \cH_\axi(\cX),$ proving that \(\axi\in\saxijoint\). The arguments for \(\saxiup\) and \(\saxiunion\) are identical, replacing \(T_{\sqcap L}\) by \(T_{\sqcup L}\).
\end{proof}

\begin{corollary}[Closure of the admissible classes] 
\label{cor:admissible_closure} 
The canonical admissible class satisfies $\axiadm\in\saxijoint\cap\saxiup.$ 
Moreover, \(\cH_{\axiadmp}(\cX)\) is closed under arbitrary nonempty intersections and under nonempty upward-directed unions. 
\end{corollary}

\begin{proof} 
By Corollary~\ref{cor:replace_pr_by_core}, $\axiadm \equiv \axisca\wedge\axicore\wedge\axipc\wedge\axiperm.$ 
Each of these four requirements belongs to \(\saxijoint\), so Proposition~\ref{prop:individual_property_closure} gives \(\axiadm\in\saxijoint\). Directed-union closure follows directly from $\axiadm = \axisca\wedge\axipr\wedge\axipc\wedge\axiperm,$ as each of these requirements belongs to \(\saxiup\). 

For strong admissibility, recall that $\axiadmp \equiv \axisca\wedge\axisuhr\wedge\axipc\wedge\axiperm.$ 
Scale invariance, exactness on ultrametrics, partition consistency, and permutation invariance are all preserved under arbitrary nonempty intersections. Exactness on ultrametrics also guarantees partition richness, so the resulting method remains strongly admissible. Directed-union closure follows from Proposition~\ref{prop:individual_property_closure}. 
\end{proof}

Proposition~\ref{prop:individual_property_closure} gives a compact way to transfer the order-theoretic arguments of Section~\ref{sec:structure} to alternative axiom systems. 
The next result records the consequences that are most relevant for comparison with the canonical framework.

\begin{proposition}[Robustness under alternative axioms]
\label{prop:general}
Let \(\axi\) be any conjunction of the axioms introduced in
Section~\ref{subsec:axioms:catalog}. Then:
\begin{enumerate}[nosep, label=\textup{(\roman*)}]
\item \emph{Achievability.} The set \(\cH_\axi(\cX)\) is nonempty. In particular, $\Tglob, \Tloc, \Tsl \in \cH_\axi(\cX)$. 

\item \emph{Maximal extensions.} If strict hierarchical richness is not among the requirements defining \(\axi\) or if exactness on ultrametrics is required, then every \(T\in\cH_\axi(\cX)\) is refined by a maximal
member of \(\cH_\axi(\cX)\).

\item \emph{Finite intersections.} If \(\axi\triangleright\axicore\) and
either strict hierarchical richness is not required or exactness on
ultrametrics is required, then \(\cH_\axi(\cX)\) is closed under finite
nonempty intersections.

\item \emph{Diversity versus order invariance.} If order invariance is not among the requirements defining
\(\axi\), then \(\cH_\axi(\cX)\) has uncountable height, width, and
cellularity. If order invariance is required, then
\(\cH_\axi(\cX)\) is finite.

\item \emph{Least elements.} Let \(\axicore^1\) denote the unit-backbone
property \(\Tglob\sqsubseteq T\). We have
\[
\axiuhr\wedge\axipc\triangleright\axicore^1
\quad \text{ and } \quad 
\axiord\wedge\axipc\wedge\axipr\triangleright\axicore^1.
\]
Moreover, if \(\axi\triangleright\axicore^1\), then $\Tglob$ is the least element of the poset $(\cH_\axi(\cX),\sqsubseteq)$. 
\end{enumerate}
\end{proposition}

\begin{proof} 
 For~(i), we will show in later sections that the three methods $\Tglob$, $\Tloc$, and $\Tsl$ are order invariant, cluster-wise consistent, exact on ultrametrics, and permutation invariant. Exactness on ultrametrics implies all three richness requirements as well as refinement on ultrametrics, and cluster-wise consistency implies partition consistency. Thus these methods satisfy every requirement introduced in this paper.

 For~(ii), with the exception of strict hierarchical richness, every other requirement is preserved by directed unions according to Proposition~\ref{prop:individual_property_closure}. 
If exactness on ultrametrics is required, strict hierarchical richness is redundant. Hence, $\alpha$ satisfies $\saxiup$; thus every nonempty chain has an upper bound in \(\cH_\axi(\cX)\), and Zorn's lemma gives a maximal refinement above every \(T\in\cH_\axi(\cX)\) (see the relationship~\eqref{eq:chain_union_results}).

For~(iii), all requirements other than the richness variants are preserved
by finite intersections. The assumption \(\axi\triangleright\axicore\) ensures the existence of a common backbone to every finite intersection, and \(\axicore\triangleright\axihr\triangleright\axipr\triangleright\axicwr\) restores every non-strict richness requirement. If strict hierarchical
richness is required together with exactness on ultrametrics, exactness is
preserved by intersections and implies strict hierarchical richness.

For~(iv), suppose first that order invariance is not required. By Lemma~\ref{lem:properties_Tstable} $\Tstable$ satisfies all the requirements except possibly order invariance, and because power transformations preserve these requirements by Corollary~\ref{cor:monotone_power_trf}, the methods \(\{\Tstable^{(\power)}:\power>0\}\) also satisfy them. Lemma~\ref{lem:incompatible_family} shows that they are pairwise incompatible, yielding uncountable width and cellularity. 
To establish the uncountable height, for \(a\in(0,1]\), define a margin sequence by \(\bdeta^{(a)} = a \cdot \boldsymbol{1}\) and set
\[
S_a(d):=\Tglob(d)\cup\Tloc^{\bdeta^{(a)}}(d).
\]
Both $\Tglob(d)$ and $\Tloc^{\bdeta^{(a)}}(d)$ are subhierarchies of \(\Tloc(d)\), so the union is a hierarchy. By Lemma~\ref{lem:sep_methods_are_hierarchies} and Proposition~\ref{prop:individual_property_closure}, \(S_a\) is scale invariant, cluster-wise consistent, and permutation invariant. Moreover, for every ultrametric \(u\), 
\[
\Psi_u=\Tglob(u)\subseteq S_a(u)\subseteq\Tloc(u)=\Psi_u,
\]
so \(S_a\) is exact on ultrametrics and therefore satisfies every requirement except possibly order invariance.

The family \(\{S_a:a\in(0,1]\}\) is an increasing chain. It is strict: for
\(0<a<b\le1\), choose \(r\in[a,b)\), \(L>1/r\), and distinct
\(x_1,x_2,x_3,z\in\cX\). With \(C=\{x_1,x_2,x_3\}\), set
\[
d(x_1,x_2)=d(x_1,x_3)=r,\quad
d(x_2,x_3)=rL,\quad
d(x_1,z)=1,\quad
d(x_2,z)=d(x_3,z)=L,
\]
and, for every remaining \(w\notin C\cup\{z\}\), set
\(d(x_1,w)=1\) and \(d(x_2,w)=d(x_3,w)=L\). Complete the remaining
dissimilarities arbitrarily. Then
\(\tau(d,C)=r\) and \(\varrho(d,C)=rL>1\), so
\(C\notin S_a(d)\) but \(C\in S_b(d)\).

Conversely, suppose order invariance is required. 
We say that two dissimilarities \(d,d'\in\cD(\cX)\) are order-equivalent, written $d \sim d'$, if they induce the same weak ordering of the \(\binom n2\) unordered pairs, that is 
\[
d(x_1,y_1) \wle d(x_2,y_2) \iff  d'(x_1,y_1) \wle d'(x_2,y_2) \quad \text{ for all } (x_1,y_1), (x_2,y_2) \in \cX.
\]
Each equivalence class corresponds to a weak ordering of the $\binom{n}{2}$ pairs, and there are only finitely many such weak orderings, so the quotient space $\cD(\cX)/{\sim}$ is finite. 
If \(d\sim d'\), a strictly increasing interpolation of the finitely many distinct values taken by $d$ and $d'$ gives \(d'=g\circ d\); hence an order invariance hierarchical method is constant on each equivalence class. Because
\(\cT(\cX)\) is finite, 
\(
|\cH_{\axiord}(\cX)|
\wle |\cT(\cX)|^{\,|\cD(\cX)/{\sim}|}
\ < \ \infty,
\)
and therefore every subclass \(\cH_\axi(\cX)\) satisfying order invariance is finite.

For (v), the first claim $\axiuhr\wedge\axipc\triangleright\axicore^1$ is established in Lemma~\ref{lem:uhr_pc_implies_strong_core}, whereas the second claim is established in Lemma~\ref{lemma:core-under-ord}. 
Finally, if \(\axi\triangleright\axicore^1\), then every \(T\in\cH_\axi(\cX)\) refines \(\Tglob\), while \(\Tglob\in\cH_\axi(\cX)\) by~(i). Therefore \(\Tglob\) is the least element of $\cH_{\alpha}(\cX)$. 
\end{proof}

\section{Relational Properties and Axiom-Preserving Transformations}
\label{app:relational_prop}

Section~\ref{subsec:composed_methods} introduced transformations of dissimilarities and the notion of preservation of an axiom under preprocessing. In this appendix, we formalize a common mechanism behind several of the preservation results stated there. 
Indeed, many of the axioms considered in this paper compare the outputs of a method on two inputs that are related in a prescribed way. Scale invariance, for instance, compares \(T(d)\) and \(T(\beta d)\), while partition consistency compares \(T(d)\) and \(T(d')\) when \(d'\) is a strengthening of \(d\). 
Relational properties provide a common language for such requirements.

\subsection{Relational Properties}
\label{subsec:relational_properties}

\begin{definition}[Relational property]
\label{def:relational_property}
A \emph{relational property} \(\axi\) is specified by
\begin{itemize}[nosep]
\item an \emph{input relation} $R^\axi\subseteq\cD(\cX)\times\cD(\cX)$;
\item an \emph{output relation} $S^\axi\subseteq\cT(\cX)\times\cT(\cX)$. 
\end{itemize}
A hierarchical clustering method \(T\in\cH(\cX)\) satisfies the relational property \(\axi\) if
\[
 (d,d')\in R^\axi \quad\Longrightarrow\quad \left(T(d),T(d')\right)\in S^\axi.
\]
\end{definition}

More generally, we write \(\axi\in\saxitb\) if the requirement \(\axi\) can be expressed as a relational property or as a conjunction of relational properties. This is the meaning of the \(\saxitb\) column in Table~\ref{tab:axioms-summary}. 

The usefulness of this formulation for preprocessing comes from the following simple principle: to preserve a relational property, it is sufficient for the transformation to preserve its input relation.

\begin{proposition}[Relational preservation principle]
\label{prop:relational_stability_simple}
Let \(\axi\) be a relational property with input relation \(R^\axi\), and
let \(\trf:\cD(\cX)\to\cD(\cX)\) be a transformation. If
\[
(d,d')\in R^\axi
\quad\Longrightarrow\quad
\bigl(\trf(d),\trf(d')\bigr)\in R^\axi,
\]
then \(\trf\) preserves \(\axi\).
\end{proposition}

\begin{proof}
Let \(T\in\cH_\axi(\cX)\). If \((d,d')\in R^\axi\), the assumption gives $\bigl(\trf(d),\trf(d')\bigr)\in R^\axi$.
Because \(T\) satisfies \(\axi\), we have  $\bigl(T(\trf(d)),T(\trf(d'))\bigr)\in S^\axi$.
Thus \(T\circ\trf\) satisfies \(\axi\), and therefore \(\trf\) preserves \(\axi\).
\end{proof}

We will also use the following elementary observation when an axiom is
represented as a conjunction of relational properties.

\begin{lemma}[Preservation under conjunction]
\label{lem:conjunction_trf_closure}
Let \((\axi_j)_{j\in J}\) be a family of requirements. If a transformation \(\trf\) preserves \(\axi_j\) for every \(j\in J\), then it preserves the conjunction $\bigwedge_{j\in J}\axi_j.$ 
\end{lemma}

\begin{proof}
Let \(T\) satisfy \(\bigwedge_{j\in J}\axi_j\). Then \(T\) satisfies every \(\axi_j\). Because \(\trf\) preserves each \(\axi_j\), the composition \(T\circ\trf\) also satisfies every \(\axi_j\), and hence satisfies their conjunction.
\end{proof}

Consequently, when \(\axi\in\saxitb\), preservation of \(\axi\) can be verified by checking that \(\trf\) preserves the input relation of each relational property appearing in its representation. Notice that partition richness is handled differently in Lemma~\ref{lem:trf_preserve_axioms}: its preservation follows from surjectivity of the transformation rather than from the relational principle above.

\subsection{Relational Formulations of the Axioms}
\label{subsec:relational_axioms}

We now give relational representations of the axioms marked by a checkmark in the \(\saxitb\) column of Table~\ref{tab:axioms-summary}. These representations also make explicit which structure of the input must be preserved by a preprocessing transformation.

We denote the diagonal relation on hierarchies by
\[
\Delta_{\cT} := \{(\Psi,\Psi):\Psi\in\cT(\cX)\}. 
\]

\noindent\emph{Scale and order invariance.}
Scale invariance is the relational property defined by
\begin{equation}
\label{eq:R_sca}
R^{\axisca}
=
\{(d,\beta d):d\in\cD(\cX),\ \beta>0\},
\qquad
S^{\axisca}
=
\Delta_{\cT}.
\end{equation}
Indeed, the corresponding implication is exactly
\(T(\beta d)=T(d)\).

Likewise, order invariance is relational. Let
\[
\mathcal G_\uparrow
:=
\left\{
g:\R_{\ge0}\to\R_{\ge0}:
g(0)=0,\ g\text{ strictly increasing}
\right\}.
\]
Then order invariance is obtained from
\begin{equation}
\label{eq:R_ord}
R^{\axiord}
=
\{(d,g\circ d):d\in\cD(\cX),\ g\in\mathcal G_\uparrow\},
\qquad
S^{\axiord}
=
\Delta_{\cT}.
\end{equation}

\noindent\emph{Permutation invariance.}
Here the required output relation depends on the permutation. For each permutation
\(\phi\in\Pi(\cX)\), define the relational property
\(\axiperm^\phi\) by
\begin{equation}
\label{eq:R_perm}
\begin{aligned}
R^{\axiperm^\phi}
=
\{(d,d_\phi):d\in\cD(\cX)\}, \qquad
S^{\axiperm^\phi}
=
\{(\Psi,\phi\cdot\Psi):\Psi\in\cT(\cX)\}.
\end{aligned}
\end{equation}
Then
$\axiperm \equiv \bigwedge_{\phi\in\Pi(\cX)}\axiperm^\phi.$ 
Indeed, satisfying every \(\axiperm^\phi\) is precisely the requirement
\[
T(d_\phi)=\phi\cdot T(d)
\qquad
\text{for every }d\in\cD(\cX),\ \phi\in\Pi(\cX).
\]

\noindent\emph{Consistency.}
Fix a partition \(\cC\in\cP(\cX)\). Define
\(\axipc^\cC\) by
\begin{equation}
\label{eq:R_pc_C}
\begin{aligned}
R^{\axipc^\cC}
&=
\left\{
(d,d')\in\cD(\cX)^2:
d'\text{ is a }\cC\text{-strengthening of }d
\right\},\\
S^{\axipc^\cC}
&=
\left\{
(\Psi,\Psi')\in\cT(\cX)^2:
\cC\subseteq\Psi
\Longrightarrow
\cC\subseteq\Psi'
\right\}.
\end{aligned}
\end{equation}
Partition consistency is therefore $\axipc \equiv \bigwedge_{\cC\in\cP(\cX)}\axipc^\cC.$

Cluster-wise consistency admits the analog representation. For every nonempty \(C\subseteq\cX\), let \(\axicwc^C\) be defined by 
\begin{equation}
\label{eq:R_cwc_C}
\begin{aligned}
R^{\axicwc^C}
&=
\left\{
(d,d')\in\cD(\cX)^2:
d'\text{ is a }\{C\}\text{-strengthening of }d
\right\},\\
S^{\axicwc^C}
&=
\left\{
(\Psi,\Psi')\in\cT(\cX)^2:
C\in\Psi
\Longrightarrow
C\in\Psi'
\right\}.
\end{aligned}
\end{equation}
Hence
$\axicwc \equiv \bigwedge_{\emptyset\neq C\subseteq\cX}\axicwc^C.$ 
By Proposition~\ref{thm:relation_consistencies}, \(\axihc\equiv\axicwc\) and \(\axiabc\equiv\axipc\), so hierarchical consistency and absence consistency are also representable as conjunctions of relational properties.

\medskip
\noindent\emph{Ultrametric properties.}
For a hierarchy \(\Psi\in\cT(\cX)\), let $\cU_\Psi := \{u\in\cU(\cX):\Psi_u=\Psi\}$ be the set of ultrametrics inducing \(\Psi\). 
To represent refinement on ultrametrics, define \(\axiuhr^\Psi\) by
\begin{equation}
\label{eq:R_uhr}
R^{\axiuhr^\Psi}
=
\{(u,u):u\in\cU_\Psi\},
\qquad
S^{\axiuhr^\Psi}
=
\{(\Psi',\Psi'):\Psi'\in\cT(\cX),\ \Psi\subseteq\Psi'\}.
\end{equation}
Then
$\axiuhr \equiv \bigwedge_{\Psi\in\cT(\cX)}\axiuhr^\Psi.$

Exactness on ultrametrics is obtained by strengthening the output relation.
For each \(\Psi\in\cT(\cX)\), define \(\axisuhr^\Psi\) by
\begin{equation}
\label{eq:R_suhr}
R^{\axisuhr^\Psi}
=
\{(u,u):u\in\cU_\Psi\},
\qquad
S^{\axisuhr^\Psi}
=
\{(\Psi,\Psi)\}.
\end{equation}
Thus $\axisuhr \equiv \bigwedge_{\Psi\in\cT(\cX)}\axisuhr^\Psi.$

These representations establish the corresponding entries in the \(\saxitb\) column of Table~\ref{tab:axioms-summary}. In particular, as \( \axiadmp \equiv \axisca\wedge\axisuhr\wedge\axipc\wedge\axiperm,\) \(\axiadmp\) also belongs to \(\saxitb\).

\subsection{Applications to Axiom-Preserving Transformations}
\label{subsec:relational_applications}


\begin{proof}[Proof of Lemma~\ref{lem:trf_preserve_axioms}] 
For scale invariance, let \((d,\beta d)\in R^{\axisca}\). By assumption, there exists \(\beta'>0\) such that
$\trf(\beta d)=\beta'\trf(d).$ 
Hence
$\bigl(\trf(d),\trf(\beta d)\bigr) = \bigl(\trf(d),\beta'\trf(d)\bigr) \in R^{\axisca}$, and Proposition~\ref{prop:relational_stability_simple} ensures that \(\trf\) preserves scale invariance.

Partition richness is the one property in the lemma that we verify directly, as it is not a relational property.
Let \(T\) be partition rich and let \(\cC\in\cP(\cX)\). There exists \(d_0\in\cD(\cX)\) such that
$\cC\subseteq T(d_0).$ 
Because \(\trf\) is surjective, choose \(d\in\cD(\cX)\) such that \(\trf(d)=d_0\). 
Then $\cC\subseteq T(\trf(d)),$  so \(T\circ\trf\) is partition rich.

For partition consistency, fix \(\cC\in\cP(\cX)\). The assumption states
precisely that
\[
(d,d')\in R^{\axipc^\cC}
\quad\Longrightarrow\quad
\bigl(\trf(d),\trf(d')\bigr)\in R^{\axipc^\cC},
\]
and thus Proposition~\ref{prop:relational_stability_simple} shows that \(\trf\) preserves \(\axipc^\cC\). Because \(\cC\) was arbitrary, Lemma~\ref{lem:conjunction_trf_closure} gives preservation of \(\axipc\).

For permutation invariance, fix \(\phi\in\Pi(\cX)\). If \((d,d_\phi)\in R^{\axiperm^\phi}\), then
$\trf(d_\phi)=\trf(d)_\phi,$ and therefore
$\bigl(\trf(d),\trf(d_\phi)\bigr) = \bigl(\trf(d),\trf(d)_\phi\bigr) \in R^{\axiperm^\phi}.$
Proposition~\ref{prop:relational_stability_simple} and the conjunction Lemma~\ref{lem:conjunction_trf_closure} give preservation of \(\axiperm\).

Finally, fix \(\Psi\in\cT(\cX)\) and \(u\in\cU_\Psi\). By assumption,
\(\trf(u)\) is an ultrametric satisfying
$\Psi_{\trf(u)}=\Psi_u=\Psi.$ 
Thus \(\trf(u)\in\cU_\Psi\), and hence
$(u,u)\in R^{\axisuhr^\Psi}
\Longrightarrow
\bigl(\trf(u),\trf(u)\bigr)\in R^{\axisuhr^\Psi}.$
Again, Proposition~\ref{prop:relational_stability_simple} and Lemma~\ref{lem:conjunction_trf_closure} prove preservation of exactness on ultrametrics.
\end{proof}

\begin{proof}[Proof of Corollary~\ref{cor:monotone_power_trf}]
The condition \(g(0)=0\), together with \(g(t)>0\) for \(t>0\), ensures that \(\trf_g(d)\in\cD(\cX)\) whenever \(d\in\cD(\cX)\). 

Suppose first that \(g\) is strictly increasing. Then \(g\) preserves all inequalities defining a \(\cC\)-strengthening. Hence, whenever \(d'\) is a \(\cC\)-strengthening of \(d\), \(\trf_g(d')\) is a
\(\cC\)-strengthening of \(\trf_g(d)\). By Lemma~\ref{lem:trf_preserve_axioms}, \(\trf_g\) therefore preserves partition consistency.

Moreover, if \(u\) is an ultrametric, then
\[
g(u(x,y))
\wle
g\!\left(\max\{u(x,z),u(y,z)\}\right)
\weq
\max\{g(u(x,z)),g(u(y,z))\},
\]
so \(\trf_g(u)\) is an ultrametric. Since \(g\) is strictly increasing, it
preserves the ordering of all pairwise dissimilarities, and therefore
\(u\) and \(\trf_g(u)\) induce the same hierarchy. Thus
Lemma~\ref{lem:trf_preserve_axioms} also gives preservation of exactness on
ultrametrics.

For every permutation \(\phi\in\Pi(\cX)\),
\[
\trf_g(d_\phi)(x,y)
\weq
g\!\left(d(\phi^{-1}(x),\phi^{-1}(y))\right)
\weq
\bigl(\trf_g(d)\bigr)_\phi(x,y),
\]
and hence $\trf_g(d_\phi)=\bigl(\trf_g(d)\bigr)_\phi.$ 
Therefore \(\trf_g\) preserves permutation invariance.

If \(g\) is surjective, let \(d_0\in\cD(\cX)\). For every unordered pair \(\{x,y\}\subseteq\cX\) with \(x\neq y\), choose \(t_{xy}>0\) such that $g(t_{xy})=d_0(x,y).$ 
Such a positive preimage exists because \(d_0(x,y)>0\), \(g\) is surjective, and \(g(0)=0\). Define $d(x,x):=0$, and $d(x,y)=d(y,x):=t_{xy}$ for $x \neq y$. 
Then \(d\in\cD(\cX)\) and \(\trf_g(d)=d_0\). Hence \(\trf_g\) is surjective and, by Lemma~\ref{lem:trf_preserve_axioms}, preserves partition richness.

Finally, if for every \(\beta>0\) there exists \(c_\beta>0\) such that $g(\beta t)=c_\beta g(t)$ for every $t\ge0$, then $\trf_g(\beta d)=c_\beta\,\trf_g(d).$
Lemma~\ref{lem:trf_preserve_axioms} therefore gives preservation of scale invariance.

For \(g(t)=t^\power\), with \(\power>0\), all the preceding conditions are satisfied, with \(c_\beta=\beta^\power\). Hence the power transformation preserves all five properties, and thus preserves both admissibility and strong admissibility.
\end{proof}

\section{Proofs for Section~\ref{sec:admissible_HCs} (Admissible Methods)}
\label{app:admissible_HCs}

We now prove the requirements satisfied by the methods introduced in Section~\ref{sec:admissible_HCs}. Whenever convenient, we establish stronger requirements than those required for admissibility: e.g., we establish that $\Tsl, \Tglob, \Tloc \in \cH_{\axiord \wedge \axicwc \wedge \axiadmp}(\cX)$, and $\Tstable \in \cH_{\axicwc \wedge \axiadmp}(\cX)$.

\subsection{Single Linkage}
\label{app:proofs_single_linkage}
\subsubsection{Single linkage and minimum spanning tree}
Recall that the \emph{minimum bottleneck} between two points $x, y \in \cX$ is defined by 
\begin{align}
\label{eq:def_min_bottleneck_appendix}
 B^*(d)(x,y) \ := \ 
 \begin{cases}
\displaystyle
\min_{\gamma\in\Gamma_{x,y}} \max_{\{u,v\}\in\gamma} d(u,v), & x\neq y,\\[1ex]
0,  & x=y,
\end{cases}
\end{align}
where \(\Gamma_{x,y}\) denotes the set of paths from \(x\) to \(y\). 
The value $B^*(d)(x,y)$ can be computed by finding a minimum spanning tree (MST) of the complete graph on $\cX$ with edge weights $d$.
\begin{lemma}
\label{lem:mst_path_realizes_trfmst}
Let $d\in\cD(\cX)$, and let $M$ be a minimum spanning tree of the complete weighted graph on $\cX$ with edge weights $d$. For every $x,y\in\cX$, let $\gamma^d_{xy}$ denote the unique path from $x$ to $y$ in $M$. Then 
$B^*(d)(x,y)=\max_{(x',y')\in \gamma^d_{xy}} d(x',y').$
\end{lemma}
\begin{proof}
Let $e^\ast$ be an edge of maximal weight on $\gamma^d_{xy}$, and denote 
\(t:=d(e^\ast)=\max_{e\in \gamma^d_{xy}} d(e).\) 
Because $\gamma^d_{xy}$ is a path from $x$ to $y$, we obtain from~\eqref{eq:def_min_bottleneck_appendix} 
\(B^*(d)(x,y)\le t.\)
To prove the reverse inequality, observe that removing $e^\ast$ disconnects $M$ into two components,
separating $x$ and~$y$. Any path from $x$ to $y$ must cross the resulting cut.
If some crossing edge had weight strictly smaller than $t$, replacing $e^\ast$ by that edge would
produce a spanning tree of smaller total weight, contradicting the minimality of
$M$. Hence every $x$--$y$ path contains an edge of weight at least $t$, so
\(
B^*(d)(x,y)\ge t.
\) 
\end{proof}

\begin{lemma}
\label{lem:sl_mst_characterization}
For every \(d\in\cD(\cX)\), \(B^*(d)\) is an ultrametric and $\Tsl(d)=\Psi_{B^*(d)}$. 
Equivalently, for every \(r\ge0\), the clusters of \(\Tsl(d)\) at level \(r\) are the connected components of the graph $G_r(d) := \bigl(\cX,\{\{x,y\}:d(x,y)\le r\}\bigr).$ 
\end{lemma}
Lemma~\ref{lem:sl_mst_characterization} is the standard minimax-path characterization of single linkage \citep[see, for example,][Proposition~8]{carlsson2010characterization}, and thus we omit the proof.

\subsubsection{Properties of Single Linkage}

We first record the transformation properties of the minimax map $B^*$.

\begin{lemma}
\label{lem:trfmst_ord_suhr_perm}
Transformation $d \mapsto B^*(d)$ preserves $\axiord, \axisuhr$, and $\axiperm$.
\end{lemma}
\begin{proof}
\noindent\emph{1. Order invariance.}
Let \(g:\R_{\ge 0}\to\R_{\ge 0}\) be strictly increasing with \(g(0)=0\).
For every \(d\in\cD(\cX)\) and every \(x,y\in\cX\),
\begin{align*}
B^*(g\circ d)(x,y)
&:=
\min_{\gamma\in\Gamma_{xy}}
\max_{\{u,v\}\in\gamma}
g(d(u,v)) =
\min_{\gamma\in\Gamma_{xy}}
g\!\left(
\max_{\{u,v\}\in\gamma} d(u,v)
\right)\\
&= g\!\left( \min_{\gamma\in\Gamma_{xy}} \max_{\{u,v\}\in\gamma} d(u,v) \right)  = g\bigl(B^*(d)(x,y)\bigr).
\end{align*}
The second and third equalities use the strict monotonicity of \(g\). 
Hence $B^*(g\circ d)=g\circ B^*(d).$
Therefore the transformation \(B^*\) preserves order equivalence of dissimilarities, and consequently preserves order invariance. 

\medskip 
\noindent\emph{2. Exactness on ultrametric: }
Let $u$ be an ultrametric. We prove $B^*(u)=u$ by sandwiching. 
The inequality $B^*(u)\le u$ is immediate from the definition of $B^*$, as the one-edge path from $x$ to $y$ gives
\(B^*(u)(x,y)\le u(x,y).\) 
For the reverse inequality, fix $x,y\in\cX$ and introduce an arbitrary $x$ to $y$ path
$\gamma=((x_0,x_1),\dots,(x_{m-1},x_m))$. Repeatedly applying the strong triangle inequality gives
\(u(x,y)\le \max_{i=1,\dots,m} u(x_{i-1},x_i).\)
Because this holds for any $x$-to-$y$ path~$\gamma$,
\(
u(x,y)\le \min_{\gamma\in\Gamma_{x,y}} \max_{(u,v)\in\gamma} u(u,v)
= B^*(u)(x,y).
\)
Hence $B^*(u)=u$.

\medskip 
\noindent\emph{3. Permutation invariance: }
For a path $\gamma = ((u_1,u_2), (u_2,u_3), \cdots, (u_{m-1},u_m) )$, define the path 
$\phi^{-1}\gamma = ( (\phi^{-1}(u_1),\phi^{-1}(u_2)), (\phi^{-1}(u_2),\phi^{-1}(u_3)), \cdots, (\phi^{-1}(u_{m-1}),\phi^{-1}(u_m) ))$. 
The map \(\gamma\mapsto\phi^{-1}\gamma\) is a bijection from \(\Gamma_{x,y}\) onto \(\Gamma_{\phi^{-1}(x),\phi^{-1}(y)}\). Hence
\[
B^*(d_\phi)(x,y)
\weq B^*(d)(\phi^{-1}(x),\phi^{-1}(y))
\weq B^*(d)_\phi(x,y).
\]
\end{proof}

\begin{lemma}
Single linkage $\Tsl$ satisfies $\axiord, \axisuhr,\axiperm$, and $\axicwc$. 
\end{lemma}
\begin{proof}
Because $\Tglob$ satisfies $\axiord, \axisuhr$, and $\axiperm$, 
order invariance, exactness on ultrametrics, and permutation invariance follow immediately from Proposition~\ref{prop:relational_stability_simple},  Lemma~\ref{lem:sl_mst_characterization},  and Lemma~\ref{lem:trfmst_ord_suhr_perm}.

\medskip 
\noindent\emph{Cluster-wise consistency: }
Let \(C\in\Tsl(d)\), and let \(d'\) be a \(\{C\}\)-strengthening of \(d\).
By Lemma~\ref{lem:sl_mst_characterization}, there exists \(r\ge0\) such that \(C\) is a connected component of \(G_r(d)\). 

Because \(d'\) is a \(\{C\}\)-strengthening, every edge within \(C\) whose \(d\)-weight is at most \(r\) still has \(d'\)-weight at most \(r\). Thus \(C\) remains connected in \(G_r(d')\). Conversely, every edge between \(C\) and \(\bC\) has \(d\)-weight greater than \(r\), and its weight can only increase under the strengthening. Hence no edge of \(G_r(d')\) joins \(C\) to \(\bC\).

Therefore \(C\) is also a connected component of \(G_r(d')\), and hence $C\in\Tsl(d').$ 
Thus \(\Tsl\) is cluster-wise consistent.
\end{proof}

\subsection{Other Linkage Methods}
\label{app:linkage}

\subsubsection{Definitions of the Linkage Methods Considered}
We briefly recall the non-binary variants of the standard linkage rules considered in Section~\ref{subsec:linkage}. Recall that linkage methods proceed by iteratively merging clusters (starting with singletons). At a given step of the algorithm, let \(\cC_{\mathrm{active}}\) denote the current partition into active clusters, and let $D:\cC_{\mathrm{active}}\times\cC_{\mathrm{active}}\to\R_{\ge0}$ denote the linkage dissimilarity between active clusters, as determined by the chosen linkage rule. The next merge occurs at the minimum inter-cluster dissimilarity 
\[
r := \min_{\substack{C_1,C_2\in\cC_{\mathrm{active}}\\ C_1\neq C_2}} D(C_1,C_2).
\]
However, several pairs of active clusters may attain this minimum simultaneously. To merge all clusters involved in such ties simultaneously, form the graph whose vertices are the active clusters \(C\in\cC_{\mathrm{active}}\), with an edge between \(C_1\) and \(C_2\) whenever $D(C_1,C_2)=r.$
Each nontrivial connected component of this graph is then merged into a single new cluster, obtained as the union of its vertices. In particular, if \(C_1\) is tied with \(C_2\) and \(C_2\) is tied with \(C_3\), then all three are merged together even if \(D(C_1,C_3)>r\). Each cluster created in this way is added to the output hierarchy, and the procedure is repeated with the resulting collection of active clusters.

The linkage dissimilarities considered here admit the Lance-Williams update~\citep{murtagh2017algorithms}
\begin{align*}
D(C_{11}\cup C_{12},C_2)
& =
\alpha_1D(C_{11},C_2)
+\alpha_2D(C_{12},C_2)
+\beta D(C_{11},C_{12}) 
+\gamma
\left|
D(C_{11},C_2)-D(C_{12},C_2)
\right|,
\end{align*}
where the coefficients are given in Table~\ref{tab:linkage}. We take \(D(\{x\},\{y\})=d(x,y)\).
\begin{table}[ht]
\centering
\footnotesize
\renewcommand{\arraystretch}{1.5}
\begin{tabular}{@{} l p{0.28\linewidth} l @{}}
\toprule
\emph{Method} & $\boldsymbol{D(C_1,C_2)}$ & $\boldsymbol{(\axi_1,\axi_2,\beta,\gamma)}$ \\
\midrule
Single
& $\min\limits_{x\in C_1,\,y\in C_2} d(x,y)$
& $\left(\tfrac12,\,\tfrac12,\,0,\,-\tfrac12\right)$ \\

Complete
& $\max\limits_{x\in C_1,\,y\in C_2} d(x,y)$
& $\left(\tfrac12,\,\tfrac12,\,0,\,\tfrac12\right)$ \\

Unweighted average (UPGMA)
& $\frac{1}{|C_1|\,|C_2|}\sum\limits_{x\in C_1,\,y\in C_2} d(x,y)$
& $\left(\frac{|C_{11}|}{|C_1|},\,\frac{|C_{12}|}{|C_1|},\,0,\,0\right)$ \\

Weighted average (WPGMA)
& $\frac{D(C_{11},C_2)+D(C_{12},C_2)}{2}$
& $\left(\tfrac12,\,\tfrac12,\,0,\,0\right)$ \\

Ward's
& $\frac{|C_1|\,|C_2|}{|C_1|+|C_2|}\,\left\|m_{C_1}-m_{C_2}\right\|^2$
& $\left(\frac{|C_{11}|+|C_2|}{|C_1|+|C_2|},\,\frac{|C_{12}|+|C_2|}{|C_1|+|C_2|},\,-\frac{|C_2|}{|C_1|+|C_2|},\,0\right)$ \\

Centroid (UPGMC)
& $\left\|m_{C_1}-m_{C_2}\right\|^2$
& $\left(\frac{|C_{11}|}{|C_1|},\,\frac{|C_{12}|}{|C_1|},\,-\frac{|C_{11}||C_{12}|}{|C_1|^2},\,0\right)$ \\

Weighted centroid (WPGMC / Median)
& $\left\|\frac{m_{C_{11}}+m_{C_{12}}}{2}-m_{C_2}\right\|^2$
& $\left(\tfrac12,\,\tfrac12,\,-\tfrac14,\,0\right)$ \\
\bottomrule
\end{tabular}
\caption{Classical hierarchical clustering linkages and their Lance--Williams coefficients. $m_C$ denotes the centroid of cluster $C$.
}
\label{tab:linkage}
\end{table}

Ward and Centroid linkages are typically defined for squared Euclidean dissimilarities, as the updates can become negative on arbitrary dissimilarities. We implicitly restrict to squared Euclidean dissimilarities when working with these linkages. 
Moreover, for linkage rules such as WPGMA and median linkage, the updated dissimilarity after a multiway merge may depend on the order in which clusters in a tied component are merged. To obtain a well-defined hierarchical clustering method, we therefore regard a fixed deterministic tie-breaking convention as part of the linkage rule. This choice will not affect the counterexamples below, since all comparisons determining the relevant merges are strict.

\subsubsection{Non-admissibility of the Other Linkage Methods}
\label{app:other_linkages_not_admissible}

We demonstrate the non-admissibility of complete, average, Ward, UPGMC, and WPGMC linkage methods by providing counterexamples that show they do not satisfy consistency. 
\begin{lemma}
Complete, unweighted and weighted average, Ward, centroid, and median linkages do not satisfy partition consistency. Thus, none of them is admissible, regardless of the tie-breaking rule.
\label{lem:other_linkages_not_admissible}
\end{lemma}

\begin{proof}
Fix distinct points \(1,2,3,4\in\cX\), let
\(Y:=\cX\setminus\{1,2,3,4\}\), \(A:=\{1,2,3\}\), and
\(
\cC:=\{A,\{4\}\}\cup\bigl\{\{y\}:y\in Y\bigr\}.
\)
Define \(d,d'\in\cD(\cX)\) by
\[
d(1,2)=d(1,3)=6,\qquad
d(1,4)=d(2,4)=8,\qquad
d(2,3)=4,\qquad d(3,4)=q, \text{ and}
\]
\[
d'(1,2)=3,\qquad d'(i,j)=d(i,j)
\quad\text{for all other pairs}.
\]
Whenever at least one of \(x,y\) belongs to \(Y\), set
\(d(x,y)=d'(x,y)=100\). Thus \(d'\) is a
\(\cC\)-strengthening of \(d\).

Choose \(q\) as in table below. For every value of \(q\) in the table, both \(d\) and \(d'\) are squared Euclidean dissimilarities. The relevant linkage values are also given in the table. 
\begin{table}[!ht]
\centering
\begin{tabular}{@{}lccccc@{}}
\toprule
\emph{Linkage} & $q$ & $D(\{2,3\},\{1\})$ & $D(\{2,3\},\{4\})$ & $D'(\{1,2\},\{3\})$ & $D'(\{1,2\},\{4\})$ \\
\midrule
\text{Complete} &5&6&8&6&8\\
\text{UPGMA/WPGMA} &$\frac92$&6&$\frac{25}{4}$&5&8\\
\text{Centroid/Median} &$\frac{41}{10}$&5&$\frac{101}{20}$& $\frac{17}{4}$ & $\frac{29}{4}$ \\
\text{Ward} &$\frac{41}{10}$&$\frac{20}{3}$&$\frac{101}{15}$ & $\frac{17}{3}$ & $\frac{29}{3}$ \\
\bottomrule
\end{tabular}
\label{tab:linkage_counterexamples}
\end{table}

Under \(d\), every rule first merges \(\{2,3\}\), since
\(d(2,3)=4\) is the unique smallest initial dissimilarity. The table
then gives
\(
D(\{2,3\},\{1\})
<
\min\bigl\{D(\{2,3\},\{4\}),d(1,4)\bigr\},
\)
while all dissimilarities involving \(Y\) are larger. Hence the second
merge forms \(A\), so \(A\in T(d)\) and therefore \(\cC\subseteq T(d)\).

Under \(d'\), the unique first merge is \(\{1,2\}\). The last two
columns show that
\(
q<
\min\bigl\{
D'(\{1,2\},\{3\}),
D'(\{1,2\},\{4\})
\bigr\},
\)
so the second merge is \(\{3,4\}\). Consequently, every subsequent
cluster containing \(3\) also contains \(4\), and hence
\(A\notin T(d')\).

Thus \(A\in T(d)\setminus T(d')\), although \(d'\) is a
\(\cC\)-strengthening of \(d\) and \(\cC\subseteq T(d)\). Therefore
each of the six linkage rules violates partition consistency. All
comparisons are strict, so the conclusion is independent of the
tie-breaking convention.
\end{proof}

\subsection{Separation Methods: Proof of Proposition~\ref{prop:admissible_separation-based}}
\label{app:sep_base}

Lemma~\ref{lem:sep_methods_are_hierarchies} proves that $\Tglob^{\bdeta}$ and $\Tloc^{\bdeta}$ are hierarchical clustering methods satisfying scale invariance, cluster-wise consistency, and permutation invariance, and Lemma~\ref{lem:separation_methods_rich} establishes strict hierarchical richness. 
Because cluster-wise consistency implies partition consistency and strict hierarchical richness implies partition richness, we conclude that \(\Tglob^{\bdeta}\) and \(\Tloc^{\bdeta}\) are admissible for every separation margin sequence \(\bdeta\). 
Lemma~\ref{lem:Tloc_family} shows that distinct margin sequences define distinct methods, completing the proof of Proposition~\ref{prop:admissible_separation-based}. 
We also record the stronger properties of exactness on ultrametrics and order invariance for the unit-margin methods in Lemmas~\ref{lem:separation_methods_exact_ultrametrics} and~\ref{lem:Tloc_order_invariant}.

\subsubsection{Basic properties}

\begin{lemma}
\label{lem:sep_methods_are_hierarchies}
For every separation margin sequence $\bdeta$, $ \Tglob^{\bdeta}$ and $ \Tloc^{\bdeta} $ are hierarchical clustering methods satisfying \(\axisca\), \(\axicwc\), and \(\axiperm\).
\end{lemma}

\begin{proof}
$\Tloc$ is a known hierarchical clustering method, and for any dissimilarity~$d$, $\Tloc(d)$ satisfies all the conditions in Definition~\ref{def:hierarchy}~\citep{dreveton2025hierarchical,balcan2008discriminative}. 
Let \(\bdeta\) be a separation margin sequence. By definition of the corresponding separation conditions, for every dissimilarity \(d\), we have $\Tglob^{\bdeta}(d) \subseteq \Tloc^{\bdeta}(d) \subseteq \Tloc.$ 
Because \(\Tloc(d)\) is laminar, each of its subfamilies is also laminar. Hence the outputs of \(\Tglob^{\bdeta}\) and \(\Tloc^{\bdeta}\) are laminar as well. Moreover, by construction, each of these outputs contains the root \(\cX\) and all leaves \(\{x\}\), \(x \in \cX\). Therefore, \(\Tglob^{\bdeta}\) and \(\Tloc^{\bdeta}\) are hierarchical clustering methods.

Next, we show that $\Tloc^{\bdeta}$ satisfies each property. The proofs for $\Tglob^{\bdeta}$ are analogous and hence omitted.

\medskip
\noindent\emph{1. $\Tloc^{\bdeta}$ satisfies $\axisca$.} 
Let $\beta>0$. For every $x,y,z \in \cX$ we have $\frac{\beta \cdot d(x,y)}{\beta \cdot d(x,z)} = \frac{d(x,y)}{d(x,z)}$.
Hence $C\in\Tloc^{\bdeta}(\beta d)$ iff $C\in\Tloc^{\bdeta}(d),$ and therefore \(\Tloc^{\bdeta}(\beta d)=\Tloc^{\bdeta}(d)\).

\medskip
\noindent\emph{2. $\Tloc^{\bdeta}$ satisfies $\axicwc$.}  
Let $d$ be a dissimilarity function, and take $C \in \Tloc^{\bdeta}(d)$. 
Because the root and singleton clusters belong to $\Tloc^{\bdeta}(d)$ for every dissimilarity, it suffices to consider a nontrivial proper cluster \(C\). 
Consider a dissimilarity $d'$ such that $d'$ is a $\{C\}$-strengthening of $d$. 
Then, for every $x, y \in C$, $z\in \cX\backslash C$, $d'(x,y) \le d(x,y)$ and $d'(x,z) \ge d(x,z)$ hold.
Thus 
\[
 \tau(d',C) 
 \weq \max_{x,y\in C,z\in \bC}\frac{ d'(x,y)}{d'(x,z)} 
 \wle  \max_{x,y\in C,z\in \bC}\frac{ d(x,y)}{d(x,z)} 
 \weq \tau(d,C)
 \ < \ \eta_{|C|-1}, 
\]
establishing that $C \in \Tloc^{\bdeta}(d')$.

\medskip 
\noindent\emph{3. $\Tloc^{\bdeta}$ satisfies $\axiperm$.} 
Let $\phi \in \Pi(\cX)$ be a permutation, and recall that $d_{\phi}(x,y) := d(\phi^{-1}(x),\phi^{-1}(y))$ (Definition~\ref{def:admissible}). 
Hence, $d_\phi(\phi(x),\phi(y)) =d(x,y)$ for all $x,y \in \cX$, and  
\begin{align*}
 \tau(d,C) 
 := \max_{x,y\in C,z\in \bC}\frac{d(x,y)}{d(x,z)} 
 & \weq \max_{x,y\in C,z\in \bC}\frac{d_{\phi}(\phi(x),\phi(y))}{d_{\phi}(\phi(x),\phi(z))} 
 \weq \max_{\substack{x,y\in \phi(C), \\ z \in \overline{ \phi(C)}}}\frac{d_{\phi}(x,y)}{d_{\phi}(x,z)} 
 \weq \tau(d_\phi,\phi(C)),
\end{align*} 
where we used the fact that $\phi$ is a bijection and $\phi(\bC) = \overline{\phi(C)}$. Moreover, \(|\phi(C)|=|C|\). 
Hence $\tau(d,C)<\eta_{|C|-1}$ iff $\tau(d_\phi,\phi(C))<\eta_{|\phi(C)|-1}.$ 
Therefore: $C\in\Tloc^{\bdeta}(d)$ iff $\phi(C)\in\Tloc^{\bdeta}(d_\phi).$ 
Thus $\Tloc^{\bdeta}(d_\phi)=\phi\cdot\Tloc^{\bdeta}(d).$ 
\end{proof}

\begin{lemma}
\label{lem:Tloc_family}
Let $\bdeta \ne \bdeta'$ be two distinct separation margin sequences. Then $\Tglob^{\bdeta} \neq \Tglob^{\bdeta'}$ and $\Tloc^{\bdeta} \neq \Tloc^{\bdeta'}$. 
\end{lemma}
\begin{proof} 
Because \(\bdeta\neq\bdeta'\), there exists \(s\in[n-2]\) such that \(\eta_s\neq\eta_s'\). 
Without loss of generality, assume $\eta_s<\eta_s'.$ Choose \(C\subsetneq\cX\) with \(|C|=s+1\), set  $r:=\frac{\eta_s+\eta_s'}{2},$ and define
\[
d(x,y) = r \text{ for distinct } x,y \in C \text{ or } x,y\in\bC, \qquad \text{and} \qquad d(x,y)=1 \text{ for other distinct } x,y
\]
Observe that $\varrho(d,C) = \tau(d,C) = r$. Therefore, the set~$C$ belongs to $\Tloc^{\bdeta'}(d)$ and to $\Tglob^{\bdeta'}(d)$ but belongs neither to $\Tloc^{\bdeta}(d)$ nor to $\Tglob^{\bdeta}(d)$, proving $\Tloc^{\bdeta}(d) \neq \Tloc^{\bdeta'}(d)$ and $\Tglob^{\bdeta}(d) \neq \Tglob^{\bdeta'}(d)$. 
\end{proof}

\subsubsection{Specific properties for \texorpdfstring{$\bdeta = \mathbf{1}$}{eta=1}}

\begin{lemma}
 \label{lem:separation_methods_exact_ultrametrics}
Let \(\bdeta\) be a separation margin sequence. These statements are equivalent:\\
(i) \(\bdeta=\mathbf{1}\); (ii) \(\Tloc^{\bdeta}\) satisfies \(\axisuhr\); (iii) \(\Tglob^{\bdeta}\) satisfies \(\axisuhr\).
\end{lemma}

\begin{proof}
We first show that \textup{(i)} implies both \textup{(ii)} and
\textup{(iii)}. Assume that \(\bdeta=\mathbf{1}\), so that $\Tloc^{\bdeta}=\Tloc$ and $\Tglob^{\bdeta}=\Tglob$. 
Fix an ultrametric \(u\in\cU(\cX)\), and let \(\Psi_u\) denote its associated hierarchy. We prove that $\Psi_u \subseteq \Tglob(u)$ and $\Tloc(u) \subseteq \Psi_u$. 
Combined with $\Tglob(u) \subseteq \Tloc(u)$, this establishes $\Psi_u = \Tglob(u) = \Tloc(u)$

We first prove \(\Psi_u\subseteq\Tglob(u)\). 
The root and singleton clusters belong to \(\Tglob(u)\) by definition, so let \(C\in\Psi_u\) be a non-root and non-singleton cluster. 
By the ultrametric-ball representation of \(\Psi_u\) (see Equation~\eqref{eq:ultrametric_hierarchy}), there exists $r\ge0$ such that $u(x,y)\le r<u(x,z)$ for all $x,y\in C,\ z\notin C.$ 
Hence, for all \(x_1,x_2,y\in C\) and \(z\notin C\), we have $\frac{u(x_1,y)}{u(x_2,z)}<1.$ 
Therefore $\varrho(u,C)<1,$ and thus \(C\in\Tglob(u)\).
This establishes $\Psi_u \subseteq \Tglob(u)$

It remains to prove \(\Tloc(u)\subseteq\Psi_u\). 
Again, the root and singleton clusters are immediate, so let \(C\in\Tloc(u)\) be a non-root and non-singleton cluster. 
Because \(\tau(u,C)<1\), we have $u(x,y)<u(x,z)$ for all $x,y\in C,\ z\notin C.$ 
Fix \(x_0\in C\), and set $R_{\mathrm{in}} := \max_{y\in C}u(x_0,y)$ and $R_{\mathrm{out}} := \min_{z\notin C}u(x_0,z).$ 
Because \(\cX\) is finite, these extrema are attained, and the preceding strict inequalities imply $R_{\mathrm{in}}<R_{\mathrm{out}}.$ 
Therefore, $B_u(x_0,R_{\mathrm{in}})=C.$ 
By the ultrametric-ball representation of \(\Psi_u\), this gives \(C\in\Psi_u\). 
This establishes $\Tloc(u) \subseteq \Psi_u$. 

Hence, $\Psi_u=\Tglob(u)=\Tloc(u)$ for every ultrametric \(u\), proving (i) implies (ii) and (iii).

We now prove the converses by contraposition. 
Suppose that \(\bdeta\neq\mathbf{1}\). 
Then there exists \(s\in[n-2]\) such that \(\eta_s<1\). 
Choose a subset \(C\subsetneq\cX\) with $|C|=s+1$ and choose \(r\) such that $\eta_s<r<1.$ 
Define a dissimilarity \(u\) by 
\[u(x,y) = r  \text{ for distinct } x,y\in C, \qquad \text{and} \qquad u(x,y) = 1 \text{ for other distinct} x,y,\]
and observe that $u$ is an ultrametric. 
Moreover, for every \(x\in C\), $B_u(x,r)=C,$ so \(C\in\Psi_u\). 
However, $\tau(u,C)=\varrho(u,C)=r>\eta_s=\eta_{|C|-1}.$
Hence $C\notin\Tloc^{\bdeta}(u)$ and $C\notin\Tglob^{\bdeta}(u).$ 
Thus neither \(\Tloc^{\bdeta}(u)\) nor \(\Tglob^{\bdeta}(u)\) equals \(\Psi_u\). Therefore neither method is exact on ultrametrics.
This proves (ii) implies (i) and (iii) implies (i). 
\end{proof}

Similarly, we obtain the strict hierarchical richness of $\Tglob^{\bdeta}$ and $\Tloc^{\bdeta}$. 

\begin{lemma}
\label{lem:separation_methods_rich}
$\Tglob^{\bdeta}$ and $\Tloc^{\bdeta}$ satisfy $\axishr$ for any separation margin sequence~$\bdeta$.
\end{lemma}

\begin{proof}
Fix an arbitrary hierarchy \(\Psi\in\cT(\cX)\), and let $q\in\left(0,\min_{1\le s\le n-2}\eta_s\right).$ We will construct a dissimilarity $d$ such that $\Tglob^{\bdeta}(d) = \Tloc^{\bdeta}(d) = \Psi$. 
View \(\Psi\) as a rooted tree, with root $\cX$ and leaves the singletons $\{x\}$ for $x\in \cX$; let \(\operatorname{depth}(C)\) denote the depth of cluster $C \in \Psi$, with \(\operatorname{depth}(\cX)=0\). Assign to every non-singleton \(C\in\Psi\) the height $h(C):=q^{\operatorname{depth}(C)},$ and assign height $0$ to the singleton leaves. Define for every distinct $x,y$, \[u_\Psi(x,y) :=h\bigl(\operatorname{lca}_{\Psi}(x,y)\bigr).\]
Then \(u_\Psi\) is an ultrametric whose associated hierarchy is exactly \(\Psi\).

Let \(C\in\Psi\) be non-singleton and proper cluster, and denote by \(\operatorname{par}(C)\) its parent in \(\Psi\). 
For all \(x_1,x_2,y\in C\) and \(z\notin C\), we have  $u_\Psi(x_1,y)\le h(C)$ and $u_\Psi(x_2,z)\ge h(\operatorname{par}(C))$.
Using $\frac{h(C)}{h(\operatorname{par}(C))}=q$ and the definition of $\varrho(u_\Psi,C)$, we obtain  $\varrho \le q < \eta_{|C|-1}.$
As \(\tau(u_\Psi,C)\le\rho(u_\Psi,C)\), it follows that every \(C\in\Psi\) belongs to \(\Tglob^\eta(u_\Psi)\) and to \(\Tloc^\eta(u_\Psi)\), and thus
\[
\Psi \subseteq \Tglob^\eta(u_\Psi)
\quad \text{ and } \quad 
\Psi \subseteq \Tloc^\eta(u_\Psi). 
\]
Moreover, because $\Tglob$ and $\Tloc$ are exact on ultrametrics (Lemma~\ref{lem:separation_methods_exact_ultrametrics}), we also have 
\[
\Tglob^\eta(u_\Psi)\subseteq\Tglob(u_\Psi)=\Psi,
\quad \text{ and } \quad 
\Tloc^\eta(u_\Psi)\subseteq\Tloc(u_\Psi)=\Psi,
\]
Hence $\Tglob^\eta(u_\Psi) = \Tloc^\eta(u_\Psi) = \Psi$, and both methods satisfy strict hierarchical richness. 
\end{proof}

\begin{lemma}
\label{lem:Tloc_order_invariant}
Let \(\bdeta\) be a separation margin sequence. These statements are equivalent: \\
(i)  \(\bdeta=\mathbf{1}\); (ii) \(\Tloc^{\bdeta}\) satisfies \(\axiord\); (iii)  \(\Tglob^{\bdeta}\) satisfies \(\axiord\).
\end{lemma}

\begin{proof}
We only prove the equivalence between (i) and (ii). The proof of the equivalence between (i) and (iii) is almost identical and thus is omitted. 

To show that (i) implies (ii), let $\bdeta=\mathbf{1}$, so that $\Tloc^{\bdeta}=\Tloc$. 
Let \(g:\R_{\ge0}\to\R_{\ge0}\) be strictly increasing with \(g(0)=0\).
For every nonsingleton cluster \(C\subsetneq\cX\),
\[
\tau(d,C)<1 \iff d(x,y)<d(x,z)
\quad
\text{for all }x,y\in C,\ z\notin C.
\]
Because \(g\) is strictly increasing, this is equivalent to
\[
g(d(x,y))<g(d(x,z))
\quad
\text{for all }x,y\in C,\ z\notin C,
\]
and hence to $\tau(g\circ d,C)<1.$ 
Therefore $\Tloc(g\circ d)=\Tloc(d).$

We now establish that (ii) implies (i). We prove it by contraposition. Consider a separation margin sequence $\bdeta \ne \mathbf{1}$; we will establish that $\Tloc^{\bdeta}$ is not order-invariant by constructing two dissimilarities $d,d'$ such that $d' = g \circ d$ for some strictly increasing function $g$ satisfying $g(0) = 0$, but for which $\Tloc^{\bdeta}(d) \ne \Tloc^{\bdeta}(d')$.

Because \(\bdeta\neq\mathbf{1}\), there exists \(s\in[n-2]\) such that \(\eta_s<1\). Choose a subset $C\subseteq\cX$ with $s=|C|-1$.
Define $d,d'\in\cD(\cX)$ such that for distinct $x,y\in\cX$,
\[
d(x,y):=
\begin{cases}
\eta_s/2, & \text{if } x,y\in C \text{ or } x,y\in \bC,\\
1, & \text{if } x\in C,\ y\in \bC,
\end{cases}
\]
and
\[
d'(x,y):=
\begin{cases}
\eta_s, & \text{if } x,y\in C \text{ or } x,y\in \bC,\\
1, & \text{if } x\in C,\ y\in \bC.
\end{cases}
\]
Because $0<\eta_s/2<\eta_s<1$ there exists a strictly increasing \(g:\R_{\ge0}\to\R_{\ge0}\) satisfying $g(0)=0$, $g(\eta_s/2)=\eta_s$, and $g(1)=1$. 
Hence $d'=g\circ d.$ 

Observe that for every $x,y\in C$ and $z\notin C$, \( \frac{d(x,y)}{d(x,z)}=\frac{\eta_s/2}{1}<\eta_s,\) hence \( C\in \Tloc^{\bdeta}(d) \) holds.
However, fix any $z_0\in \bC$. For every $x,y\in C$ with $x\neq y$, \( \frac{d'(x,y)}{d'(x,z_0)}=\frac{\eta_s}{1}=\eta_s \not< \eta_s. \) Therefore \(C\notin \Tloc^{\bdeta}(d').\)
Hence $\Tloc^{\bdeta}$ is not order invariant.
\end{proof}

\subsection{Bryant-Berry Stable Clusters: Proof of Proposition~\ref{prop:stable_cluster_is_admissible}}

We first establish the strong admissibility of the Bryant--Berry method. 
Moreover, Corollary~\ref{cor:monotone_power_trf} shows that \(\mu_{\mathrm{power}}^p\) preserves strong admissibility for every
\(p>0\). Hence $\Tstable^{(p)} = \Tstable\circ\mu_{\mathrm{power}}^p$ is strongly admissible for every \(p>0\). In particular, \(\Tstable\) and all its power-transformed variants are admissible, proving
Proposition~\ref{prop:stable_cluster_is_admissible}.

\begin{lemma}
\label{lem:properties_Tstable}
$\Tstable$ is a hierarchical clustering method and satisfies $\axisca, \axisuhr, \axicwc$ and $\axiperm$. 
In particular, \(\Tstable\) is strongly admissible. 
\end{lemma}
\begin{proof}
For every \(d\in\cD(\cX)\), Bryant and Berry~\citep{bryant2001structured} show that the stable clusters $\{C\subseteq\cX:\iota^d(C)>0\}$ form a laminar family. Because \(\Tstable(d)\) additionally contains the root
\(\cX\) and all singleton clusters, it is a hierarchy. Thus \(\Tstable\) is a hierarchical clustering method. 
We now verify the four axioms.

\medskip 
\noindent\emph{1. Scale invariance.} 
For all \(\beta>0\) and \(u,v,z\in\cX\), we have $\rho_{\beta d}(uv\mid z) = \beta\,\rho_d(uv\mid z).$ 
Hence
\[
\overline{\rho}_{\beta d}(UV\mid Z)
=
\beta\,\overline{\rho}_d(UV\mid Z)
\]
for every admissible triplet \((U,V,Z)\), and therefore $\iota^{\beta d}(C)=\beta\,\iota^d(C).$
Because \(\beta>0\), we also have $\iota^{\beta d}(C)>0$ iff $\iota^d(C)>0.$ 
Thus \(\Tstable(\beta d)=\Tstable(d)\).

\medskip
\noindent\emph{2. Exactness on ultrametrics.}
Let \(u\) be an ultrametric, and let \(\Psi_u\) be its associated hierarchy.
Because $\Tloc\sqsubseteq\Tstable$ and $\Tloc(u)=\Psi_u$, we immediately obtain $\Psi_u\subseteq\Tstable(u).$
It remains to prove the reverse inclusion.

Let \(C\in\Tstable(u)\). The root and singleton cases are immediate, so assume that \(C\) is a non-singleton proper subset of \(\cX\). Then \(\iota^u(C)>0\). 

We first show that every external point is equidistant from all points of \(C\).
To do so, fix \(z\in\cX\setminus C\), and we will prove that the function \(x\mapsto u(x,z)\) is constant on \(C\). By contradiction, suppose this is not the case. Let
\[
r:=\min_{x\in C}u(x,z),
\qquad
U:=\{x\in C:u(x,z)=r\},
\qquad
V:=C\setminus U,
\qquad
Z:=\{z\}.
\]
Then \(U,V,Z\) are nonempty, and \(U \cap V = \emptyset\) and \(C=U\cup V\). For every \(x\in U\) and \(y\in V\), we have
\(
r = u(x,z)<u(y,z).
\)
Because \(u\) is an ultrametric, the two larger distances among
\(u(x,y),u(x,z),u(y,z)\) are equal. Hence
\(
u(x,y)=u(y,z).
\)
Therefore
\[
\rho_u(xy\mid z)
=
\min\{u(x,z),u(y,z)\}-u(x,y)
=
u(x,z)-u(y,z)
<0.
\]
Averaging over \(x\in U\), \(y\in V\), and \(Z=\{z\}\), we obtain
\(
\overline{\rho}_u(UV\mid Z)<0,
\)
contradicting \(\iota^u(C)>0\). Thus \(u(x,z)\) is constant on \(C\).
Denote its common value by
\(
\lambda_z:=u(x,z), \, x\in C.
\)

We next show that every internal dissimilarity is strictly smaller than this common external dissimilarity. 
Suppose, toward a contradiction, that there exist distinct \(x,y\in C\) such that
\(
u(x,y)\ge \lambda_z.
\)
Because \(u(x,z)=u(y,z)=\lambda_z\), the ultrametric inequality gives
\[
u(x,y)\le \max\{u(x,z),u(y,z)\}=\lambda_z.
\]
Hence
\(
u(x,y)=\lambda_z.
\)
Define a relation \(\sim\) on \(C\) by
\(
a\sim b
\,\Longleftrightarrow\,
u(a,b)<\lambda_z.
\)
By the ultrametric inequality, \(\sim\) is an equivalence relation on \(C\).
Because \(u(x,y)=\lambda_z\), it has at least two equivalence classes. Let
\(U\) be one equivalence class and set
\(
V:=C\setminus U,
\, 
Z:=\{z\}.
\)
For every \(a\in U\) and \(b\in V\), we have \(u(a,b)\ge \lambda_z\).
On the other hand,
\(
u(a,b)\le \max\{u(a,z),u(b,z)\}=\lambda_z.
\)
Thus
\(
u(a,b)=\lambda_z.
\)
Therefore
\(
\rho_u(ab\mid z)
=
\min\{u(a,z),u(b,z)\}-u(a,b)
=
\lambda_z-\lambda_z
=
0.
\)
Hence
\(
\overline{\rho}_u(UV\mid Z)=0,
\)
again contradicting \(\iota^u(C)>0\). Therefore, for every
\(z\in\cX\setminus C\) and all distinct \(x,y\in C\),
\(
u(x,y)<u(x,z)=u(y,z).
\)

Now set $R:=\max_{x,y\in C}u(x,y).$ 
The preceding inequality implies that, for every \(x\in C\), $B_u(x,R)=C,$ where $B_u$ is defined in Equation~\eqref{eq:ultrametric_hierarchy}. 
Hence \(C\in\Psi_u\) by the ultrametric-ball representation of \(\Psi_u\). 
This proves that \(\Tstable(u)\subseteq \Psi_u.\)

\medskip
\noindent\emph{3. Cluster-wise consistency: }
Let \(C\in\Tstable(d)\), and let \(d'\) be a \(\{C\}\)-strengthening of \(d\). The root and singleton cases are immediate, so assume that \(C\) is nontrivial and proper.

Consider two nonempty disjoint sets $U, V \subseteq \cX$ with $U \cup V = C$, and a nonempty $Z \subseteq \cX \setminus C$. For every \(a\in U\), \(b\in V\), \(z\in Z\), we have  $d'(a,b)\le d(a,b)$, $d'(a,z)\ge d(a,z)$, and $d'(b,z)\ge d(b,z)$. Thus,
\begin{align*}
 \rho_{d'}(ab \mid z)
 = \min\{d'(a, z), d'(b, z)\} - d'(a, b) 
 \wge \min\{d(a, z), d(b, z)\} - d(a, b) 
 = \rho_d(ab \mid z).
\end{align*}
Averaging and then minimizing over all admissible triples gives $\iota^{d'}(C)\ge\iota^d(C)>0.$ 
Hence \(C\in\Tstable(d')\), proving cluster-wise consistency.

\medskip
\noindent\emph{4. Permutation invariance.}
Let \(\phi\) be a permutation of \(\cX\), and recall \(d_\phi(x,y):=d(\phi^{-1}(x),\phi^{-1}(y))\). 
For every \(u,v,z\in\cX\), we have $\rho_{d_\phi}\bigl(\phi(u)\phi(v)\mid\phi(z)\bigr) = \rho_d(uv\mid z).$ 
Moreover, the map $(U,V,Z)\longmapsto \bigl(\phi(U),\phi(V),\phi(Z)\bigr)$ is a bijection between the triples admissible in the definition of \(\iota^d(C)\) and those admissible in the definition of \(\iota^{d_\phi}(\phi(C))\). 
Therefore $\iota^{d_\phi}(\phi(C))=\iota^d(C)$ for every \(C\subseteq\cX\). 
Hence
\[
C\in\Tstable(d)
\iff
\phi(C)\in\Tstable(d_\phi),
\]
proving $\Tstable(d_\phi)=\phi\cdot\Tstable(d).$ 
\end{proof}

\section{Proofs for Section~\ref{sec:structure}}
\subsection{Incompatibility among Methods}
\label{app:incompatibility}

\begin{proof}[Proof of Lemma~\ref{lem:incompatible_family}]
First, we consider $\cX = [4]$.
We show that for any $0 < p < q$, the methods $\Tstable^{(p)}$ and $\Tstable^{(q)}$ admit no common refinement that outputs hierarchies. Let
\[
\cX = \{1,2,3,4\},
\qquad
C_1 := \{1,2,3\},
\qquad
C_2 := \{2,3,4\}.
\]
Then $C_1 \cap C_2 = \{2,3\}$ is nonempty while neither $C_1 \subseteq C_2$ nor $C_2 \subseteq C_1$, so $C_1$ and $C_2$ are not laminar-compatible; in particular, no hierarchy on $\cX$ contains both. We construct $d \in \cD(\cX)$ such that
\(
C_1 \in \Tstable^{(p)}(d)
\,
\text{and}
\,
C_2 \in \Tstable^{(q)}(d),
\)
hence prove that $\Tstable^{(p)}$ and $\Tstable^{(q)}$ are mutually incompatible.

\medskip 
\noindent\emph{Construction of $d$.} 
Let $r := q/p > 1$. Fix $\alpha > 0$ small enough that $\bigl(1 + \alpha/2\bigr)^r < 2$; such an $\alpha$ exists because $(1+\alpha/2)^r \to 1$ as $\alpha \to 0^+$. Set
\(
h := 1 + \alpha,
\,
t_0 := 1 + \frac{\alpha}{2} = \frac{1 + h}{2}.
\)
Notice 
\(
t_0^r
=
\left(\frac{1+h}{2}\right)^r
< 
\frac{1 + h^r}{2},
\)
equivalently $1 + h^r > 2 t_0^r$ holds. By continuity, both strict inequalities $t_0^r < 2$ and $1 + h^r > 2 t_0^r$ persist after a small perturbation of $t_0$: there exists $\delta \in (0, \alpha/2)$ such that
$t := t_0 + \delta$ satisfies $t^r < 2$ and $1 + h^r > 2 t^r.$ 
The constraint $\delta < \alpha/2$ ensures $1 < t < h$. Finally, choose $\varepsilon$ with
\(
0 < \varepsilon < 1
\,\text{and}\,
\varepsilon^r < 2 - t^r;
\)
both constraints are compatible because $t^r < 2$.

Pick any $M > h$, and define $d \in \cD(\cX)$ by prescribing the $p$-th powers of its values:
\[
\begin{aligned}
&d(1,2)^p = 1, &&d(1,3)^p = h, &&d(1,4)^p = M,\\
&d(2,3)^p = \varepsilon, &&d(2,4)^p = t, &&d(3,4)^p = t.
\end{aligned}
\]
All six prescribed values are strictly positive, so $d$ is a valid dissimilarity. By using the ordering
\(
\varepsilon < 1 < t < h < M,
\)
we can check by direct computation that $C_1 \in \Tstable^{(p)}(d)$ whereas $C_2 \in \Tstable^{(q)}(d)$. 

Suppose, for contradiction, that $T$ refines both $\Tstable^{(p)}$ and $\Tstable^{(q)}$. Then $T(d)$ contains both $C_1$ and $C_2$, contradicting laminarity.

For a domain $\cX'$ with $|\cX'|>4$, retain the above dissimilarities on $\{1,2,3,4\}$ and set $d(w,i)=L$ for every new point $w$ and every $i\in\{1,2,3,4\}$, where $L$ is larger than every dissimilarity in the four-point construction; assign arbitrary positive symmetric dissimilarities among the new points. For any new external witness $w$, every isolation weight contributing to the stability of $C_1$ under power $p$ or of $C_2$ under power $q$ is strictly positive. An average over an external set containing both old and new witnesses is therefore a positive average of the already verified terms and these new positive terms. Hence $C_1\in\Tstable^{(p)}(d)$ and $C_2\in\Tstable^{(q)}(d)$ still hold on $\cX'$, proving incompatibility for every $|\cX'|\ge4$.
\end{proof}

\begin{lemma}
\label{lem:pairwise_incompatible}
Let $\power>0$. The methods $\Tsl$ and $\Tstable^{(\power)}$ are incompatible. 
\end{lemma}

\begin{proof}
We first construct a dissimilarity function $d_0$ on $\cX=\{1,2,3,4\}$ such that $\Tsl(d_0)$ and $\Tstable^{(\power)}(d_0)$ are incompatible. Define $d_0 \in \cD(\cX)$ such that
\[
d_0(1,2) =3, \quad d_0(1,3)=d_0(1,4) = 10, \quad d_0(2,3)=1, \quad d_0(2,4) = d_0(3,4) = 4.
\]
We can verify that $\{1,2,3\}\in \Tsl(d_0)$ while $\{2,3,4\} \in \Tstable(d_0).$
The clusters $\{1,2,3\}$ and $\{2,3,4\}$ overlap, but neither contains the other. Thus $\Tsl({d_0})\cup\Tstable({d_0})$ is not laminar.

Now fix $\power>0$ and define $d$ entry-wise by $d(x,y):={d_0}(x,y)^{1/\power}.$ 
Then $\trf^\power(d)={d_0}$. Because single linkage is order invariant, \(\Tsl(d)=\Tsl({d_0}),\) whereas, by definition,
\(
\Tstable^{(\power)}(d)
=
\Tstable(\trf^\power(d))
=
\Tstable({d_0}).
\)
Hence the same two non-laminar clusters belong to
\(
\Tsl(d)\cup\Tstable^{(\power)}(d).
\)

For $n>4$, extend ${d_0}$ by setting ${d_0}(i,j)=20$ whenever at least one of
$i,j$ belongs to $\{5,\ldots,n\}$. This does not affect the single-linkage
cluster $\{1,2,3\}$. Moreover, for every new point $z$ and every
$u,v\in\{2,3,4\}$,
\(
\rho_{d_0}(uv\mid z)=20-{d_0}(u,v)>0.
\)
Thus $\iota^{d_0}(\{2,3,4\})>0$ remains valid. Taking $d(i,j)={d_0}(i,j)^{1/\power}$ completes the proof for every $n\geq4$.
\end{proof}

\subsection{Well-separated Backbone Hierarchy}
\label{app:core_hierarchy}

\subsubsection{Proof of Theorem~\ref{thm:core_hierarchy}}

The canonical partition consistency axiom is formulated at the level of partitions: if a partition already appears in the hierarchy, then strengthening that partition preserves all of its blocks. 
To establish the existence of a well-separated backbone hierarchy, however, we need a cluster-wise consequence of this axiom: a cluster that is sufficiently well separated from its complement must be selected by any admissible
method. 

We prove it in two steps. 
First, we show in Lemma~\ref{lemma:single_cluster_extraction} that partition consistency yields a cluster-wise consequence: if there exists a dissimilarity \(d_0\) such that \(\{C,\bC\}\subseteq T(d_0)\), the cluster \(C\) remains selected under any dissimilarity that does not increase dissimilarities within \(C\) and does not decrease dissimilarities from $C$ to $\bC$. 
Second, partition richness provides a reference dissimilarity \(d_0\) realizing \(\{C,\bC\}\). 
Moreover, by scale invariance, we can rescale any dissimilarity $d$ satisfying $\varrho(d,C)<\eta_C$, for a sufficiently small constant \(\eta_C>0\), by a factor \(\beta>0\) so that $\beta d(x,y)\le d_0(x,y)$ for all $x,y\in C$, and $\beta d(x,z)\ge d_0(x,z)$ for all $x\in C,\ z\in\bC.$ 
Lemma~\ref{lemma:single_cluster_extraction} then gives \(C\in T(\beta d)\), and scale invariance yields \(C\in T(d)\). 
This is the content of Lemma~\ref{lemma:core_hierarchy}. 


\begin{lemma}
\label{lemma:single_cluster_extraction}
Let \(T\) be a hierarchical method satisfying partition consistency, and let \(C\subsetneq\cX\) be a nonempty proper cluster.  
Let \(d_0,d\in\cD(\cX)\) such that $d$ is a $\{C\}$-strengthening of $d_0$. 
If $\{C,\bC\}\subseteq T(d_0)$ then \(C\in T(d)\). 
\end{lemma}

\begin{proof}
Define \(d_1\in\mathcal D(\cX)\) by setting
\[
d_1(x,y):=\min\{d(x,y),d_0(x,y)\}\ \text{when }x,y\in\bC,
\quad\text{and}\quad
d_1(x,y):=d_0(x,y)\ \text{otherwise}.
\]
Then \(d_1\) is a \(\{C,\bC\}\)-strengthening of \(d_0\). 
Because \(\{C,\bC\}\subseteq T(d_0)\), partition consistency ensures that $\{C,\bC\}\subseteq T(d_1)$.
In particular, the partition $\cP_C:=\{C\}\cup\{\{x\}:x\in\bC\}$
satisfies \(\mathcal P_C\subseteq T(d_1)\), because every hierarchy contains all singleton
clusters.

We now prove that \(d\) is a \(\mathcal P_C\)-strengthening of \(d_1\). 
Indeed, we have
\begin{itemize}[nosep]
 \item \(d(x,y) \le d_0(x,y) = d_1(x,y)\) for every $x,y \in C$,
 \item \(d(x,y) \ge d_0(x,y) = d_1(x,y)\) for every $x\in C$, $y \in \bC$,
 \item \(d(x,y) \ge \min \left\{(d(x,y), d_0(x,y)\right\} = d_1(x,y)\) for every $x,y \in \bC$.
\end{itemize}
Thus, partition consistency ensures $\cP_C\subseteq T(d)$, and hence \(C\in T(d)\).
\end{proof}

The previous lemma reduces the problem to finding a rescaling of a given dissimilarity \(d\) which is dominated by \(d_0\) inside \(C\), but which dominates \(d_0\) across the cut \((C,\bC)\). Such a rescaling exists whenever the internal diameter of \(C\) under \(d\) is sufficiently small compared to its distance from the complement.

\begin{lemma}
\label{lemma:core_hierarchy}
Let \(T\) satisfy scale invariance, partition richness, and partition consistency.
For any cluster \(C\subsetneq \cX\), with \(|C|\ge 2\), there exists \(\eta_C>0\) such that, for every \(d\in\cD(\cX)\),
\[
\varrho(d,C)<\eta_C \quad\Longrightarrow\quad C\in T(d).
\]
If $T$ further satisfies permutation invariance, $\eta_C$ only depends on the size $|C|$ and $\eta_C \in (0,1]$.
\end{lemma}

Applying Lemma~\ref{lemma:core_hierarchy} to an admissible method \(T\), define the separation margin sequence \(\bdeta=(\eta_s)_{1\le s\le n-2}\) as provided by Lemma~\ref{lemma:core_hierarchy}. 
Then every non-singleton proper cluster \(C\in\Tglob^{\bdeta}(d)\) satisfies $\varrho(d,C)<\eta_{|C|-1},$ and therefore belongs to \(T(d)\). 
Because both hierarchies contain the root and all singleton clusters, $\Tglob^{\bdeta}(d)\subseteq T(d)$ for every $d\in\cD(\cX)$, and hence $\Tglob^{\bdeta}\sqsubseteq T.$ This proves Theorem~\ref{thm:core_hierarchy}.

\begin{proof}[Proof of Lemma~\ref{lemma:core_hierarchy}]
By partition richness, choose \(d_0\in\mathcal D(\cX)\) such that $\{C,\bC\}\subseteq T(d_0)$. 
Define $a:=\min_{\substack{x,y\in C\\x\neq y}} d_0(x,y)$ and $b:=\max_{x\in C,\ z\in\bC} d_0(x,z).$ 
Both \(a\) and \(b\) are positive. Set $\eta_C:=\frac{a}{b}.$ 

Now let \(d\) be a dissimilarity satisfying \(\varrho(d,C)<\eta_C\). 
Write $M:=\max_{x,y\in C} d(x,y)$ and $m:=\min_{x\in C,\ z\in\bC} d(x,z)$.
Then $\frac{M}{m}=\varrho(d,C)<\frac{a}{b}$, and hence $\frac{b}{m}<\frac{a}{M}$. 
Next, choose \(\beta>0\) such that $\frac{b}{m}<\beta<\frac{a}{M}$.
This choice ensures that,
\[
\beta d(x,y)\le \beta M<a\le d_0(x,y), \quad \text{for all } x,y \in C, \text{ and}
\]
\[
\beta d(x,z)\ge \beta m>b\ge d_0(x,z) \quad \text{for all } x \in C, z\in\bC.
\]
Thus, by Lemma~\ref{lemma:single_cluster_extraction}, applied to \(\beta d\), we get $C\in T(\beta d).$ 
Finally, scale invariance gives \(T(\beta d)=T(d)\), hence \(C\in T(d)\).

Assume now that \(T\) is permutation invariant. For every non-singleton proper cluster \(C\subsetneq \cX\), define
\[
\Gamma_T(C)
:=
\sup_{\substack{d_0\in\cD(\cX)\\ \{C,\bC\}\subseteq T(d_0)}} \frac{\min_{\substack{x,y\in C\\ x\neq y}} d_0(x,y)}{\max_{x\in C,\ z\in\bC} d_0(x,z)}.
\]
By the previous paragraph, for every \(d\in\cD(\cX)\), we have: 
\(
\varrho(d,C)<\Gamma_T(C)
\Longrightarrow C\in T(d).
\)

 We now prove that \(\Gamma_T(C)\) depends only on the cardinality of \(C\).
Let \(C,C'\subsetneq\cX\) be two nontrivial subsets such that \(|C|=|C'|\). 
Because \(\cX\) is finite, there exists a permutation \(\phi\in\Pi(\cX)\) such that $\phi(C)=C'.$ 
Because permutation invariance of \(T\) implies $\Gamma_T(\phi(C))=\Gamma_T(C)$, we have
\(
\Gamma_T(C')=\Gamma_T(C).
\)
Hence \(\Gamma_T(C)\) is constant over all subsets of \(\cX\) having the same cardinality, and thus depends only on \(|C|\).
Thus, for \(1\le s\le |\cX|-2\), we may define
\[
\eta_{s}:=\Gamma_T(C)
\]
for any subset \(C\subseteq \cX\) with \(s=|C|-1\). This is well-defined by the preceding argument, and, for every \(C\subseteq \cX\),
\[
\varrho(d,C)<\eta_{|C|-1}
\quad\Longrightarrow\quad
C\in T(d).
\]

It remains to prove that, for \(1\le s\le |\cX|-2\), one has $0<\eta_{s}\le 1.$ 
Observe that $\eta_{s} > 0$ follows from partition richness, as for each \(C\) there exists \(d_0\) such that \(\{C,\bC\}\subseteq T(d_0)\), and the corresponding ratio is strictly positive.

To establish $\eta_{s}\le 1$, suppose by contradiction that \(\eta_{s}>1\) for some \(1\le s\le |\cX|-2\). Let \(d\) be the uniform dissimilarity on \(\cX\), namely \(d(x,y)=1\) for all \(x\neq y\). Then \(\varrho(d,C)=1<\eta_{s}\) for every subset \(C\subsetneq \cX\) of size \(s+1\). Hence every such \(C\) belongs to \(T(d)\).
But there exist two subsets of \(\cX\) of size \(s+1\) that overlap without either containing the other.
For instance, since \(1\le s\le n-2\), the sets $C_1:=\{1,\ldots,s+1\}$ and $C_2:=\{2,\ldots,s+2\}$ both have cardinality \(s+1\), overlap nontrivially, and neither contains the other.
This contradicts the laminarity of \(T(d)\). Therefore \(\eta_{s}\le 1\).
\end{proof}

\subsection{Additional Lemmas}

\begin{lemma}
\label{lem:uhr_pc_implies_strong_core}
We have \( \axiuhr \wedge \axipc \triangleright \axicore^1\). 
\end{lemma}

\begin{proof}
Let $T \in \cH_{\axiuhr \wedge \axipc}(\cX).$ We prove that \(\Tglob \sqsubseteq T\). 
Fix \(d \in \cD(\cX)\) and \(C \in \Tglob(d)\). The cases
\(C=\cX\) and \(\abs{C}=1\) are immediate because every hierarchy
contains the root and all singleton clusters. We therefore assume that \(C\subsetneq\cX\) and \(\abs{C}\geq 2\). 
Define $\alpha := \max_{\substack{x,y\in C\\x\neq y}} d(x,y)$ and $\beta := \min_{\substack{x\in C\\z\in\bC}} d(x,z).$ 
Because \(C\in\Tglob(d)\), we have \(\alpha<\beta\). Choose\(\alpha<\widetilde{\beta}<\beta\)
and set
\(\gamma:= \min\left( \{\widetilde{\beta}\} \cup \{d(z,w):z,w\in\bC,\ z\neq w\} \right). \)
In particular, \(0<\gamma\leq\widetilde{\beta}\).

Define \(u\in\cD(\cX)\) by \(u(x,x):=0\) and, for distinct
\(x,y\in\cX\),
\[
u(x,y)
:=
\begin{cases}
\alpha,
    & x,y\in C,\\[1mm]
\widetilde{\beta},
    & x\in C,\ y\in\bC
      \text{ or }y\in C,\ x\in\bC,\\[1mm]
\gamma,
    & x,y\in\bC.
\end{cases}
\]
Because $\alpha<\widetilde{\beta}$ and $\gamma\le\widetilde{\beta}$, \(u\) is an ultrametric: every triangle meeting both \(C\) and \(\bC\) has two dissimilarities equal to \(\widetilde{\beta}\), while the third is
either \(\alpha\) or \(\gamma\), and triangles contained entirely in \(C\) or in \(\bC\) are equilateral.

Recall the ultrametric-ball representation in Equation~\eqref{eq:ultrametric_hierarchy}. 
For every \(x\in C\), $B_u(x,\alpha)=C$ because \(u(x,y)\le\alpha\) for \(y\in C\), whereas \(u(x,z)=\widetilde{\beta}>\alpha\) for \(z\in\bC\). 
Hence \(C\in\Psi_u\). 
As \(T\) satisfies refinement on ultrametrics, $C\in\Psi_u\subseteq T(u).$

Now consider the partition
\(
\cP_C
:=
\{C\}
\cup
\bigl\{\{z\}:z\in\bC\bigr\}.
\)
Because every hierarchy contains all singleton clusters, we have
\(
\cP_C\subseteq T(u).
\) 
We claim that \(d\) is a \(\cP_C\)-strengthening of \(u\). Indeed, for
distinct \(x,y\in C\),
\(
d(x,y)\leq\alpha=u(x,y).
\)
For \(x\in C\) and \(z\in\bC\),
\(
d(x,z)\geq\beta>\widetilde{\beta}=u(x,z).
\)
Finally, for distinct \(z,w\in\bC\), the points \(z\) and \(w\) belong
to different singleton blocks of \(\cP_C\), and the definition of
\(\gamma\) gives
\(
d(z,w)\geq\gamma=u(z,w).
\)

Partition consistency therefore yields
\(
\cP_C\subseteq T(d),
\)
and hence \(C\in T(d)\). Because \(d\in\cD(\cX)\) and
\(C\in\Tglob(d)\) were arbitrary, we conclude that
\(
\Tglob\subseteq T.
\)
\end{proof}

\begin{lemma}
\label{lemma:core-under-ord}
We have
\(
 \axiadm \wedge \axiord \triangleright \axicore^1,
\)
i.e., for any $T \in \cH_{\axiadm \wedge \axiord}$, $\Tglob \sqsubseteq T$ holds.
\end{lemma}

\begin{proof} Let \(T\) be admissible and order invariant. 
 By Theorem~\ref{thm:core_hierarchy}, there exists a separation margin
sequence $\bdeta=(\eta_s)_{1\le s\le n-2}$ such that $\Tglob^{\bdeta}\sqsubseteq T.$ 
Set $\eta_{\min}:=\min_{1\le s\le n-2}\eta_s>0.$ 

Fix \(d\in\cD(\cX)\). 
We construct a strictly increasing function \(g:\R_{\ge0}\to\R_{\ge0}\), with \(g(0)=0\), such that $\Tglob(d)\subseteq \Tglob^{\bdeta}(g\circ d).$ 

Let $0<\delta_1<\delta_2<\cdots<\delta_m$ be the distinct positive values taken by \(d\). Choose $q>\frac{1}{\eta_{\min}},$ 
and prescribe $g(0):=0$ and $g(\delta_i):=q^i$ for all $i\in[m].$ 
Because the sequences $(\delta_i)$ and $(q^i)$ are strictly increasing, these prescribed values can be extended to a strictly
increasing function $g:\R_{\ge0}\to\R_{\ge0}$ with \(g(0)=0\); for instance, using piecewise-linear interpolation on \([0,\delta_m]\) and extending linearly with positive slope beyond \(\delta_m\).

Now let \(C\in\Tglob(d)\) be a nontrivial proper cluster. 
By definition, $d(x_1,y)<d(x_2,z)$ for every \(x_1,x_2,y\in C\) and \(z\notin C\) whenever
\(d(x_1,y)>0\). 
Hence, if $d(x_1,y)=\delta_i$ and $d(x_2,z)=\delta_j,$ then \(i<j\), and therefore
\[
\frac{g(d(x_1,y))}{g(d(x_2,z))}
\weq q^{i-j}
\wle \frac1q
\ < \ \eta_{\min}
\wle \eta_{|C|-1}.
\]
If \(d(x_1,y)=0\), the same inequality holds trivially because \(g(0)=0\). 

Therefore, $\varrho(g\circ d,C) < \eta_{|C|-1}$, and hence $C\in\Tglob^{\bdeta}(g\circ d).$ 
The root and singleton clusters belong to both hierarchies by definition, so
$\Tglob(d)
\subseteq
\Tglob^{\bdeta}(g\circ d).
$
Using \(\Tglob^{\bdeta}\sqsubseteq T\), we further obtain 
\[
\Tglob(d)
\subseteq
\Tglob^{\bdeta}(g\circ d)
\subseteq
T(g\circ d).
\]
Finally, order invariance of \(T\) gives $T(g\circ d)=T(d),$ and therefore $\Tglob(d)\subseteq T(d).$ 
Because \(d\in\cD(\cX)\) was arbitrary, $\Tglob\sqsubseteq T.$ 
\end{proof}

\section{Empirical Size of Backbone Hierarchy}
\label{app:size_Tglob}
We quantify the size of the backbone hierarchy across four standard scikit-learn datasets~\citep {scikit-learn}: Iris, Wine, Breast Cancer, and Digits. We remove exact duplicate feature vectors, standardize each feature, and use Euclidean dissimilarities.
We count only nontrivial clusters \(C\) satisfying \(1<|C|<n\), and report the ratio of the number returned by $\Tglob$ to the number returned by $\Tsl$.

\begin{table}[!ht]
\centering
\caption{Size of $\Tglob$ relative to the non-binary single-linkage hierarchy
$\Tsl$.}
\label{tab}
\begin{tabular}{lrrrr}
\toprule
Dataset & $n$ & $\Tglob$ & $\Tsl$ & Ratio \\
\midrule
Iris & 149 & 42 & 145 & 0.290 \\
Wine & 178 & 43 & 176 & 0.244 \\
Breast Cancer & 569 & 107 & 567 & 0.189 \\
Digits & 1797 & 426 & 1795 & 0.237 \\
\bottomrule
\end{tabular}
\end{table}

Across these datasets, $\Tglob$ contains approximately \(19\%\)--\(29\%\) of the nontrivial clusters in $\Tsl$ (\(23.0 \%\) in aggregate). Thus, the backbone is nontrivial on these examples.